\PassOptionsToPackage{table}{xcolor}

\documentclass{article}

\usepackage[sort,numbers]{natbib}
\usepackage[linesnumbered,ruled,vlined]{algorithm2e}

\newif\ifjournal
\journaltrue      

\ifjournal
    \usepackage[main,final]{neurips_2026}
\else
    \usepackage[margin=1in]{geometry}
    \usepackage{authblk}
    \usepackage{lmodern}
\fi

\usepackage{customize}

\usepackage[most]{tcolorbox}

\def \figdir {.}

\newcommand{\mbf}[1]{\mathbf{#1}}              
\newcommand{\mbif}[1]{\boldsymbol{#1}}         
\newcommand{\tens}[1]{\mbf{#1}}                
\newcommand{\gtens}[1]{\mbif{#1}}              
\newcommand{\vect}[1]{\mbif{#1}}               

\DeclareMathOperator*{\supp}{\operatorname{supp}}

\DeclareMathOperator*{\Res}{Res}
\DeclareMathOperator*{\argmin}{argmin}

\theoremstyle{definition}
\newtheorem*{example*}{Example}

\newcommand{\freealg}{freealg}
\newcommand{\bemph}[1]{\textbf{\emph{#1}}}

\def \ie {i.e.,\ }

\def \eg {e.g.,\ }

\newcommand{\fulltitle}{Free Decompression with Algebraic Spectral Curves}

\title{\fulltitle}

\newcommand{\authrone}{Siavash Ameli}
\newcommand{\affilone}{ICSI and Department of Statistics}
\newcommand{\addrsone}{University of California, Berkeley}
\newcommand{\emailone}{sameli@berkeley.edu}

\newcommand{\authrtwo}{Chris van der Heide}
\newcommand{\affiltwo}{Dept. of Electrical and Electronic Engineering}
\newcommand{\addrstwo}{University of Melbourne}
\newcommand{\emailtwo}{chris.vdh@gmail.com}

\newcommand{\authrthree}{Liam Hodgkinson}
\newcommand{\affilthree}{School of Mathematics and Statistics}
\newcommand{\addrsthree}{University of Melbourne}
\newcommand{\emailthree}{lhodgkinson@unimelb.edu.au}

\newcommand{\authrfour}{Michael W. Mahoney}
\newcommand{\affilfour}{ICSI, LBNL, and Department of Statistics}
\newcommand{\addrsfour}{University of California, Berkeley}
\newcommand{\emailfour}{mmahoney@stat.berkeley.edu}

\ifjournal

    \newcommand{\AuthorBlock}[4]{%
        {\normalfont                          
        \begin{minipage}[t]{0.49\textwidth}   
          \raggedright
          \textbf{#1}\\[2pt]                  
          #2\\                                
          #3\\                                
          \texttt{#4}                         
        \end{minipage}}%
    }

    \author{%
      \AuthorBlock{\authrone}{\affilone}{\addrsone}{\emailone}\hfill
      \AuthorBlock{\authrtwo}{\affiltwo}{\addrstwo}{\emailtwo}%
      \\[17mm]
      \AuthorBlock{\authrthree}{\affilthree}{\addrsthree}{\emailthree}\hfill
      \AuthorBlock{\authrfour}{\affilfour}{\addrsfour}{\emailfour}%
    }
    
\else
    \author[1]{\authrone}
    \author[2]{\authrtwo}
    \author[3]{\authrthree}
    \author[4]{\authrfour}
    \affil[1]{\small{\textit{\affilone, \addrsone} \protect\\ \href{mailto:\emailone}{\protect\nolinkurl{\emailone}}} \vspace{2mm}}
    \affil[2]{\small{\textit{\affiltwo, \addrstwo} \protect\\ \href{mailto:\emailtwo}{\protect\nolinkurl{\emailtwo}}} \vspace{2mm}}
    \affil[3]{\small{\textit{\affilthree, \addrsthree} \protect\\ \href{mailto:\emailthree}{\protect\nolinkurl{\emailthree}}} \vspace{2mm}}
    \affil[4]{\small{\textit{\affilfour, \addrsfour} \protect\\ \href{mailto:\emailfour}{\protect\nolinkurl{\emailfour}}} \vspace{2mm}}
    
    \date{}
\fi

\ifjournal
    \def \FreeAlgPypiUrl {\url{https://pypi.org/project/freealg}}
    \def \FreeAlgDocUrl {\url{https://ameli.github.io/freealg}}
    \def \FreeAlgGithubUrl {\url{https://github.com/ameli/freealg}}
\else
    \def \FreeAlgPypiUrl {\url{https://pypi.org/project/freealg}}
    \def \FreeAlgDocUrl {\url{https://ameli.github.io/freealg}}
    \def \FreeAlgGithubUrl {\url{https://github.com/ameli/freealg}}
\fi

\begin{document}

\addtocontents{toc}{\protect\setcounter{tocdepth}{-1}}  

\maketitle


\begin{abstract}
    Tools from random matrix theory have become central to deep learning theory, using spectral information to provide mechanisms for modeling generalization, robustness, scaling, and failure modes. While often capable of modeling empirical behavior, practical computations are limited by matrix size, often imposing a restriction to models that are too small to be realistic. This motivates the inference of properties of larger models from the behavior of smaller ones. Free decompression (FD) is a recently proposed method for extrapolating spectral information across matrix sizes, but its utility is currently limited by strong assumptions that preclude its implementation on more realistic machine learning (ML) models. We use algebraic spectral curve theory to provide a general FD methodology for spectral densities whose Stieltjes transform satisfies an algebraic relation, a modeling assumption that is more likely to hold in practice. This recasts FD as an evolution along spectral curves which can be readily integrated. Our framework enables the expansion of spectral densities that have multiple or multi-modal bulks, that exist at multiple scales, and that contain atoms, all characteristic of real-world data and popular ML models. We demonstrate the efficacy of our framework on models of interest in modern ML, including Hessian and activation matrices associated with neural networks and large-scale diffusion models.
\end{abstract}


\section{Introduction} 
\label{sec:intro}

At the core of scientific computing and much of modern machine learning (ML) lies the challenge of estimating the eigenvalues of high-dimensional Hermitian matrices. Such matrices, including kernels, Hessians, and graph representations, encode the intrinsic geometry and connectivity of the data and models built on them, rendering the pursuit of efficient spectral techniques a primary concern for both theory and practice. Studying eigenspectra has become a prominent approach to understanding performance and guiding training in deep learning \cite{PENNINGTON-2017c, WEI-2022, BONNAIRE-2025, HODGKINSON-25,HODGKINSON-2023a,HODGKINSON-2023b,HODGKINSON-2023c}. In many cases, the spectra of such matrices have non-trivial structure, often containing spikes, multiple multi-modal bulks, and heavy-tails \citep{EL_KAROUI-2010, MARTIN-2021}. Conventional algorithms to extract eigenvalue information from these matrices have required that the data are able to be stored in memory, scratch space, or can at least be accessed as an implicit operator (via matrix-vector products).

More recently, a new class of algorithms has emerged that is able to provide highly-accurate estimates of the eigenvalues (or summary functionals thereof \citep{AMELI-2025a}) of matrices, \emph{even without implicit or explicit access to the full matrix}, i.e., of so-called \emph{impalpable} matrices \citep{AMELI-2025b}. One such method, termed \emph{Free Decompression} (FD), shows great promise as a tool for gaining access to the spectral distributions of such impalpable matrices. The central premise is that by appropriately sampling a small sub-matrix from the large impalpable matrix of interest, one can evolve a partial differential equation (PDE) in the Stieltjes transform of a spectral density in the decompression ratio to the desired matrix dimension. Despite demonstrating impressive accuracy when considering synthetic examples that naturally arise in random matrix theory (RMT) and free probability, the methods in \citep{AMELI-2025b} are highly restrictive, assuming the distribution of eigenvalues can be described by a single smooth connected bulk. This assumption is inevitably violated when challenged with more realistic examples, preventing widespread adoption of FD-based methods. 

Instead, we develop a more broadly applicable FD methodology using algebraic spectral curve theory. We assume that the Stieltjes transform of the limiting spectral distribution of a class of random matrices is \emph{algebraic}, i.e., it can be represented as a root of a (typically low-order) bivariate polynomial relation. All canonical random matrix ensembles satisfy this assumption (see \Cref{rem:EnsembleQuadratic}), and with increasing polynomial degree, an extraordinarily broad class of spectral densities can be captured in this way. For example, Pennington et al. \citep{PENNINGTON-2017a,PENNINGTON-2017c} highlight that the Stieltjes transforms of certain Hessian approximations and neural network gram matrices are, respectively, solutions to cubic and quartic polynomial equations. \citet{RAO-2008} introduced the \emph{polynomial method}, showing how algebraic representations can, in principle, simplify random-matrix computations because many operations and transforms preserve this structure. Conveniently, it turns out that FD does too, recasting decompression as evolution on a \emph{spectral curve}. This fact can be exploited to circumvent many of the challenges faced by the approach taken in \citet{AMELI-2025b}, including the problem of analytic continuation, which can be handled geometrically on algebraic spectral curves. Unfortunately, to our knowledge, there has been no general numerical implementation of the polynomial method for empirical matrix computations to date, due to major practical challenges in fitting the polynomial and identifying its physical root. Indeed, \citet[Sections~7 and~11]{RAO-2008} identified both as open problems. These issues must first be overcome to accomplish FD at scale.



\begin{figure*}[t]
    \begin{minipage}[b]{0.37\textwidth}
        \includegraphics[width=\textwidth]{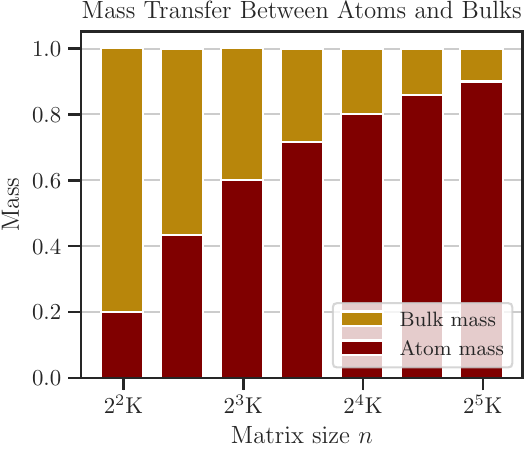}
    \end{minipage}%
    \hfill%
    \begin{minipage}[b]{0.61\textwidth}
        \includegraphics[width=\textwidth]{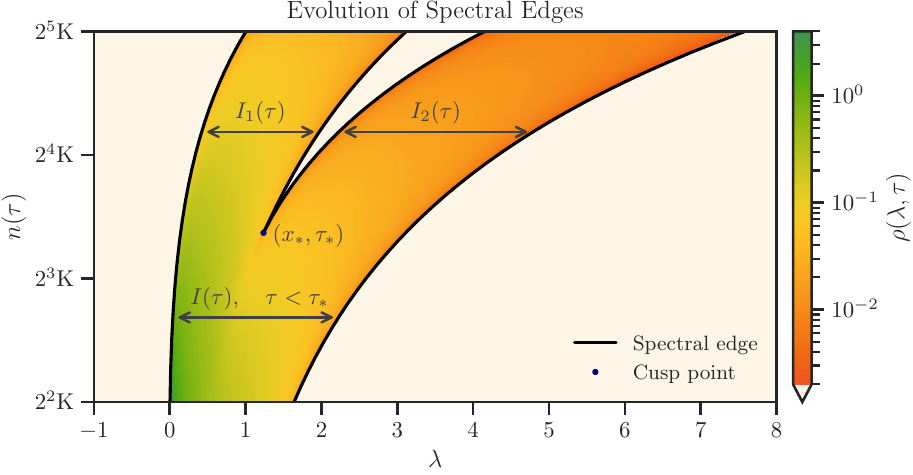}
    \end{minipage}
    \ifjournal
        \vspace{-1.5mm}
    \fi
    \caption{Evolution of atomic mass (left), density and spectral edges (right) with increasing matrix size $n(\tau)$ for the compound free Poisson law in \Cref{sec:results}. A complexity our method handles is that the bulk with support $I(\tau)$ splits at the cusp $(x_{\ast}, \tau_{\ast})$ to reveal two evolving supports $I_1(\tau)$ and $I_2(\tau)$.}
    \label{fig:cfp-atom}
    \ifjournal
        \vspace{-3.25mm}
    \fi
\end{figure*}

\paragraph{Contributions.}
The central contribution of this work is to introduce a new general FD methodology and practical tools for spectral computations on large impalpable matrices. We develop the first general numerical implementation of the polynomial method for empirical spectra and apply it to obtain a state-of-the-art FD implementation. This enables us to directly trace several spectral quantities, including the edges of the distribution, locations of atoms and their weights, as well as spectral moments (see \Cref{fig:cfp-atom}). This is highly valuable in many applications: for positive-definite matrices, the ratio of the largest to the smallest eigenvalues constitutes the \emph{condition number} $\mathrm{cond}(\tens{A})$ of the matrix $\tens{A}$, a critical component for understanding rates of convergence in optimization algorithms, for example. Similarly, it was demonstrated in \citet{BONNAIRE-2025} that the edges of bulk regions of the spectrum play an important role in the training dynamics of diffusion models. 
In doing so, we 
\begin{itemize}[labelindent=12pt, leftmargin=*]
    \item develop an algebraic FD framework (AFD) that generalizes previous methods beyond quadratic Stieltjes transforms, extending previous work to accurately approximate multi-scale spectral distributions, which may be multi-modal and/or contain multiple bulk regions; 
    \item derive equations to accurately compute the dynamics of important spectral quantities (atoms, moments, and edges), as well as track the mechanism of splitting of bulks, without the need to compute the full spectral distribution; 
    \item demonstrate empirically that the method works on several classes of ML models; and
    \item deploy an implementation into the open source software project \texttt{\freealg}.\footnote{Available at \FreeAlgGithubUrl, with documentation at \FreeAlgDocUrl.}
\end{itemize}

The remainder of the paper is structured as follows. \Cref{sec:fd} outlines some background and related work, as well as tools from free probability. \Cref{sec:poly} describes algebraic free decompression (AFD), detailing FD in this context, and the numerical methods we use, including associated methods for atoms, spectral edges, and moments. \Cref{sec:results} provides numerical examples demonstrating the utility of our method. We conclude and discuss limitations in \Cref{sec:conclusion}. Background, proofs, additional information and experiments are provided in a detailed set of appendices.

\section{Background on Free Decompression} \label{sec:fd}

Free decompression was recently introduced by \citet{AMELI-2025b} as a method for gaining access to the spectral distribution of so-called \emph{impalpable} matrices: cases where the user does not have complete access to the matrix, perhaps due to its sheer size, corruption, or data availability. The method relies upon tools from free probability to derive a PDE in the familiar Stieltjes transform, a standard tool in RMT that at this stage has been widely deployed in both ML theory and the development of practical algorithms. 

We briefly recall some relevant notation. Let $\tens{A}$ denote a matrix of size $n \times n$, where $n$ is assumed to be large, and let $m_{\tens{A}}$ denote the corresponding \emph{empirical Stieltjes transform}, given by
\begin{align}
    \label{eq:EmpStieltjes}
    m_{\tens{A}}(z) \coloneqq \frac{1}{n}\tr(\tens{A} - z \tens{I})^{-1} = \frac{1}{n}\sum_{i=1}^n\frac{1}{\lambda_i - z},
\end{align}
where $z\in\mathbb{C}\setminus\{\lambda_i\}_{i=1}^n$ and $\lambda_i$ is the $i$-th eigenvalue of $\tens{A}$. 
It is generally assumed that $\tens{A}$ is an element of a sequence $(\tens{A}_n)_{n=1}^{\infty}$ where each $\tens{A}_i$ is contained in $\tens{A}_j$ for $j > i$, $\tens{A}_n \in \mathbb{R}^{n \times n}$, and as $n \to \infty$, $m_{\tens{A}_n}(z) \to m(z)$ pointwise. 
Here, $m$ is  the \emph{Stieltjes transform} of a limiting \emph{spectral measure} $\nu$ in the sense that \(m: \mathbb{C} \setminus \supp(\nu) \to \mathbb{C}\) satisfies 
    $m(z) \coloneqq \int_{\mathbb{R}} \frac{1}{x - z} \, \nu(\mathrm{d} x).$ 
The R-transform $R(z)$ corresponding to $m$ is the (possibly multi-valued) function satisfying $R(-m(z)) = z + \frac{1}{m(z)}$ for $z \in \mathbb{C} \setminus \supp(\nu)$. 

To model the process of matrix subsampling, let $\sigma_n$ denote a uniformly random permutation of $\{1,\dots,n\}$ and for $1 \leq k \leq n$, let $\tens{A}_n^{(k)} = (\tens{A}_{\sigma_n(i)\sigma_n(j)})_{i,j=1}^k$. Then for any subsampling factor $\tau < 1$, $\mathbb{E}m_{\tens{A}_n^{([\tau n])}}(z) \to m_\tau(z)$ as $n \to \infty$. \citet{AMELI-2025b} assert that the map $m \mapsto m_\tau$ is equivalent to the map that operates on the R-transform by $R(z) \mapsto R(\tau z)$ via the Nica-Speicher Theorem \citep{NICA-1996}.
However, this identification does not hold for arbitrary matrix sequences: it is justified when the matrix and the sampling projection are asymptotically free, with orthogonal invariance providing a sufficient, but not necessary, condition (see \Cref{sec:nica} for a proof).%
\footnote{Orthogonal invariance arises naturally in some ML models, such as Gaussian Gram matrices of the form $\tens{Z}^{\intercal} \gtens{\Sigma} \tens{Z}$ with iid Gaussian $\tens{Z}$. Related covariance approximations arise for certain high-dimensional kernel matrices \citep{EL_KAROUI-2010}, although these do not guarantee invariance in general. Our experiments suggest that FD can remain effective even beyond this sufficient setting.}
When $\tau < 1$, as is typical in free probability theory, this is called \bemph{free compression}. However, this operation can be formally inverted by taking $\tau > 1$, moving from the Stieltjes transform of the submatrix to that of the full matrix; when well-defined, this inverse operation is \bemph{free decompression}.


While a number of obvious strategies exist to implement free decompression, they each have significant drawbacks. The multi-valued nature of the R-transform, and the need to eventually recover $m$, make it difficult to work with directly. The approach taken by \citet{AMELI-2025b} relies on numerically solving a PDE in $m$, where the major challenges faced are accurate curve-fitting of the approximated spectral distribution of the initial matrix, and analytic continuation of the Stieltjes transform across the $x$-axis. Reparameterization of this PDE results in an inviscid complex Burgers equation, and while this presents another potential approach, it does not seem straight-forward to implement in practice. While the formation of singularities and shocks can be expected for equations of this type, we have identified a family of distributions for which the flow exists for all time; see \Cref{sec:decompressibility}. For further discussion on these difficulties, we refer the reader to \Cref{sec:bg-pde-hard}.

\section{Implementing the Polynomial Method} \label{sec:poly}

An especially wide array of operations in random matrix theory, including free additive and multiplicative convolution, can be distilled into operations acting on various \emph{transforms} of the Stieltjes transform, including the R-transform, the S-transform, and the moment transform. We have seen that free decompression is defined in terms of the R-transform. \bemph{These transforms are all related to each other by algebraic relations.} This is inconvenient for performing direct numerical conversions between the transforms themselves \cite{OLVER-2012, CHEN-2026}. However, it is \emph{very convenient} if the Stieltjes transform is already defined through an implicit algebraic relation, as deriving relations for each of the transforms and their inverses offers no further complexity. This algebraic viewpoint is central to the polynomial method introduced by \citet{RAO-2008}. At the heart of our approach is the following ansatz.


\begin{assumption}\label{ass:polynomial}
    The Stieltjes transform \(m\) is \emph{algebraic}: there exists a nonzero bivariate polynomial \(P: \mathbb{C}\times\mathbb{C}\to\mathbb{C}\) such that for any $z \in \mathbb{C} \setminus \mathcal{K}$ where $\mathcal{K}$ is a finite set of singularities,
\begin{equation}
    P(z, m(z)) = 0. \label{eq:polynomial}
\end{equation}
\end{assumption}
We write $d_z$ for the degree of this polynomial in $z$, and $s$ for its degree in $m$. In addition to the quadratic examples previously handled by \citet{AMELI-2025b}, this unlocks, for example, models of Hessian matrices and neural network gram matrices \citep{PENNINGTON-2017a, PENNINGTON-2017c}, which we demonstrate in \Cref{sec:results}.
Under \Cref{ass:polynomial}, Cauchy, R-, S- and moment transforms, as well as common algebraic operations and stochastic transformations, can also be realized as the roots of corresponding polynomials.

Despite its convenience, \citet{RAO-2008} highlight fundamental obstacles to a general numerical implementation of the polynomial method. While $m$ is easily estimated, finding $P$ is nontrivial. Polynomial roots are sensitive to their coefficients, and because valid $P$ for algebraic Stieltjes transforms are rare among bivariate polynomials, na\"{i}ve polynomial fits often radically fail. Furthermore, identifying the correct numerical solution for $m$ from $P$ alone is nearly impossible. Consequently, to our knowledge, no general numerical implementations of the polynomial method exist. In the following sections, we address these challenges and provide its first general numerical implementation.


\paragraph{Fitting the Bivariate Polynomial.}
To fit \(P\), we first choose its degrees in \(z\) and \(m\), denoted by \(d_z\) and \(s\), respectively. These degrees can be selected post hoc using standard model-selection criteria \cite[Section 6.1.3]{JAMES-2013}; in practice, we increase them until the fitting error stabilizes while avoiding overparameterization. Our degree-sensitivity experiment in \Cref{tab:fl-degree} shows that choosing too low a degree can fail to capture the correct bulk topology, while modest over-specification is well tolerated. By exploiting the conjugate symmetry of the Stieltjes transform, we show that its irreducible defining polynomial must have real coefficients up to an overall scalar (\Cref{lem:poly-real}). Thus, we write
\begin{equation}
    P(z, m) \coloneqq \sum_{i=0}^{d_z} \sum_{j=0}^s c_{ij} z^i m^j, \qquad c_{ij} \in \mathbb{R}.
    \label{eq:poly-real}
\end{equation}
To estimate the coefficients \(c_{ij}\), we need samples of the Stieltjes transform. Unlike the density-fitting approach of \citet{AMELI-2025b}, which first fits the initial density and then computes its Stieltjes transform, we obtain these samples directly from the empirical Stieltjes transform \eqref{eq:EmpStieltjes}, computed from the observed submatrix eigenvalues. We choose sampling points \(z\) on Bernstein ellipses of varying radii around each spectral bulk. For spectra spanning many orders of magnitude, we found that ordinary ellipses did not sample the bulks effectively, so we developed contours adapted to logarithmic scales. We further impose linear constraints on the coefficients of \(P\) using the Stieltjes transform's expansion at infinity, including its normalization \(m(z) \sim -1/z\) and empirical spectral moments when available. These constraints enforce necessary conditions for \(P\) to admit a Stieltjes-transform branch. The fitting procedure is detailed in \Cref{sec:implicit}, with pseudocode given in \Cref{alg:poly}.



\paragraph{The Physical Branch.}
While algebraic manipulations of \(P\) are straightforward, a central numerical challenge is identifying, for each \(z \in \mathbb{C}\), which root of \(P(z, \cdot)\) is the physical Stieltjes value. The physical branch is characterized globally by the Herglotz property \(\Im m(z) > 0\) for \(z \in \mathbb{C}^{+}\) and by the normalization \(m(z) = -z^{-1} + O(z^{-2})\) at infinity. However, these conditions do not by themselves provide a stable pointwise root-selection rule. In high-degree or multi-bulk examples, roots can come close, exchange ordering, or pass near branch points, and relying only on pointwise tests can easily lead to jumps between sheets.


\paragraph{Branch Selection via Continuation.}
We therefore select the physical branch by geometric continuation on the algebraic spectral curve \(P(z, m) = 0\). Starting from an anchor point where the physical root can be identified reliably, either from the normalization at infinity or by comparison with the empirical Stieltjes transform, we continue the solution along a path in the \(z\)-plane. Along a smooth portion of the curve, implicit differentiation gives
\begin{equation*}
    \frac{\mathrm{d} m}{\mathrm{d} z} = - \frac{\partial_z P(z, m)}{\partial_m P(z, m)}.
\end{equation*}
This ordinary differential equation respects the local geometry of the sheet and provides a predictor for transporting the physical root from the anchor to the query point. The predictor is then corrected by Newton projection onto \(P(z, m) = 0\), with step subdivision when the geometry becomes stiff. This continuation procedure is substantially more stable than independent pointwise root selection and allows the fitted polynomial to be used as a numerical Stieltjes transform even for the high-degree, multi-scale curves used in our experiments; details are given in \Cref{sec:root-selection}.


\subsection{Free Decompression on Algebraic Spectral Curves} \label{sec:fd-on-sc-main-body}

Since the R-transform that underpins free decompression satisfies an algebraic relation whenever \Cref{ass:polynomial} holds, the Stieltjes transform after free decompression also satisfies an algebraic relation. This is stated in our main Theorem, proved as \Cref{thm:FormalFreeDecomp} in \Cref{sec:formal}.

\begin{theorem}
    \label{thm:SpectralCurve}
    Let $m$ be a Stieltjes transform satisfying \Cref{ass:polynomial}, $\tau\geq1$ be a decompression ratio, and $m_\tau$ be the free decompression of $m$ by the factor $\tau$. Then every point $(z, m_\tau(z)) \in \mathbb{C}^2$ satisfies
    \begin{align}\label{eq:fd-poly}
        P\left(z + \frac{\tau - 1}{\tau m_\tau}, \tau m_\tau\right) = 0.    
    \end{align}
\end{theorem}

Especially noteworthy here is the lack of any further assumptions on the nature of \(m\) or on the explicit construction of the R-transform. \Cref{ass:polynomial} naturally allows free decompression to be viewed through a different lens, characterizing it as a family of rational reparameterizations of the spectral curve \(P(z, m) = 0\). This characterization is convenient, as it circumvents the manual analytic continuation that plagued the PDE-based approach in \citet{AMELI-2025b}. Indeed, for each fixed \(z\) away from the singular set, the Fundamental Theorem of Algebra guarantees that \(P(z, m) = 0\) has \(s\) solutions in \(m\), up to multiplicity. Viewing \(m(z)\) as a multivalued algebraic function, these roots represent the sheets of the analytic continuation of the Stieltjes transform. The physical Stieltjes transform is one such sheet, and free decompression transports it along the spectral curve. This geometric viewpoint is developed in \Cref{sec:fd-on-sc}; \Cref{fig:mp-branches} illustrates the sheet structure and the gluing of branches in the Marchenko--Pastur case.

\begin{figure*}[!t]
    \begin{minipage}[b]{0.45\textwidth}
        \centering
        \includegraphics[width=\textwidth]{\figdir/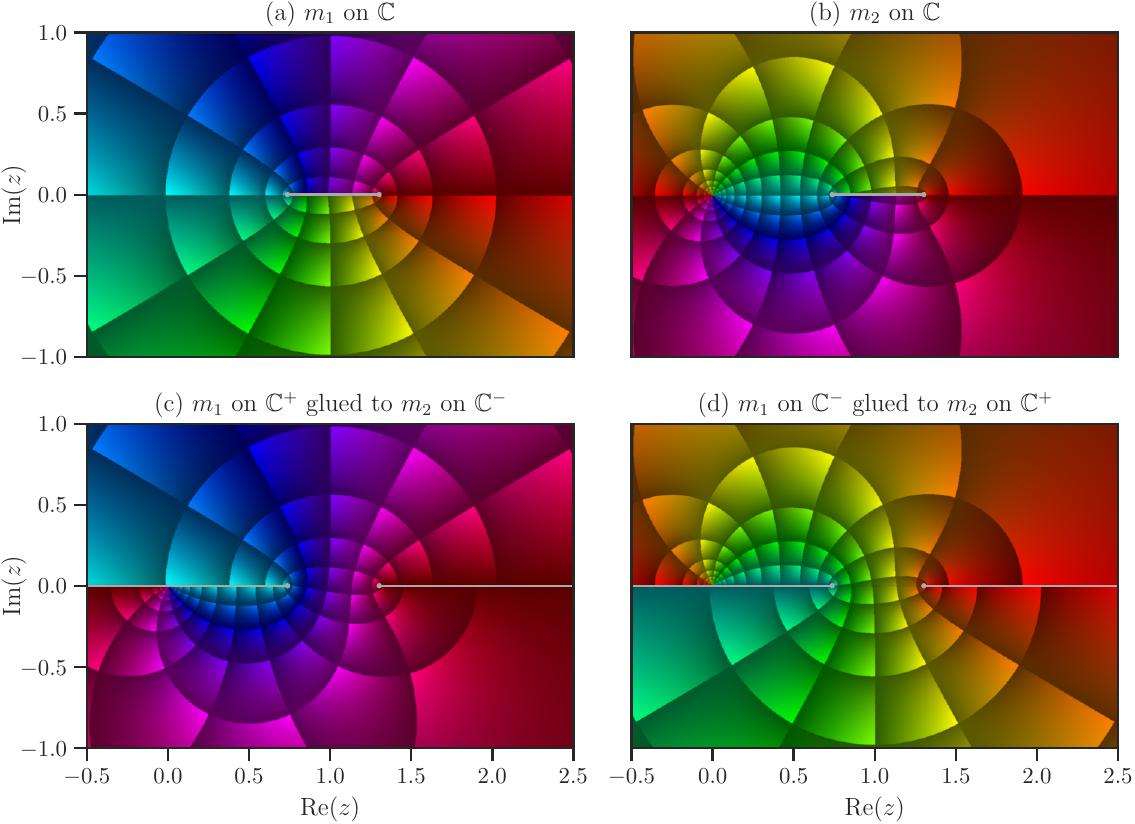}
    \end{minipage}
    \hfill %
    \begin{minipage}[b]{0.54\textwidth}
        \centering
        \raisebox{2mm}{\includegraphics[width=\textwidth]{\figdir/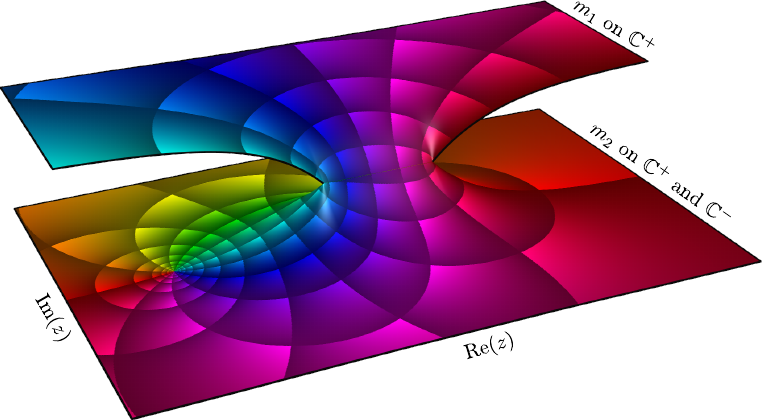}}
    \end{minipage}
    \caption{Visualization of multiple solutions to the algebraic relation \eqref{eq:polynomial} for the Marchenko--Pastur density with $\lambda = \frac{1}{50}$ (left), and the surface obtained by gluing the Stieltjes solution $m_1$ to the secondary branch $m_2$ (right); colors are interpreted as in \Cref{sec:notation}. Free decompression requires tracing families of continuous paths along this surface, from the top sheet down to the lower sheet.}
    \label{fig:mp-branches}
    \ifjournal
        \vspace{-1mm}
    \fi
\end{figure*}


\paragraph{Evolving the Physical Branch.}
The evolved algebraic relation in \Cref{thm:SpectralCurve}, and its polynomial form in \Cref{cor:CoefficientMethod}, gives all candidate roots of \(m_{\tau}(z)\), but it does not identify the physical branch. To select the physical root, we track the known physical sheet at \(\tau = 1\) along the free-decompression flow. For a fixed query point \(z\), we work in the original spectral-curve coordinates \((\zeta, y)\), where \(P(\zeta, y) = 0\), and impose the characteristic relation
\begin{equation*}
    \zeta - (\tau - 1)y^{-1} - z = 0.
\end{equation*}
Equivalently, \((\zeta, y)\) is determined by the coupled system
\begin{equation*}
    F_1(\zeta, y) \coloneqq P(\zeta, y) = 0,
    \qquad
    F_2(\zeta, y; \tau, z) \coloneqq \zeta - (\tau - 1)y^{-1} - z = 0.
\end{equation*}
At \(\tau = 1\), this system reduces to \(\zeta = z\) and \(y = m_0(z)\), where \(m_0(z)\) is already known from the physical-branch selection procedure above. Differentiating the system with respect to \(\tau\), while keeping \(z\) fixed, gives the tangent equation
\begin{equation*}
    \begin{bmatrix}
        \partial_{\zeta} P(\zeta, y) & \partial_y P(\zeta, y) \\
        1 & (\tau - 1)y^{-2}
    \end{bmatrix}
    \begin{bmatrix}
        \dot{\zeta} \\
        \dot{y}
    \end{bmatrix}
    =
    \begin{bmatrix}
        0 \\
        y^{-1}
    \end{bmatrix}.
\end{equation*}
This ordinary differential equation on the spectral curve provides the predictor direction for transporting the physical sheet from one decompression ratio to the next. The predicted point is then corrected by Newton projection onto the nonlinear system \((F_1, F_2) = (0, 0)\), with step subdivision when the local geometry becomes stiff. After correction, the decompressed Stieltjes value is recovered as \(m_{\tau}(z) = \tau^{-1} y\). The complete geometric predictor--corrector procedure is given in \Cref{sec:fd-on-sc}, while \Cref{alg:summary} provides a high-level overview of the end-to-end density decompression workflow.

\begin{algorithm}[ht!]
    \caption{Algebraic free decompression from a random principal submatrix}
    \label{alg:summary}
    \SetKwInOut{Input}{Input}
    \SetKwInOut{Output}{Output}
    \DontPrintSemicolon

    \vspace{1mm}
    \Input{
        Eigenvalues \(\{\lambda_i\}_{i=1}^{n_0}\) of an accessible random principal submatrix; \\
        Grid of decompression ratios \(1 = \tau_0 < \tau_1 < \cdots < \tau_T\); \\
        Polynomial degrees \((d_z, s)\) and optional submatrix spectral moments \(\{\mu_p\}_{p=0}^{r}\); \\
        Spectral grid \(\{x_j\}_{j=1}^{N_x}\) and small imaginary offsets \(0 < \delta_1 < \cdots < \delta_L\).
    }

    \Output{
        Target density \(\hat{\rho}_{\tau_{T}}\), optionally the path
        \(\{\hat{\rho}_{\tau_\ell}\}_{\ell=0}^{T}\).
    }

    Construct contours \(\Gamma\) around the bulks. Sample \(\hat{m}_1(z) \leftarrow \frac{1}{n_0} \sum_{i=1}^{n_0} (\lambda_i - z)^{-1}\) for \(z \in \Gamma\).\;
    
    Fit \(\hat{P}(z, m) = 0\) to samples with a stabilized basis and moment constraints (\Cref{alg:poly}).\;

    \ForEach{\(z = x_j + \mathrm{i} \delta_k\)}{
        Choose an anchor \(z_{\mathrm a} \in \mathbb{C}^{+}\) and the root of \(\hat{P}(z_{\mathrm a}, m) = 0\) nearest \(\hat{m}_1(z_{\mathrm a})\); continue it geometrically along the path from \(z_a\) to \(z\) to obtain \(m_1(z)\) (\Cref{alg:stieltjes}).\;

        Initialize \((\zeta, y) \leftarrow (z, m_1(z))\) and \(m_{\tau_0} \leftarrow m_1(z)\). \;
    
        \For{\(\ell = 1, \dots, T\)}{
            Evolve \((\zeta, y)\) from \(\tau_{\ell-1}\) to \(\tau_{\ell}\) on \(\hat{P}(\zeta, y)=0\) with \(z = \zeta - (\tau - 1) y^{-1}\), via tangent prediction, damped Newton correction, and adaptive subdivision (\Cref{alg:decompress}). \;
    
            Set \(m_{\tau_\ell}(z) \leftarrow \tau_{\ell}^{-1} y\). \;
        }
    }

    \ForEach{\((\tau_{\ell}, x_j)\)}{
        Extrapolate \(\{\frac{1}{\pi} \Im \, m_{\tau_\ell}(x_j + \mathrm{i} \delta_k)\}_{k=1}^{L}\) to \(\delta = 0\), obtaining \(\hat{\rho}_{\tau_\ell}(x_j)\) (\Cref{sec:inv-stieltjes}). \;
    }
    
    \Return{\(\hat{\rho}_{\tau_{T}}\), \textnormal{optionally }
    \(\{\hat{\rho}_{\tau_{\ell}}\}_{\ell = 0}^{T}\)}\;
\end{algorithm}


\subsection{Evolution of Spectral Quantities}
\label{sec:quantities}

The algebraic framework does more than decompress the full spectral density: it also lets us evolve spectral quantities directly, without computing the density itself. These include atom weights, spectral moments, bulk edges, and gaps between bulks. We can track how these quantities change with matrix size and even identify where and when bulks split.



\paragraph{Atoms.}
An atom of a distribution at a point $x_{\circ}$ corresponds to a simple pole of the Stieltjes transform $m(z)$ at $z=x_{\circ}$, whose weight is given in terms of its residue $w = \lim_{z\to x_{\circ}}(x_{\circ} - z)m(z).$ Under \Cref{ass:polynomial}, writing $P(z,m) = \sum_{i=0}^s a_i(z)m^i$, a necessary condition for the existence of a pole at finite real $x_{\circ}$ is that $x_{\circ}$ is a root of the leading coefficient, i.e., $a_s(x_{\circ})=0$. When \(x_\circ\) is a simple root of \(a_s\), the residue gives \(w(x_{\circ}) = a_{s-1}(x_\circ) / a_s'(x_{\circ})\). Under free decompression, the atom's location stays fixed, and we can compute how this weight evolves with \(\tau\). This is depicted in \Cref{fig:cfp-atom}, showing the proportion of probability mass occupied by the atomic component as the matrix decompresses.
The following proposition gives the precise dynamics, with proof in \Cref{sec:atoms}.

\begin{proposition} \label{prop:atom-short}
    Let \(x_{\circ} \in \mathbb{R}\) be an atom of \(m(z)\) with initial mass \(w_{\circ} \in (0, 1]\). Under the free decompression flow, the atom location remains fixed at \(x_{\circ}(\tau) = x_{\circ}\), while its weight evolves by \(w(\tau) = 1 - \tau^{-1}(1 - w_{\circ})\).
\end{proposition}



\paragraph{Spectral Edges.}
We can also compute how the boundaries of spectral bulks evolve under free decompression, including when bulks split. This turns out to be a simpler task than accurately estimating the entire spectral density, and can help diagnose and correct poor density approximations. We assume that the spectral distribution consists of finitely many atoms and absolutely continuous bulk components. Writing the support of the latter as \(I(\tau) = \bigcup_{j=1}^{k} I_j(\tau)\), where \(I_j(\tau) = [a_j(\tau),b_j(\tau)]\), the spectral edges are the endpoints \(\{a_j(\tau),b_j(\tau)\}_{j=1}^{k}\). Under \Cref{ass:polynomial}, these points have a precise characterization in terms of the \emph{branch points} of the decompressed spectral curve. At such points, different sheets coalesce, satisfying
\begin{align*}
    P(\zeta,y) = 0,
    \qquad \text{and} \qquad
    y^2 \frac{\partial P}{\partial y}(\zeta, y) - (\tau-1)\frac{\partial P}{\partial \zeta}(\zeta, y) = 0, \qquad \zeta \in \mathbb{R},
\end{align*}
where $y$ belongs to the physical branch, as derived in \Cref{prop:fd-branch-point} in \Cref{sec:edge}. To identify which candidates are spectral edges, we use continuation from known points on the physical branch, similar to the method described in \Cref{sec:fd-on-sc}. At \(\tau = 1\), we identify roots on the physical branch through anchored continuation, and for \(\tau > 1\) we use the roots identified at the previous time-step. An algorithmic description of a procedure for finding the spectral edges is given in \Cref{alg:edge}. When a spectral edge approaches zero, effects of finite matrix size can become significant for near-zero eigenvalues. We investigate this phenomenon in \Cref{sec:finite-correction}, deriving a Tracy--Widom-type scaling regime from a finite-size spectral-curve ansatz.


\paragraph{Merging and Splitting Bulks.}
As distributions evolve under free decompression, a connected bulk may split into separate bulks. Viewed in reverse, this is the merger of neighboring bulks under free compression. We can identify where and when these transitions occur. Geometrically, regular transitions correspond to second-order critical points of the decompressed spectral projection, satisfying
\begin{gather*}
    P(\zeta_{\ast}, y_{\ast}) = 0, \qquad
    y_{\ast}^2\ \partial_{y} P - (\tau_{\ast} - 1)\ \partial_{\zeta} P = 0,\qquad\text{and} \\
    y_{\ast} \left((\partial_{\zeta \zeta} P) (\partial_y P)^2 - 2 (\partial_{\zeta y} P)  (\partial_{\zeta} P) (\partial_y P) + (\partial_{yy} P) (\partial_{\zeta} P)^2\right) + 2 (\partial_{\zeta} P)^2 (\partial_y P) = 0.
\end{gather*}
These relations are derived in \Cref{prop:cusp} in \Cref{sec:merge}. We also derive equations for the evolution of bulk intervals and the gaps between them in \Cref{sec:gap}.


\paragraph{Moments.}
We can also evolve the spectral moments \(\mu_n(\tau) = \int x^n \nu_\tau(\mathrm{d}x)\) directly from their initial values, without computing the density. Starting from the free-decompression PDE of \citet{AMELI-2025b}, we derive a triangular system of ODEs for these moments (\Cref{lem:moment}). Solving this system gives \(\mu_n(\tau) = \sum_{k=0}^{n-1} \kappa_{n,k}\tau^k\), where the coefficients \(\kappa_{n,k}\) depend only on the first \(n\) initial moments and are computed recursively in \Cref{prop:moment-decompression}. Notably, this evolution does not require \Cref{ass:polynomial}. For algebraic distributions, \Cref{prop:AlgebraicMomentFormula} also provides an efficient \(\mathcal{O}(n^2)\) recurrence for computing the first \(n\) moments directly from \(P\), without root selection.

\begin{figure}[t]
    \centering
    \includegraphics[width=\textwidth]{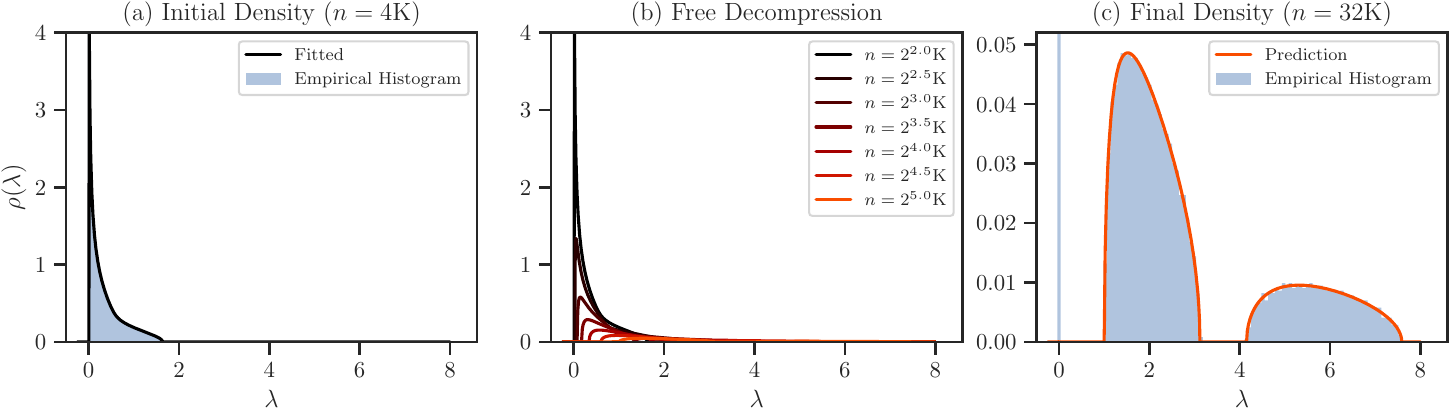}
    \ifjournal
        \vspace{-6mm}
    \fi
    \caption{Example of free decompression for recovering the ESD of a \(\num{32000} \times \num{32000}\) matrix, generated from the compound free Poisson law described in \Cref{sec:results}. (a) ESD of the sampled \(\num{4000} \times \num{4000}\) submatrix (histogram) and the fitted density (solid curve). (b) Free decompression flow over $n$. (c) Final predicted ESD at size \(n = \num{32000}\), compared with the ESD of the full matrix.}
    \label{fig:cfp-flow}
\end{figure}


\section{Empirical Results} \label{sec:results}

\begin{figure}[t]
    \centering
    \includegraphics[width=\textwidth]{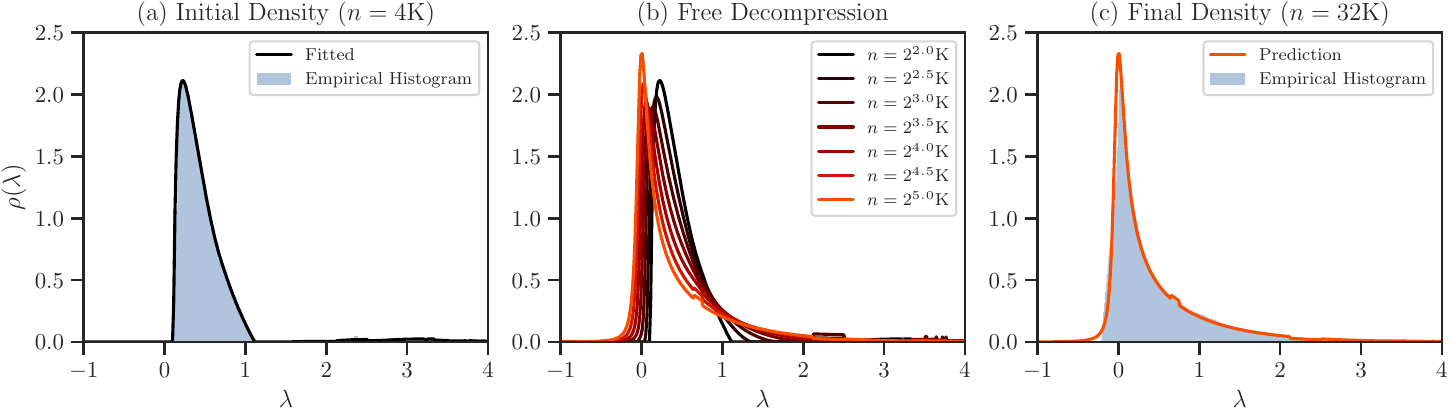}
    \caption{Example of free decompression for recovering the ESD of a \(\num{32768} \times \num{32768}\) matrix, generated from the neural network Hessian described in \Cref{sec:results}. (a) ESD of the sampled \(\num{4096} \times \num{4096}\) submatrix (histogram) and the fitted density (solid curve). (b) Free decompression flow from \(n_0\) to \(n\). (c) Final predicted ESD at size \(n = \num{32768}\), compared with the ESD of the original matrix.}
    \label{fig:hessian_flow}
    \ifjournal
        \vspace{-4mm}
    \fi
\end{figure}

We now consider three examples that demonstrate the ability to decompress distributions of interest in ML. First, we consider a compound free Poisson law containing an atom and bulk regions that split under decompression. Then, we consider the Hessian of a two-layer neural-network autoencoder at initialization, related to the Pennington--Bahri model of limiting Hessian spectra at critical points of the MSE loss \citep{PENNINGTON-2017c}. Finally, we consider the expected activation matrix of a random-features diffusion model. Further experiments with the free L\'evy and Pennington--Bahri models are presented in \Cref{sec:levy,sec:pennington-bahri}.


\paragraph{Compound Free Poisson.}
Our first example belongs to the family of compound free Poisson distributions (see \Cref{sec:compound-poisson} for details). For aspect ratios \(0 < \lambda \le 1\), this law can be expressed as \(\big((1-\lambda) \delta_0 + \lambda \nu\big) \boxtimes \mu_{\mathrm{MP}_1}\), where \(\nu\) is the jump distribution. In our example, we take \(\nu(\mathrm{d}x) = \frac{3}{4} \delta_2(\mathrm{d}x) + \frac{1}{4} \delta_{5.5}(\mathrm{d}x)\) and \(\lambda = 0.1\). The resulting distribution contains an atom (spike) at zero and two absolutely continuous bulk regions. We generate a \(\num{32000} \times \num{32000}\) matrix from this model and sample a \(\num{4000} \times \num{4000}\) principal submatrix as initial data. At this smaller size, the fitted spectral density consists of a single bulk. Free decompression from this submatrix is depicted in \Cref{fig:cfp-flow}, where the bulk splits into two as the matrix size increases. This example clearly demonstrates both the atom dynamics and bulk-splitting behavior described in \Cref{sec:quantities}, as shown in \Cref{fig:cfp-atom}. We further demonstrate free decompression for the more general free L\'evy law in \Cref{sec:levy}, where a single spectral bulk splits into three.


\paragraph{Neural Network Hessian.}
Our second example is the Hessian matrix of the \(\ell^2\)-loss of a two-layer neural-network autoencoder with CIFAR-10 data, evaluated at initialization. This is related to the free-probabilistic Hessian approximation of \citet{PENNINGTON-2017c}, who studied a similar but smaller autoencoder at critical points of the MSE loss. They analyzed the two components of the Hessian \(\tens{H} = \tens{H}_0 + \tens{H}_1\) separately, where \(\tens{H}_0\) is the squared-Jacobian component and \(\tens{H}_1\) is the network Hessian term. We examine their theoretical model separately in \Cref{sec:pennington-bahri}, while details of our network architecture, data, and Hessian construction are given in \Cref{sec:hessian-example}.

\begin{figure}[t]
    \centering
    \includegraphics[width=0.54\textwidth]{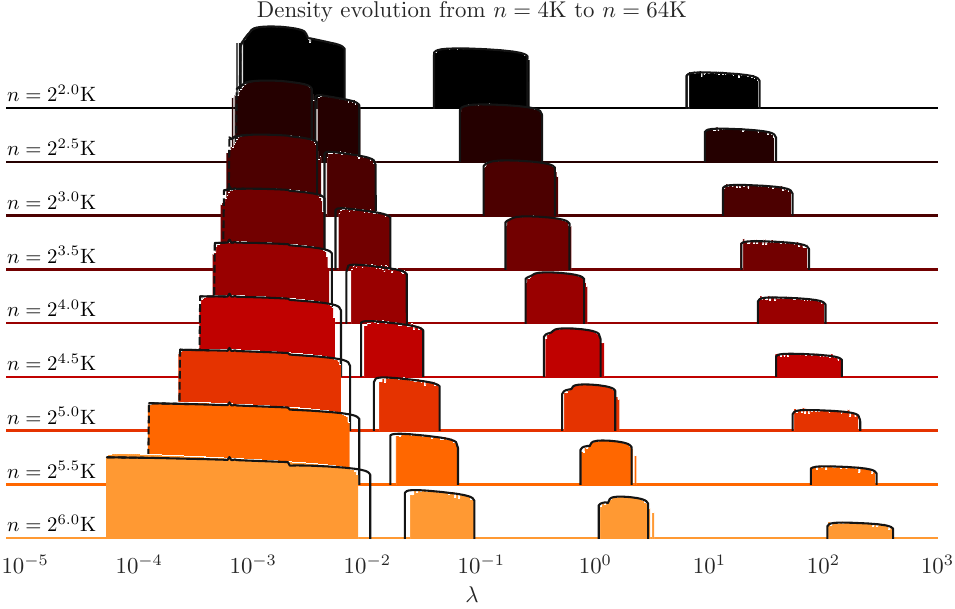}%
    \hfill%
    \includegraphics[width=0.45\textwidth]{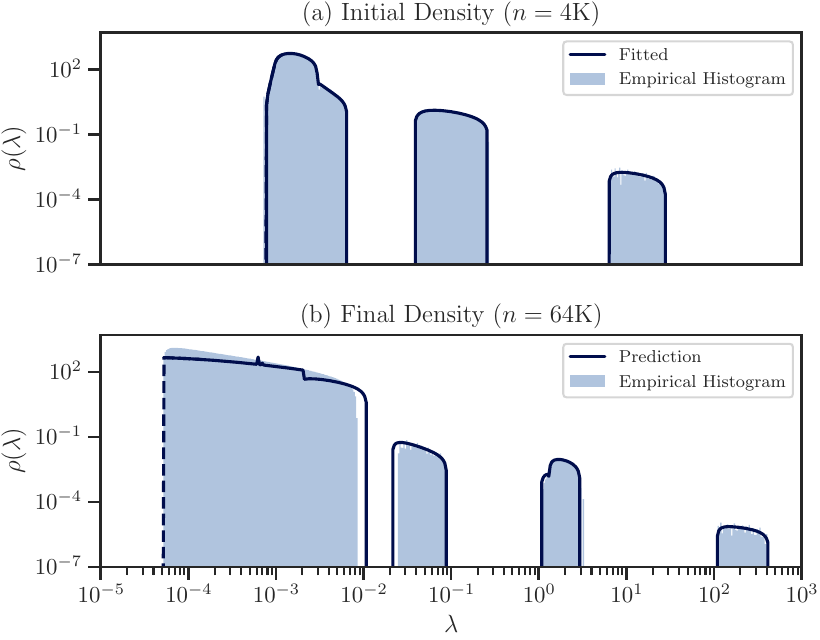}
    \ifjournal
        \vspace{-2mm}
    \fi
    \caption{Free decompression for recovering the ESD of an \(n \times n\) activation matrix \(\tens{A}\) of a random-features diffusion model with \(n = \num{64000}\) random feature weights. \emph{Left:} ridge plot of the FD density evolution on log scale (solid lines), compared with empirical spectral distributions of sampled submatrices (colored histograms). \emph{Right:} (a) ESD of the sampled \(n_0 \times n_0\) matrix \(\tens{A}_0\), with \(n_0 = \num{4000}\), together with the fitted density; and (b) final predicted ESD at \(n = \num{64000}\), compared with the ESD of the original matrix \(\tens{A}\). The \(y\)-axis is logarithmic in all panels; the leftmost narrow bulk is atom-like and stands \emph{ten billion} times higher than the right bulk.}
    \label{fig:diff-ridge}
    \ifjournal
        \vspace{-3mm}
    \fi
\end{figure}

For this experiment, the full \(\num{32768} \times \num{32768}\) matrix was formed, and a \(\num{4096} \times \num{4096}\) submatrix was sampled to form the initial empirical spectral density. After fitting the bivariate polynomial, we applied free decompression to recover the full-matrix spectral density, as shown in \Cref{fig:hessian_flow}. In \Cref{sec:hessian-example}, we use the decompressed distributions to accurately track the normalized Hessian index (the fraction of negative eigenvalues) across matrix sizes. We remark that while the full Hessian was computed as a baseline for our experiments, subsamples of Hessian matrices too large to practically compute can be readily obtained by masking using standard automatic differentiation tools.


\begin{wrapfigure}{R}{0.58\textwidth}
    \centering
    \ifjournal
        \vspace{-4mm}
    \fi
    \includegraphics[width=0.58\textwidth]{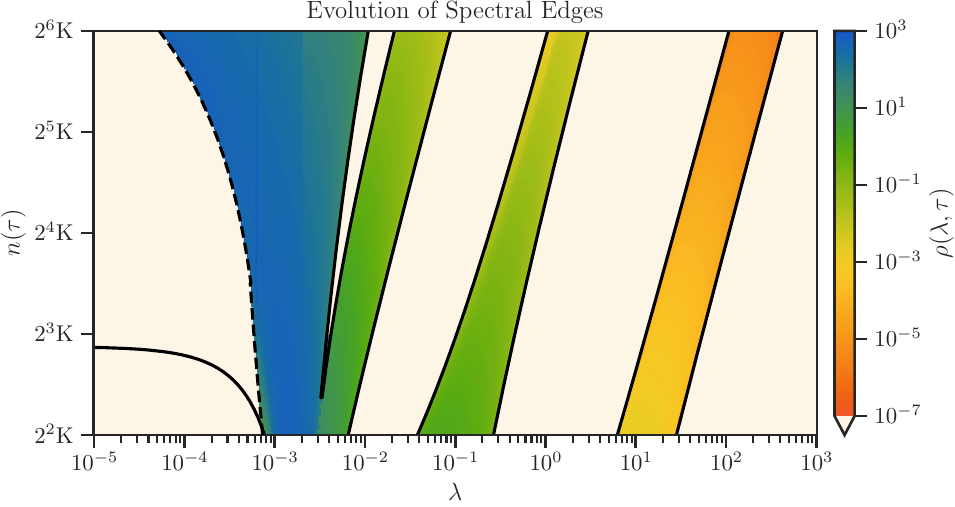}
    \ifjournal
        \vspace{-7mm}
    \fi
    \caption{Evolution of the diffusion-model spectral density (colors, on a logarithmic scale), with independently evolved spectral edges (solid black curves). The edges closely follow the bulk boundaries and identify a cusp where the leftmost bulk splits. The dashed curve shows the finite-size estimate of the far-left edge, whose empirical evolution differs from the FD prediction near zero (\Cref{sec:finite-correction}). The blue bulk exhibits atom-like behavior.}
    \label{fig:diff-edge}
    \ifjournal
        \vspace{-4mm}
    \fi
\end{wrapfigure}

\paragraph{Diffusion Model.}
Our third example considers the expected activation matrix of a diffusion model studied in \citet{GEORGE-2025}; see \Cref{sec:bg-diffusion} for model and matrix details. Its large-scale limit, studied in \citet{BONNAIRE-2025} in connection with memorization--generalization transitions, has a highly multi-scale spectrum with several bulks, including a narrow atom-like bulk away from zero, making this a stringent log-scale test case for branch selection and density recovery.

To perform free decompression, we formed a \(\num{64000} \times \num{64000}\) matrix and sampled a \(\num{4000} \times \num{4000}\) submatrix. The eigenvalues of the subsampled matrices were pooled to obtain a more accurate initial ESD. \Cref{fig:diff-ridge} (left) shows ridge plots of the FD evolution up to \(64\)K, overlaid with empirical spectra from subsamples at each scale. The right panels isolate the initial and final densities on the same vertical scale, showing that the final FD prediction matches the empirical spectrum across the full log-scale range.


Despite the enormous range in both eigenvalues and spectral density, a single fitted polynomial captures and evolves all the bulks. The narrow leftmost bulk is particularly interesting: although it resembles a smeared-out atom, we fit and evolve it as a continuous component alongside the other bulks, without modeling it separately as an atom. Remarkably, its center remains nearly fixed and its mass closely follows the atomic evolution law derived in \Cref{sec:quantities}, as quantified in \Cref{tab:diffusion-atom}.

We also compute the spectral edges independently of the density evolution. In \Cref{fig:diff-edge}, these edges (solid curves) closely follow the boundaries of the decompressed density (colored regions) and identify where the leftmost bulk splits. The left edges of the non-atomic bulks also matter for ML, as \citet{BONNAIRE-2025} relate them to distinct dynamical timescales of generalization and memorization during diffusion-model training.

The far-left edge of the narrow atom-like bulk presents a different phenomenon. Its empirical location is the smallest eigenvalue of a finite matrix, which at microscopic scales near zero does not follow the macroscopic FD prediction. We investigated this discrepancy further and found a Tracy--Widom-type scaling regime at intermediate matrix sizes (\Cref{sec:finite-correction}).

\begin{table}[t!]
\caption{Comparison of direct eigendecomposition with free decompression (FD) for the diffusion-model random-feature matrix. The Direct column reports the process time for computing the eigenvalues at each matrix size. In the FD column, the first term is the time to compute eigenvalues of the initial \(2^{2}\mathrm{K}\)-sized submatrix, and the second term is the FD evolution time. Distributional distances compare the empirical spectrum with the FD reconstruction, with \(W_1/L\) denoting the Wasserstein-1 distance on log-eigenvalues normalized by the log-width \(L\) of the plotted spectral range, and MMD denoting the Gaussian-kernel maximum mean discrepancy on log-eigenvalues. Moment errors are relative errors of the first moment and standard deviation. Log edge errors compare empirically detected spectral edges with FD-evolved edges on the log-\(\lambda\) axis, excluding the unstable left edge of the spike-like leftmost bulk. All reported errors are multiplied by \(100\) and shown as percentages.}
\label{tab:diffusion-complexity}
\vspace{1mm}
\centering
\ifjournal
    \begin{adjustbox}{width=\textwidth}
\else
\fi
\begin{tabular}{lrrrrrrrr}
    \toprule
    Size & \multicolumn{2}{c}{Process Time (sec)} & \multicolumn{2}{c}{Distributional Distance} & \multicolumn{2}{c}{Moments Rel. Error} & \multicolumn{2}{c}{Log Edges Error} \\
    \cmidrule(lr){2-3} \cmidrule(lr){4-5} \cmidrule(lr){6-7} \cmidrule(lr){8-9}
    \(n\) & Direct & FD (ours) & \(W_1/L\) & MMD & \(\Delta \mu_1 /\mu_{1}\) & \(\Delta \sigma / \sigma\) & Mean & Max \\
    \midrule
    \(2^{2}\)K & \(     32.2\) & \(32.2 + 0.0\) & \(0.01\%\) & \(0.01\%\) & \(0.01\%\) & \(0.01\%\) & \(0.1\%\) & \(0.2\%\) \\
    \(2^{3}\)K & \(   431.43\) & \(32.2 + 2.3\) & \(0.22\%\) & \(0.08\%\) & \(0.02\%\) & \(0.03\%\) & \(0.2\%\) & \(0.5\%\) \\
    \(2^{4}\)K & \(   3534.5\) & \(32.2 + 2.9\) & \(0.43\%\) & \(0.12\%\) & \(0.01\%\) & \(0.00\%\) & \(0.2\%\) & \(0.2\%\) \\
    \(2^{5}\)K & \(  27056.2\) & \(32.2 + 4.3\) & \(0.70\%\) & \(0.15\%\) & \(0.05\%\) & \(0.08\%\) & \(0.4\%\) & \(0.8\%\) \\
    \(2^{6}\)K & \( 192447.7\) & \(32.2 + 6.8\) & \(2.19\%\) & \(0.53\%\) & \(0.08\%\) & \(0.09\%\) & \(0.4\%\) & \(1.2\%\) \\
    \bottomrule
\end{tabular}
\ifjournal
    \end{adjustbox}
\else
\fi
\end{table}

We also compare FD with stochastic Lanczos quadrature (SLQ) on the same diffusion-model matrix (\Cref{tab:diffusion-slq}). The two methods operate under different information requirements: SLQ accesses the full \(64\)K target through matrix--vector products, whereas FD observes only \(4\)K principal submatrices. Even with this difference, FD achieves lower errors in both metrics at matched runtime. With over three times the runtime, SLQ achieves comparable accuracy. Moreover, FD recovers the entire spectral evolution across matrix sizes, including densities, edges, gaps, atoms, and bulk transitions, rather than only the spectrum of a fixed target.

\begin{table}[t!]
\caption{Comparison of free decompression (FD) and stochastic Lanczos quadrature (SLQ) for the \(64\)K diffusion-model matrix, on the same hardware. SLQ uses matrix--vector products with the full target, while FD uses only smaller principal submatrices. SLQ (\(p \times \ell\)) denotes \(p\) independent Rademacher probes with \(\ell\) Lanczos steps each. Distributional errors use the same metrics as \Cref{tab:diffusion-complexity}.}
\label{tab:diffusion-slq}
\vspace{1mm}
\centering
\begin{tabular}{llrrrr}
    \toprule
    Method & Observed input & Time (sec) & Target matvecs & \(W_1/L\) & MMD \\
    \midrule
    FD                       & \(4\)K submatrix    & \( 39.0\) & \(   0\) & \(2.19\%\) &  \(0.53\%\) \\
    SLQ (\(8  \times  60\))  & Full \(64\)K target & \( 39.1\) & \( 480\) & \(5.79\%\) & \(11.52\%\) \\
    SLQ (\(16 \times 200\))  & Full \(64\)K target & \(121.3\) & \(3200\) & \(2.69\%\) & \( 0.47\%\) \\
    \bottomrule
\end{tabular}
\end{table}


\section{Conclusions}
\label{sec:conclusion}
We have developed a new method of free decompression that generalizes previous work by representing the Stieltjes transform of the spectral density through a bivariate polynomial of arbitrary degree. Under this algebraic ansatz, free decompression becomes an evolution on spectral curves, allowing us to track the physical branch, evolve the density, and derive equations for spectral quantities such as atoms, moments, and bulk edges. This enables FD to handle substantially more challenging spectra than previous approaches, including multi-modal, multi-bulk, and multi-scale distributions arising in realistic ML matrices.

\paragraph{Limitations.}
While the current method expands the range of matrices accessible to free decompression, some important regimes remain challenging. In particular, heavy-tailed spectral distributions remain elusive, as their unbounded support and possibly divergent moments lie beyond our current setting. Moreover, the underlying FD law is justified under asymptotic freeness between the matrix and the sampling projection, which need not hold for arbitrary ML matrices. For empirical spectra, the fitted polynomial serves as an algebraic surrogate, and obtaining a reliable fit and tracking its physical branch can remain numerically challenging. We hope these results provide a foundation for further development of methods for large-scale spectral inference.


\newcommand{\acknowledgmenttext}
{
    LH is supported by the Australian Research Council through a Discovery Early Career Researcher Award (DE240100144). MWM acknowledges partial support from DARPA, DOE, NSF, and~ONR. This work was supported in part by the U.S. Department of Energy, Office of Science, Office of Advanced Scientific Computing Research's Applied Mathematics Competitive Portfolios program under Contract No. AC02-05CH11231.
}

\newcommand{\competinginterest}
{
    The authors declare no competing interests.
}

\ifjournal
    \newpage
\else
\fi

\ifjournal
   \section*{Acknowledgments}
   \acknowledgmenttext

   \section*{Competing Interests}
   \competinginterest
\else
    \vspace{0.4cm}
    \noindent\textbf{Acknowledgments.}
    \acknowledgmenttext

    \vspace{0.4cm}
    \noindent\textbf{Competing interests.}
    \competinginterest
\fi


\ifjournal
\else
    \newpage
\fi

\ifjournal
    \bibliographystyle{apalike2}
\else
    \bibliographystyle{apalike2}
\fi
{
\bibliography{refs}
}


\newpage

\ifjournal
    \onecolumn
\fi

\appendix
\begin{appendices}

\addtocontents{toc}{\protect\setcounter{tocdepth}{2}}  
\tableofcontents

\section{Notation and Conventions} \label{sec:notation}

\paragraph{Nomenclature.}
We use boldface lowercase letters for vectors, boldface upper case letters for matrices, and normal face letters for scalars, including the components of vectors and matrices. \Cref{tab:nomenclature} summarizes the main symbols and notations used throughout the paper, organized by context.

\begin{table}[hpt!]
\centering

\caption{Common notations used throughout the manuscript.} 
\label{tab:nomenclature}

\begin{tabular}{lll}
    
    \toprule
    Context & Symbol & Description \\
    \midrule
    \multirow{2}{*}{\makecell[l]{Matrices \& \\ Models}}
                  & \(\mathbf{A}, \mathbf{A}_n\) & General \(n \times n\) matrix (often symmetric or Hermitian) \\
                  & \(\mathbf{A}_n^{(k)}\) & \(k \times k\) principal submatrix of \(\mathbf{A}_n\) \\
                  & \(\mathbf{I}\) & The identity matrix \\
                  & \(\lambda_i\) & The \(i\)-th eigenvalue of a matrix \\
                  & \(\operatorname{trace}(\cdot)\) & Standard trace of a matrix \\
                  & \(\operatorname{tr}_n(\cdot)\) & Normalized trace, \(n^{-1} \operatorname{trace}(\cdot)\) \\
                  & \(\mathbf{Q}_n\) & Haar-distributed random orthogonal matrix \\
                  & \(\mathbf{H}, \mathbf{J}\) & Hessian matrix (and subcomponents) and Jacobian matrix \\
                  & \(\mathbf{\Sigma}\) & Population covariance matrix \\
                  & \(\mathbf{W}, \mathbf{X}, \mathbf{Z}\) & Random matrices with i.i.d. entries (e.g., Wigner, Gaussian) \\
                  & \(\mathbf{U}\) & Expected diffusion activation matrix (random features) \\
                  & \(\lambda, c\) & Aspect ratios (\(p/n\)) or intensities/rates for random matrix models \\
    \addlinespace[1mm]  
    \arrayrulecolor{gray!30}\cline{1-3} 
    \addlinespace[1mm]
    \arrayrulecolor{black}
    Free Probability
                  & \(\nu\) & Spectral probability measure on \(\mathbb{R}\), possibly including atoms \\
                  & \(\rho\) & Density of the absolutely continuous part of \(\nu\), \(\rho = \mathrm{d} \nu_{\mathrm{ac}} / \mathrm{d}x\), \(\mathbb{R} \to \mathbb{R}^{+}\) \\
                  & \(I\) & Compact support of the density \(\rho\), \(I = \operatorname{supp}(\rho) \subsetneq \mathbb{R}\) \\
                  & \(m\) & Stieltjes transform of measure, \(m(z) = \int_{\mathbb{R}} \frac{\nu(\mathrm{d} y)}{y - z}\), \(\mathbb{C} \setminus \operatorname{supp}(\nu) \to \mathbb{C}\) \\
                  & \(\omega\) & Functional inverse of the Stieltjes transform, \(\omega(m(z)) = z\) \\
                  & \(\mathcal{H}\{\nu\}\) & Hilbert transform of measure, \(\mathcal{H}\{\nu\}(x) = \frac{1}{\pi} \operatorname{p.v.} \int_{\mathbb{R}} \frac{\nu(\mathrm{d} y)}{x - y}\), \(\mathbb{R} \to \mathbb{R}\) \\
                  & \(R\) & Voiculescu's R-transform, \(R(z) = \omega(-z) - \frac{1}{z}\) with \(\omega(m(z)) = z\) \\
                  & \(S\) & S-transform, \(S(z) = \frac{1+z}{z}\psi^{-1}(z)\), where \(\psi(z) = \int_{\mathbb{R}} \frac{zy}{1 - zy} \, \nu(\mathrm{d} y)\) \\
                  & \(\nu^{\boxplus s}\) & \(s\)-fold free additive convolution power of \(\nu\) \\
                  & \(D_c\nu\) & Dilation of measure \(\nu\) by \(c\), i.e., push-forward under \(x \mapsto cx\) \\
                  & \(M(z)\) & Truncated moment-generating function \\
                  & \(\chi(z)\) & Functional inverse of the truncated moment-generating function \\
                  & \(r_k(\nu)\) & \(k\)-th free cumulant of the measure \(\nu\) \\
                  & \(\mu_p, \mu_p(t)\) & \(p\)-th moment of a probability measure (and its time evolution) \\
                  & \(\boxplus\) & Free additive convolution \\
                  & \(\boxtimes\) & Free multiplicative convolution \\
    \addlinespace[1mm]  
    \arrayrulecolor{gray!30}\cline{1-3} 
    \addlinespace[1mm]
    \arrayrulecolor{black}
    \multirow{2}{*}{\makecell[l]{Algebraic\\Curve}}
                  & \(P(z,m)\) & Algebraic equation for the Stieltjes transform, \(P(z,m(z)) = 0\) \\
                  & \(\mathcal{C}\) & Complex 1D spectral curve \(\mathcal{C} = \{(z, m) \in \mathbb{C}^2 \mid P(z, m) = 0\}\) \\
    \addlinespace[1mm]  
    \arrayrulecolor{gray!30}\cline{1-3} 
    \addlinespace[1mm]
    \arrayrulecolor{black}
    \multirow{2}{*}{\makecell[l]{Free\\Decompression}}
                  & \(n\) & Size of the target (large) matrix \\
                  & \(n_0\) & Size of the sampled (small) sub-matrix, \(n_0 < n\) \\
                  & \(\tau, t\) & Decompression scale, \(\tau = \frac{n}{n_0}\), with \(t = \log(\tau)\) \\
                  & \(\rho_0, m_0\) & Density and Stieltjes transform at \(t=0\) \\
    \addlinespace[1mm]  
    \arrayrulecolor{gray!30}\cline{1-3} 
    \addlinespace[1mm]
    \arrayrulecolor{black}
    Notation     & \(\mathbb{C}^{\pm}\) & Upper (lower) half complex plane \\
                 & \(\Re, \Im\) & Real and imaginary part of complex variable \\
                 & \(\operatorname{p.v.}\) & Cauchy principal value \\
    \bottomrule
\end{tabular}
\end{table}

\paragraph{Color code for plots.}
In our complex-plane plots, colors encode the values of a complex quantity using \emph{domain coloring}: hue represents phase, while brightness represents magnitude; see the plain version in \Cref{fig:domain-coloring-convention}(b). For high-dynamic-range fields, we also use the tiled version shown in \Cref{fig:domain-coloring-convention}(a), in which phase and logarithmic magnitude are periodically reset, following the phase-portrait convention of Wegert~\cite[Section~2.5]{WEGERT-2012}.

\begin{figure}[h!]
    \centering
    \includegraphics[width=\textwidth]{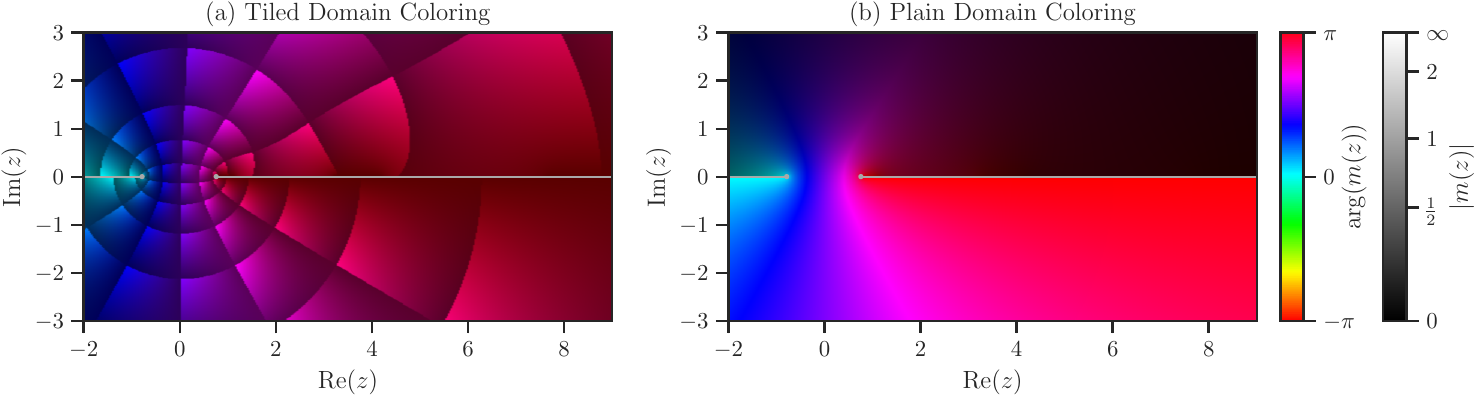}
    \caption{Domain coloring of a Stieltjes-transform branch for the Free L\'{e}vy example in \Cref{fig:fl-branches}. Right: plain domain coloring, with hue encoding \(\arg m(z)\) and brightness encoding \(|m(z)|\). Left: tiled domain coloring, where phase and logarithmic magnitude are periodically reset to make high-dynamic-range structure more visible.}
    \label{fig:domain-coloring-convention}
\end{figure}


\section{Background} \label{sec:bg}


\subsection{Stieltjes Transform and Boundary Values} \label{sec:bg-stieltjes}

Let \(\nu\) be a probability measure on \(\mathbb{R}\) with compact support \(\supp(\nu) \subset \mathbb{R}\). Let \(\nu_{\mathrm{ac}}\) denote its absolutely continuous part with respect to Lebesgue measure, and let \(\rho(x) \coloneqq \nu_{\mathrm{ac}}(\mathrm{d} x) / \mathrm{d} x\) denote its density, given by the Radon--Nikodym derivative.

The \emph{Stieltjes transform} of \(\nu\) is the function \(m: \mathbb{C} \setminus \supp(\nu) \to \mathbb{C}\) defined by
\begin{equation}
    m(z) \coloneqq \int_{\mathbb{R}} \frac{1}{x - z} \, \nu(\mathrm{d} x), \quad z \in \mathbb{C} \setminus \supp(\nu). \label{eq:def-stieltjes}
\end{equation}
The Stieltjes transform is a \emph{Herglotz} function: \(\Im m(z) > 0\) for \(z \in \mathbb{C}^{+}\). Here, \(\mathbb{C}^{+} \coloneqq \{z \in \mathbb{C} \mid \Im z > 0\}\), and \(\Im w\) denotes the imaginary part for \(w \in \mathbb{C}\). Moreover, \(m\) satisfies the \emph{normalization} at infinity: \(m(z) = -z^{-1} + O(z^{-2})\) as \(\lvert z \rvert \to \infty\).

The function \(m\) admits non-tangential boundary values on the real line, denoted by \(m^{\pm}(x) \coloneqq \lim_{y \to 0^{\pm}} m(x + i y)\) when the limits exist. If \(\nu(\mathrm{d} x) = \rho(x) \, \mathrm{d} x\) in a neighborhood of \(x\), the Sokhotski--Plemelj formula \citep[p.~42]{MUSKHELISHVILI-1992} yields
\begin{equation*}
    m^{\pm}(x) =  -\pi \mathcal{H}\{\nu\}(x) \pm i \pi \rho(x),
\end{equation*}
where \(\mathcal{H}\{\nu\}(x) \coloneqq \frac{1}{\pi} \, \mathrm{p.v.} \int_{\mathbb{R}} \frac{\nu(\mathrm{d} t)}{x - t}\) is the \emph{Hilbert transform} of the full measure \(\nu\) and \(\mathrm{p.v.}\) denotes the Cauchy's principal value of a singular integral. The density can be recovered from boundary values by
\begin{equation}
    \rho(x) = \pi^{-1} \Im m^{+}(x) = - \pi^{-1} \Im m^{-}(x),
    \label{eq:inv-stieltjes}
\end{equation}
which is the inversion of the Stieltjes transform.

The jump of \(m\) across the real axis is
\begin{equation*}
    m^{+}(x) - m^{-}(x) = 2 i \pi \rho(x).
\end{equation*}
In particular, for \(x \notin \supp(\nu)\), the boundary values coincide, and \(m\) is holomorphic across the real axis at \(x\). In contrast, for \(x \in \supp(\rho)\) the boundary values differ, which manifests as a branch cut in the boundary values of \(m\) across \(\supp(\rho)\). An example of a branch cut is shown in \Cref{fig:mp-branches}(a), illustrated by the white line segment in the center of the plot. In the remainder, we denote the support of the absolutely continuous part of \(\nu\) by \(I \coloneqq \supp(\rho)\), and we will typically assume that it consists of finitely many disjoint compact intervals,
\begin{equation*}
    I = \bigcup_{j=1}^{k} I_j, \qquad I_j \coloneqq [a_j, b_j],
\end{equation*}
where \(k \geq 1\), and \(a_j < b_j < a_{j+1}\). We refer to \(k\) as the number of spectrum bulks, which is also the number of branch cuts of \(m\). Also, any atomic component of \(\nu\) produces isolated poles of \(m\) on the real axis rather than branch cuts, and will be treated separately when relevant.


\subsection{Analytic Continuation Beyond the Physical Branch} 
\label{sec:bg-continuation}

Several applications, and in particular the method of characteristics used in free decompression \citep{AMELI-2025b}, require evaluating the Stieltjes field \(m(z)\) along trajectories in the complex \(z\)-plane that cross the cuts \(I\). Since \(I \subseteq \supp(\nu)\) and the physical Stieltjes transform \(m(z)\) is holomorphic on \(\mathbb{C} \setminus \supp(\nu)\), such evaluations require analytic continuation across \(I\). It is therefore natural to regard \(m\) as belonging to a larger object: a multi-valued analytic function of \(z\), whose single-valued realizations correspond to branches (or sheets) of a Riemann surface lying over the \(z\)-plane.

We denote the Stieltjes transform defined by \eqref{eq:def-stieltjes} by \(m_1\) and refer to it as the \emph{physical branch}, characterized uniquely by the Herglotz property together with the normalization at infinity. Other branches, when they exist, are denoted by \(m_i\), \(i > 1\), and are referred to as \emph{non-physical}.

Crossing a cut exchanges boundary traces of branches. Specifically, for \(x\) in the interior of an interval \(I_j\), analytic continuation across \(I_j\) identifies the upper trace of one branch with the lower trace of another, and vice versa. Equivalently, two sheets are glued along \(I_j\) by matching boundary values with opposite orientations, so that a trajectory approaching \(I_j\) from \(\mathbb{C}^{+}\) and crossing the cut may emerge on a different branch on the other side; see \Cref{fig:mp-branches} for an illustration.

In our setting, the physical branch \(m_1\) is assumed known, either because \(\nu\) is a model law or because it is estimated from an empirical spectrum and \(m_1\) is computed by \eqref{eq:def-stieltjes}. Analytic continuation of an exact branch along a specified path is unique whenever it exists. However, continuation beyond the physical domain can be path-dependent. In particular, the phrase ``the other branch'' is generally ambiguous: continuations around different endpoints \(\{a_j, b_j\}_{j=1}^{k}\) can lead to distinct sheets, and this ambiguity becomes more pronounced for multi-bulk spectra.

Numerically, recovering these continuations from approximate samples of the physical branch remains severely ill-posed. A convenient way to regularize this problem is to restrict attention to model classes in which the analytic continuation of \(m\) has only finitely many branches. This reduces the problem to a finite-sheeted branched covering of the \(z\)-plane. We adopt an algebraic model class, which we describe next.


\section{Algebraic Stieltjes Transforms} \label{sec:bg-alg-stieltjes}

Our structural assumption is that the Stieltjes transform \(m\) is \emph{algebraic}: there exists a nonzero bivariate polynomial \(P\) such that
\begin{equation}
    P(z, m) = 0. \label{eq:poly-appendix}
\end{equation}
We write \(s \coloneqq \deg_{m}(P)\) for the degree of \(P\) with respect to its second argument. For each \(z\) away from the branch locus, the algebraic equation \(P(z, m) = 0\) admits \(s\) (not necessarily distinct) solutions in \(m\), and the corresponding analytic branches are denoted by \(\{m_i(z)\}_{i=1}^{s}\). In this notation, \(m_1\) is the physical branch, singled out among the \(s\) branches by the Herglotz property together with the normalization at infinity, while the remaining branches are non-physical.

\emph{Branch points} occur at those \(z_{\ast}\) for which two or more roots of \(P\) collide. In particular, there exists \(m_{\ast} \in \mathbb{C}\) such that
\begin{equation}
	P(z_{\ast}, m_{\ast}) = 0, \qquad \partial_{m} P(z_{\ast}, m_{\ast}) = 0. \label{eq:branch-condition}
\end{equation}
Eliminating \(m_{\ast}\) yields the discriminant \(\Delta(z) \coloneqq \mathrm{Disc}_m P(z, \cdot) \in \mathbb{R}[z]\), whose zeros provide candidates for the branch locus of the multi-valued function defined implicitly by \eqref{eq:poly-appendix}. In Stieltjes settings with conjugation symmetry, the branch points \(z_{\ast}\) on the real line include the spectral edges \(\{a_j, b_j\}_{j=1}^{k}\) of the cut set \(I\), though the algebraic model may admit additional branch points away from the real axis.

A \emph{branch cut} \(I\) is a choice of curves connecting branch points such that each branch \(m_i\) becomes single-valued and holomorphic on \(\mathbb{C} \setminus I\). Once \(I\) is fixed, each branch admits non-tangential boundary values on \(I\), denoted by \(m_i^{\pm}\). Across a cut component, analytic continuation can only permute branches. It is therefore convenient to collect branches into a vector,
\begin{equation*}
	\vect{m}(z) \coloneqq [m_1(z), m_2(z), \dots, m_s(z)]^{\intercal}.
\end{equation*}
For \(x \in I_j^{\circ}\), analytic continuation across \(I_j\) permutes the boundary traces, hence there exists a permutation matrix \(\gtens{\Pi}_j \in \{0, 1\}^{s \times s}\) such that
\begin{equation*}
	\vect{m}^{+}(x) = \gtens{\Pi}_j \, \vect{m}^{-}(x), \qquad x \in I_j^{\circ},
\end{equation*}
where \(\vect{m}^{\pm}(x)\) denote non-tangential boundary values taken from \(\mathbb{C}^{\pm}\). The collection \(\{\gtens{\Pi}_j\}_{j=1}^{k}\) encodes how the \(s\) sheets are glued along the real cuts and provides a finite description of the continuation rules needed to follow trajectories across \(I\).

In \Cref{sec:implicit}, we describe the construction of a finite-sheet continuation from samples of the physical branch \(m_1\) by fitting \eqref{eq:poly-appendix}. Before that, in \Cref{sec:moments}, we describe moment-based necessary conditions for the feasibility of such an algebraic representation.


\subsection{Moments from the Algebraic Relation} \label{sec:moments}

One of the biggest challenges with the algebraic approach to performing free decompression is the identification of the root $m(z)$ in the algebraic relation that corresponds to the principal branch $m_1(z)$ of the Stieltjes transform we seek. This identification can be performed in one of two ways. The first is by homotopy, where one takes advantage of continuity and identifies the correct root at time $t + \delta t$ using knowledge of the correct root at time $t$. This approach performs well when the roots are distinct, but can fail in regimes where the roots become close or ``cross over''. The second approach involves the computation of \emph{moments}; while this is less direct, it is arguably more concrete and provides a powerful check to ensure the correct solution is obtained.

Let \(\mu_p\) denote the \(p\)-th moment of the probability measure $\nu$
\begin{equation}
    \mu_p \coloneqq \int_{\mathbb{R}} x^p \nu(\mathrm{d} x), \qquad p \in \mathbb{N}_{0}, \label{eq:def-moments}
\end{equation}
where \(\mu_0 = 1\) corresponds to unit mass normalization. If $m_1$ denotes the Stieltjes transform of $\nu$, then for sufficiently large $|z|$, $m_1$ satisfies the Laurent expansion
\begin{equation}
    \label{eq:MomentExpansion}
    m_1(z)=\int_{\mathbb{R}}\frac{1}{x-z}\nu(\mathrm{d}x)=-\frac{1}{z}\int_{\mathbb{R}}\frac{1}{1-\frac{x}{z}}\nu(\mathrm{d}x)=-\frac{1}{z}\sum_{p=0}^{\infty}\frac{1}{z^{p}}\int_{\mathbb{R}} x^{p}\nu(\mathrm{d}x)=-\sum_{p=0}^{\infty}\frac{\mu_{p}}{z^{p+1}},
\end{equation}
providing a concrete characterisation of the Stieltjes transform in terms of its moments. 

Suppose that a polynomial parameterized by an index set $\mathcal{A} \subset \mathbb{N}_0^2$, $P(z,m)=\sum_{(i,j) \in \mathcal{A}} c_{ij}z^{i}m^{j}$, is given. Our objective is to find a solution $m(z)$ satisfying $P(z,m(z))=0$ and an expansion of the form $m(z)=\sum_{k=0}^{\infty}\vartheta_{k}z^{-k-1}$ for all sufficiently large $|z|$. Let $w=1/z$ and $S(w)=\sum_{k=0}^{\infty}\vartheta_{k}w^{k}$ so that $m(z)=wS(w)$. If $m(z)$ corresponds to the principal branch $m_1(z)$ with moments $\mu_{k}$, then $\vartheta_{0}=-1$ and $\mu_{k}=\vartheta_{k}/\vartheta_{0}=-\vartheta_{k}$ for each $k=1,2,\dots$. 



\begin{lemma}
    \label{lem:PolyNu0}
    There exists a function \(m(z)\) satisfying \(P(z, m(z)) = 0\) and \(m(z) = \sum_{k=0}^ \infty \vartheta_k z^{-k-1}\) for sufficiently large \(|z|\) only if \(\vartheta_0\) is a root of the polynomial
    \begin{equation*}
        L(\vartheta) = \sum_{\substack{(i, j) \in \mathcal{A} \\ i-j = e_{\mathrm{max}}}} c_{ij} \vartheta^{j},
        \qquad
        e_{\mathrm{max}} = \max\{i-j: (i,j) \in \mathcal{A}, \; c_{ij} \neq 0\}.
    \end{equation*}
    Conversely, if \(L(\vartheta_0) = 0\) and \(L'(\vartheta_0) \neq 0\), there exists a unique such function with prescribed coefficient \(\vartheta_0\) near infinity.
\end{lemma}

\begin{proof}
    Setting \(w = 1/z\) and \(m(z) = w S(w)\), the equation \(P(z, m(z))=0\) becomes
    \begin{equation*}
        F(w, S) \coloneqq \sum_{(i, j) \in \mathcal{A}} c_{ij} w^{e_{\mathrm{max}}-(i-j)} S^j = 0.
    \end{equation*}
    Since \(F(0, S) = L(S)\), any such expansion must satisfy \(L(\vartheta_0) = 0\). Conversely, if \(L'(\vartheta_0) \neq 0\), the analytic implicit function theorem gives a unique analytic solution \(S(w)\) near \(w = 0\), with \(S(0) = \vartheta_0\). Therefore, \(m(z) = z^{-1} S(1/z)\) admits the required Laurent expansion.
\end{proof}

\begin{example*}
    As an example, consider the polynomial $P(z,m)=\lambda\sigma^{2}zm^{2}+(z-\sigma^{2}(1-\lambda))m+1$ for the Marchenko--Pastur distribution with scale $\sigma^{2}$ and aspect ratio $\lambda$. In this case, $e_{\mathrm{max}}=0$ and $L(\vartheta)=1+\vartheta$, which does indeed have a single root at $\vartheta=-1$. 
\end{example*}

In many practical cases, a root of $P$ will only \emph{approximate} a Stieltjes transform, and so it may only be the case that $\vartheta_0 \approx -1$. Fortunately, $\vartheta_0$ can be estimated by finding the closest root of $L(\vartheta)$ to $-1$. As long as the remaining coefficients $\vartheta_k$ can be recovered, the ``moments'' of the corresponding Stieltjes transform can be estimated by $\mu_k \approx \vartheta_k / \vartheta_0$. Recovery of $\vartheta_k$ is trickier, but can be achieved using \Cref{prop:AlgebraicMomentFormula} under the condition that $L'(\vartheta_0) \neq 0$.

\begin{proposition}
    \label{prop:AlgebraicMomentFormula}
    Suppose \(L(\vartheta_0) = 0\) and \(L'(\vartheta_0) \neq 0\). Then, for \(k \geq 1\), the coefficients \(\vartheta_k\) can be computed using the recurrence relations
    \begin{align*}
        d_{j,k} &= \begin{cases}
            \vartheta_0^j &\quad\text{if } j \geq 0,\;k=0,\\
            0 & \quad \text{if }j=0,\; k \geq 1,\\
            \sum_{l=0}^k \vartheta_l d_{j-1,k-l} & \quad \text{otherwise,}
        \end{cases}
        \\
        \tilde{d}_{j,k} &= \begin{cases}
            0 & \quad \text{if } j=0,\; k\geq1,\\
            \vartheta_0 \tilde{d}_{j-1,k}
            + \displaystyle\sum_{\ell=1}^{k-1} \vartheta_\ell d_{j-1,k-\ell}
            & \quad \text{if } j\geq1,\; k\geq1,
        \end{cases}
    \end{align*}
    where \(d_{j,q} = 0\) for \(q < 0\), and
    \[
        \vartheta_k = -\frac{1}{L'(\vartheta_0)}\left(\sum_{\substack{(i, j) \in \mathcal{A} \\ i-j=e_{\mathrm{max}}}}c_{ij}\tilde{d}_{j,k}
        + \sum_{p<e_{\mathrm{max}}}\sum_{\substack{(i, j) \in \mathcal{A} \\ i-j=p}}c_{ij}d_{j,k-e_{\mathrm{max}}+p}\right),
    \]
    with
    \[
        L'(\vartheta) = \sum_{\substack{(i, j) \in \mathcal{A} \\ i-j=e_{\mathrm{max}}}} j c_{ij} \vartheta^{j-1}.
    \]
\end{proposition}

\begin{proof}
    Let \(w = 1/z\) and \(S(w) = \sum_{k=0}^{\infty} \vartheta_k w^k\), so that \(m(z) = w S(w)\). Let \(d_{j,k}\) denote the coefficient of \(w^k\) in \(S(w)^j\), so that
    \begin{equation*}
        S(w)^j = \sum_{k=0}^{\infty} d_{j,k} w^k.
    \end{equation*}
    The Cauchy product immediately gives the stated recurrence for \(d_{j,k}\).

    To isolate the unknown coefficient \(\vartheta_k\), define
    \begin{equation*}
        S_k(w) = \sum_{\ell=0}^{k-1} \vartheta_\ell w^\ell,
        \qquad k \geq 1.
    \end{equation*}
    Let \(\tilde{d}_{j,k}\) denote the coefficient of \(w^k\) in \(S_k(w)^j\). Since \(S_k(w)\) and \(S(w)\) agree through order \(k-1\), the Cauchy product gives, for \(j,k \geq 1\),
    \begin{equation*}
        \tilde{d}_{j,k}
        = \vartheta_0 \tilde{d}_{j-1,k}
        + \sum_{\ell=1}^{k-1} \vartheta_\ell d_{j-1,k-\ell}.
    \end{equation*}
    Moreover, \(S(w) = S_k(w) + \vartheta_k w^k + O(w^{k+1})\), and hence
    \begin{equation*}
        d_{j,k} = \tilde{d}_{j,k} + j \vartheta_0^{j-1} \vartheta_k.
    \end{equation*}

    Now, substituting \(m(z) = w S(w)\) into \(P(z,m(z)) = 0\) and multiplying by \(w^{e_{\mathrm{max}}}\), we obtain
    \begin{equation*}
        \sum_{(i,j) \in \mathcal{A}}
        c_{ij} w^{e_{\mathrm{max}}-(i-j)} S(w)^j = 0.
    \end{equation*}
    Equating the coefficients of \(w^k\) yields
    \begin{equation*}
        \sum_{\substack{(i,j) \in \mathcal{A} \\ i-j=e_{\mathrm{max}}}}
        c_{ij} d_{j,k}
        +
        \sum_{p<e_{\mathrm{max}}}
        \sum_{\substack{(i,j) \in \mathcal{A} \\ i-j=p}}
        c_{ij} d_{j,k-e_{\mathrm{max}}+p}
        = 0,
    \end{equation*}
    where \(d_{j,q} = 0\) for \(q < 0\). Substituting the expression for \(d_{j,k}\) and using
    \begin{equation*}
        L'(\vartheta_0)
        =
        \sum_{\substack{(i,j) \in \mathcal{A} \\ i-j=e_{\mathrm{max}}}}
        j c_{ij} \vartheta_0^{j-1},
    \end{equation*}
    gives
    \begin{equation*}
        \vartheta_k L'(\vartheta_0)
        +
        \sum_{\substack{(i,j) \in \mathcal{A} \\ i-j=e_{\mathrm{max}}}}
        c_{ij} \tilde{d}_{j,k}
        +
        \sum_{p<e_{\mathrm{max}}}
        \sum_{\substack{(i,j) \in \mathcal{A} \\ i-j=p}}
        c_{ij} d_{j,k-e_{\mathrm{max}}+p}
        = 0.
    \end{equation*}
    Since \(L'(\vartheta_0) \neq 0\), solving for \(\vartheta_k\) yields the stated recurrence.
\end{proof}


\section{Fitting the Spectral Curve} \label{sec:implicit}

In this section, we describe a general-purpose procedure for reconstructing non-physical branches of an algebraic Stieltjes transform from pointwise evaluations of the physical branch \(m_1\). The basic idea is to fit an implicit algebraic relation \(P(z, m) = 0\) as in \eqref{eq:poly-appendix} from samples \(\{(z_\ell, m_1(z_\ell))\}_{\ell=1}^n\). Once such a relation is available, the candidate branch values at a given \(z \in \mathbb{C} \setminus I\) are obtained as the \(s\) roots of the univariate polynomial \(m \mapsto P(z, m)\), and non-physical branches can be accessed and tracked by analytic continuation away from the branch locus.

Throughout, we assume that \(m\) is algebraic of degree \(s\) in the sense that there exists a nonzero bivariate polynomial \(P \in \mathbb{C}[z, m]\), the ring of bivariate polynomials in \((z, m)\) with complex coefficients, with \(\deg_m P = s\) such that \eqref{eq:poly-appendix} holds for \(m = m_1\) on \(\mathbb{C} \setminus I\). Among all such relations, it is natural to seek a minimal one in the sense of having minimal degree in \(m\), equivalently an irreducible \(P\) that does not contain spurious factors producing extraneous roots. Before turning to the numerical fitting procedure, we record structural constraints implied by the Stieltjes symmetry of \(m_1\), starting with the fact that \eqref{eq:poly-appendix} admits a real-coefficient representative.

\begin{lemma} \label{lem:poly-real}
    Assume that \(P \in \mathbb{C}[z, m]\) satisfies \eqref{eq:poly-appendix} for \(m = m_1\) and is irreducible in \(\mathbb{C}[z, m]\). Then there exists a nonzero scalar \(\alpha \in \mathbb{C}\) such that \(\alpha P \in \mathbb{R}[z, m]\). In particular, after an irrelevant rescaling of \(P\), one may assume \(P \in \mathbb{R}[z, m]\).
\end{lemma}

\begin{proof}
    Since the distribution \(\nu\) is real-valued, the Stieltjes transform \(m_1\) from \eqref{eq:def-stieltjes} satisfies the Schwarz symmetry
    \begin{equation*}
        m_1(\bar{z}) = \overline{m_1(z)}, \qquad z \in \mathbb{C} \setminus I,
    \end{equation*}
    where \(\bar{z}\) denotes the complex conjugate of \(z\). Define the conjugate-reflected polynomial \(P^{\sharp} \in \mathbb{C}[z, m]\) by
    \begin{equation*}
        P^{\sharp}(z, m) \coloneqq \overline{P(\bar{z}, \bar{m})}.
    \end{equation*}
    Then, for every \(z \in \mathbb{C} \setminus I\),
    \begin{equation*}
        P^{\sharp}(z, m_1(z)) = \overline{P(\bar{z}, \overline{m_1(z)})} = \overline{P(\bar{z}, m_1(\bar{z}))} = 0,
    \end{equation*}
    so \(P^{\sharp}\) also annihilates \(m_1\) identically on \(\mathbb{C} \setminus I\). By irreducibility, \(P\) and \(P^{\sharp}\) must agree up to a nonzero scalar factor, hence there exists \(\lambda \in \mathbb{C} \setminus \{0\}\) such that \(P^{\sharp} = \lambda P\). Applying the \(\sharp\) operation twice yields \(P = (P^{\sharp})^{\sharp} = \bar{\lambda} \lambda P\), so \(\lvert \lambda \rvert = 1\). Choose \(\alpha \in \mathbb{C} \setminus \{0\}\) such that \(\alpha / \bar{\alpha} = \lambda\) and set \(\tilde{P} \coloneqq \alpha P\). Then \(\tilde{P}^{\sharp} = \bar{\alpha} P^{\sharp} = \bar{\alpha} \lambda P = \alpha P = \tilde{P}\), which implies that \(\tilde{P}\) has real coefficients, that is, \(\tilde{P} \in \mathbb{R}[z, m]\).
\end{proof}

\begin{remark} \label{rem:gauge}
    The rescaling in \Cref{lem:poly-real} does not affect the roots of \(m \mapsto P(z, m)\) for any fixed \(z\), and has no impact on the reconstructed branches. A numerical normalization fixing the scalar ambiguity is described later in this section.
\end{remark}

Having established that \eqref{eq:poly-appendix} admits a real-coefficient representative, we henceforth assume \(P \in \mathbb{R}[z, m]\). To fit \(P\) from pointwise evaluations of \(m_1\), we restrict \(P\) to a finite-dimensional model class by expanding it in a bivariate monomial basis. Concretely, we fix an index set \(\mathcal{A} \subset \mathbb{N}_0^2\) and parameterize
\begin{equation}
    P(z, m) \coloneqq \sum_{(i, j) \in \mathcal{A}} c_{ij} z^i m^j, \qquad c_{ij} \in \mathbb{R}.
    \label{eq:poly-param}
\end{equation}
Let \(N \coloneqq \lvert \mathcal{A} \rvert\), and collect the coefficients into a vector \(\vect{c} \in \mathbb{R}^{N}\) according to a fixed ordering of \(\mathcal{A}\).

A common choice for \(\mathcal{A}\) is a rectangular support,
\begin{equation}
    \mathcal{A} \coloneqq \{0, 1, \dots, d_z\} \times \{0, 1, \dots, s\},
    \label{eq:index-rect}
\end{equation}
which corresponds to allowing degrees up to \(d_z\) in \(z\) and up to \(s\) in \(m\). More generally, the flexibility of \(\mathcal{A}\) allows us to incorporate structural constraints on \eqref{eq:poly-appendix} such as moment constraints at infinity as in \Cref{sec:constraint-moments} or reduced supports motivated by asymptotic cancellations.


\subsection{Constraints at Infinity from Moments} \label{sec:constraint-moments}

Assuming that a finite number of moments of the target probability measure \(\nu\) are available, the Laurent expansion of \(m_1\) at infinity induces linear constraints on the coefficients \(\{c_{ij}\}_{(i, j) \in \mathcal{A}}\). Fix \(r \in \mathbb{N}_0\) and define the moment vector \(\vect{\mu} \coloneqq (\mu_0, \mu_1, \dots, \mu_r) \in \mathbb{R}^{r+1}\).

\begin{definition} \label{def:convolution}
    For \(j \in \mathbb{N}_{0}\), define the \(j\)-fold discrete convolution
    \begin{equation*}
        \vect{\mu}^{\ast j} \coloneqq \underbrace{\vect{\mu} \ast \cdots \ast \vect{\mu}}_{j \text{ times}},
    \end{equation*}
    with the convention \(\vect{\mu}^{\ast 0} = [1]\). For \(q \in \mathbb{N}_0\), denote by \(\mu^{\ast j}_{q}\) the \(q\)-th entry of \(\vect{\mu}^{\ast j}\), equivalently
    \begin{equation*}
        \mu^{\ast j}_{q} \coloneqq \sum_{\substack{p_1 + \cdots + p_j = q \\ 0 \leq p_{\ell} \leq r}} \prod_{\ell=1}^j \mu_{p_{\ell}},
    \end{equation*}
    with the conventions \(\mu^{\ast 0}_{0} = 1\), \(\mu^{\ast 0}_q = 0\) for \(q \geq 1\), and \(\mu^{\ast j}_{q} = 0\) whenever \(q < 0\) or \(q > jr\).
\end{definition}

The following \Cref{prop:moments-constraint} provides linear constraints on the coefficients using known moments. We note that \Cref{prop:moments-constraint} is effectively a higher-order generalization of \Cref{lem:PolyNu0}. 

\begin{proposition} \label{prop:moments-constraint}
    Let \(P\) be parameterized as in \eqref{eq:poly-param} for a fixed choice of index set \(\mathcal{A} \subset \mathbb{N}_0^2\), and assume that \eqref{eq:poly-appendix} holds for \(m = m_1\). Suppose that the moments \(\{\mu_p \}_{p = 0}^{r}\) are available, and let
    \begin{equation}
        e_{\mathrm{max}} \coloneqq \max\{i - j\; : \; (i, j) \in \mathcal{A},\; c_{ij} \neq 0\}.
    \end{equation}
    Then the coefficients \(\{ c_{ij} \}_{(i, j) \in \mathcal{A}}\) satisfy the linear constraints
    \begin{equation}
        \sum_{(i, j) \in \mathcal{A}} (-1)^j \, c_{ij} \, \mu^{\ast j}_{i-j-e_{\mathrm{max}} + l} = 0, \qquad l = 0, 1, \dots, r. \label{eq:moment-constraint}
    \end{equation}
\end{proposition}

\begin{proof}
    Using the identity \((x - z)^{-1} = -z^{-1} \sum_{p=0}^{r} (x/z)^p + O\big(z^{-r-2}\big)\) in \eqref{eq:def-stieltjes} and applying \eqref{eq:def-moments}, we obtain the truncated Laurent expansion
    \begin{equation*}
        m_1(z) = - \sum_{p=0}^{r} \frac{\mu_p}{z^{p+1}} + O\big(z^{-r-2}\big), \qquad \lvert z \rvert \to \infty.
    \end{equation*}
    Set \(\phi_r(z) \coloneqq \sum_{p=0}^{r} \mu_p z^{-p}\), so that \(m_1(z) = -z^{-1} \phi_r(z) + O(z^{-r-2})\). For each \(j \in \mathbb{N}_0\),
    \begin{equation*}
        m_1(z)^{j} = (-1)^{j} z^{-j} \phi_r(z)^{j} + O\big(z^{-(j+r+1)}\big), \qquad \lvert z \rvert \to \infty.
    \end{equation*}
    Expanding the multinomial \(\phi_r(z)^{j}\) and collecting coefficients gives
    \begin{equation*}
        \phi_r(z)^{j} = \sum_{q=0}^{jr} \mu^{\ast j}_{q} z^{-q},
    \end{equation*}
    where \(\mu^{\ast j}_{q}\) is the \(q\)-th entry of the \(j\)-fold discrete convolution of \(\vect{\mu}\), as stated in \Cref{def:convolution}. Consequently, for each \((i, j) \in \mathcal{A}\),
    \begin{equation*}
        z^{i} m_1(z)^{j} = (-1)^{j} \sum_{q=0}^{jr} \mu^{\ast j}_{q} z^{i - j - q} + O\big(z^{i-j-(r+1)}\big), \qquad \lvert z \rvert \to \infty.
    \end{equation*}
    Substituting this into \eqref{eq:poly-param} and canceling the coefficients of \(z^{e_{\mathrm{max}} - l}\) for \(l = 0, 1, \dots, r\) in the truncated expansion of \(P(z, m_1(z))\) yields \eqref{eq:moment-constraint}.
\end{proof}

\begin{remark} \label{rm:moments-constraint-leading}
    In practice, while higher moments of an empirical distribution may be too noisy to estimate reliably, the leading constraint \(l = 0\) in \eqref{eq:moment-constraint} is always available since \(\mu_0 = 1\) holds unconditionally. This constraint enforces cancellation on the top \((i - j)\)-edge of \(\mathcal{A}\),
    \begin{equation*}
            \sum_{\substack{(i, j) \in \mathcal{A} \\ i - j = e_{\mathrm{max}}}} (-1)^{j} c_{ij} = 0.
    \end{equation*}
    For a full rectangular support \eqref{eq:index-rect}, the top \((i - j)\)-edge consists of the single monomial \((i, j) = (d_z, 0)\), and the above condition reduces to \(c_{d_z, 0} = 0\). It is therefore natural to exclude this monomial from \(\mathcal{A}\) a priori. More refined choices of \(\mathcal{A}\), including slanted Newton polygon supports, are discussed later in connection with branch-point constraints.
\end{remark}


\subsection{Constraints from Branch Points} \label{sec:constraint-branch}

We now exploit the prior knowledge of the branch-cut geometry \(I = \bigcup_{j=1}^{k} [a_j, b_j]\). As discussed in \Cref{sec:bg-alg-stieltjes}, the branch locus of the algebraic relation \eqref{eq:poly-appendix} is contained among the zeros of the discriminant \(\Delta(z)\), obtained by eliminating \(m_{\ast}\) from \eqref{eq:branch-condition}. In the present setting, we assume that the spectral edges \(\{a_j, b_j\}_{j=1}^{k}\) are among these branch points, which yields the nonlinear endpoint constraints
\begin{equation}
    \Delta(a_j) = 0 \qquad \text{and} \qquad \Delta(b_j) = 0, \qquad j = 1, \dots, k. \label{eq:endpoint-constraint}
\end{equation}

Since \(\Delta \in \mathbb{R}[z]\) is a univariate polynomial, each condition in \eqref{eq:endpoint-constraint} implies that \((z - a_j)\) and \((z - b_j)\) divide \(\Delta\). If the endpoints \(\{a_j, b_j\}_{j=1}^k\) are distinct, then with \(\pi(z) \coloneqq \prod_{j=1}^{k} (z - a_j) (z - b_j)\) we obtain the divisibility condition \(\pi \mid \Delta\). This divisibility viewpoint suggests an alternative enforcement strategy, for instance by parameterizing \(\Delta(z) = \pi(z) \tilde{\Delta}(z)\).

The endpoint constraints \eqref{eq:endpoint-constraint} also imply a simple necessary richness condition on the polynomial model class. Recall \(s = \deg_m(P)\) and write \(P(z, m) = \sum_{j=0}^{s} a_j(z) m^j\) with \(\deg_z(a_j) \leq d_z\). Then
\begin{equation*}
    \deg_z(\Delta) \leq (2s - 1) d_z.
\end{equation*}
Indeed, up to a nonzero factor depending on the leading coefficient \(a_s(z)\), one has \(\Delta(z) = \Res_m(P, \partial_m P)\), and the resultant is the determinant of the \((2s -1) \times (2s - 1)\) Sylvester matrix in the variable \(m\). Since all Sylvester entries are coefficients of \(P\) and \(\partial_m P\) and satisfy \(\deg_z(P) = \deg_z(\partial_m P) \leq d_z\), it follows that \(\deg_z(\Delta) \leq (2s -1 ) d_z\). See, for example, \citep[Section 4.1]{STURMFELS-2002}.

From the fact that a nonzero polynomial with \(2k\) distinct real zeros must have degree at least \(2k\), this yields the necessary condition
\begin{equation}
    (2s - 1) d_z \geq 2k, \label{eq:bound-degree}
\end{equation}
for \(P\) to accommodate \(2k\) prescribed distinct endpoints. We emphasize that \eqref{eq:bound-degree} is only a lower bound; the algebraic model may admit additional branch points away from \(\{a_j, b_j\}_{j=1}^{k}\), contributing further zeros of \(\Delta\).


\subsection{Numerical Reconstruction from Samples} \label{sec:numerical-implicit}

We now describe the numerical procedure used to fit an implicit relation \(P(z, m) = 0\) from samples of the physical branch \(m_1\), and to recover non-physical branches by root finding and continuation. Throughout, we fix an index set \(\mathcal{A} \subset \mathbb{N}_0^2\) and parameterize \(P\) by \eqref{eq:poly-param} with coefficient vector \(\vect{c} \in \mathbb{R}^{N}\), \(N \coloneqq \lvert \mathcal{A} \rvert\).

A convenient choice of \(\mathcal{A}\) is the rectangular set \eqref{eq:index-rect}, where the sheet count \(s\) is set empirically, while \(d_z\) can be initialized using the lower bound \eqref{eq:bound-degree} and increased as needed. In practice, we increase \((s, d_z)\) until the fitting error stabilizes while avoiding over-parameterization.

\paragraph{Sampling.}
We choose sampling points \(\{z_{\ell}\}_{\ell=1}^{n} \subset \mathbb{C} \setminus I\) along multiple contours enclosing each \(I_j\) and the overall \(I\), but away from the branch locus, in particular separated from the endpoints \(\{a_j, b_j\}_{j=1}^{k}\) by a small buffer. At each sample, we evaluate \(\hat{m}_{\ell} \coloneqq m_1(z_{\ell})\). The details of our approach for evaluating \(m_1\) from empirical measures are described in \Cref{sec:eval-stieltjes}. To preserve the Schwarz symmetry needed for real-coefficient structure as discussed in \Cref{lem:poly-real}, we sample in conjugate pairs: whenever \(z_{\ell}\) is included, we also include \(\bar{z}_{\ell}\), which yields \(m_1(\bar{z}_{\ell}) = \overline{m_1(z_{\ell})}\).

\paragraph{Linear system for the coefficients.}
Define the design matrix \(\tens{A} \in \mathbb{C}^{n \times N}\) with entries
\begin{equation}
    A_{\ell, (i, j)} \coloneqq z^{i}_{\ell} \hat{m}^{j}_{\ell}, \qquad (i, j) \in \mathcal{A}, \quad \ell = 1, \dots, n, \label{eq:design-matrix}
\end{equation}
where the column index \((i, j)\) follows a fixed ordering of \(\mathcal{A}\). Recall that \(\vect{c} \in \mathbb{R}^{N}\) collects the coefficients \(c_{ij}\) in this same order. By construction, \((\tens{A} \vect{c})_{\ell} = P(z_{\ell}, \hat{m}_{\ell})\). The samples satisfy \(P(z_{\ell}, \hat{m}_{\ell}) \approx 0\), hence \(\tens{A} \vect{c} \approx \tens{0}\). We compute \(\vect{c}\) as a nontrivial approximate null vector via the augmented homogeneous least-squares problem
\begin{equation}
    \hat{\vect{c}} \coloneqq \argmin_{\vect{c} \in \mathbb{R}^N} \lVert \tens{A}_{\mathbb{R}} \vect{c} \rVert_2,
    \qquad \text{subject to} \qquad
    \lVert \vect{c} \rVert_2 = 1, \label{eq:LS-c}
\end{equation}
where
\begin{equation}
    \tens{A}_{\mathbb{R}} \coloneqq
    \begin{bmatrix}
        \Re(\tens{A}) \\
        \Im(\tens{A})
    \end{bmatrix}.
    \label{eq:augment}
\end{equation}
The solution is given by the right singular vector corresponding to the smallest singular value of \(\tens{A}_{\mathbb{R}}\). The scalar ambiguity of \(P\) (see \Cref{rem:gauge}) is fixed by the normalization \(\lVert \vect{c} \rVert_2 = 1\) and a deterministic sign convention, for example by requiring the first nonzero entry of \(\vect{c}\) to be positive.

\paragraph{Incorporating moment constraints.}
When moment constraints derived in \Cref{prop:moments-constraint} are used, they can be written as a linear system \(\tens{B} \vect{c} = \tens{0}\), where \(\tens{B} \in \mathbb{R}^{(r+1) \times N}\) encodes the \(r+1\) enforced constraints \(\{\mu_p\}_{p=0}^{r}\). To enforce these constraints, we restrict \(\vect{c}\) to \(\ker(\tens{B})\). Concretely, let \(\tens{Q} \in \mathbb{R}^{N \times N_c}\) have orthonormal columns spanning \(\ker(\tens{B})\), obtained for example from an SVD of \(\tens{B}\). Then, any feasible coefficient vector can be written as \(\vect{c} = \tens{Q} \vect{y}\). We then solve the reduced problem
\begin{equation}
    \hat{\vect{y}} \coloneqq \argmin_{\vect{y} \in \mathbb{R}^{N_c}} \lVert \tilde{\tens{A}}_{\mathbb{R}} \vect{y} \rVert_2
    \qquad \text{subject to} \qquad
    \lVert \vect{y} \rVert_2 = 1, \label{eq:LS-y}
\end{equation}
where
\begin{equation*}
    \tilde{\tens{A}}_{\mathbb{R}} \coloneqq
    \begin{bmatrix}
        \Re(\tens{A} \tens{Q}) \\
        \Im(\tens{A} \tens{Q})
    \end{bmatrix},
    \label{eq:augment-AQ}
\end{equation*}
followed by \(\hat{\vect{c}} \coloneqq \tens{Q} \hat{\vect{y}}\). In practice, enforcing only the leading constraint \(\mu_0 = 1\) (that is, \(r = 0\)) is often beneficial even when higher-order moments are unreliable.

\Cref{alg:poly} summarizes the coefficient-estimation step (with optional moment constraints), producing a fitted polynomial \(\hat P\).

\begin{algorithm}[ht!]
    \caption{Implicit fitting of an algebraic relation \(P(z, m)=0\) from samples of \(m_1\)}
    \label{alg:poly}
    \SetKwInOut{Input}{Input}
    \SetKwInOut{Output}{Output}
    \DontPrintSemicolon

    \vspace{1mm}
    \Input{
        Index set \(\mathcal{A}\subset\mathbb{N}_0^2\) such as in \eqref{eq:index-rect} with \(N\coloneqq|\mathcal{A}|\); \\
        Oracle for the empirical Stieltjes transform \(\hat m_1(z)\) using submatrix eigenvalues; \\
        Number of sampling points \(n\) (even); \\
        Optional moments \(\{\mu_p\}_{p=0}^{r}\) (to enforce \Cref{prop:moments-constraint}).
    }

    \vspace{2mm}
    \Output{
        Coefficient vector \(\hat{\vect{c}}\in\mathbb{R}^{N}\) defining \(\hat P(z,m)=\sum_{(i,j)\in\mathcal{A}}\hat c_{ij} z^i m^j\); \\
        Fitting residual error \(e \coloneqq \lVert \tens{A} \hat{\vect{c}} \rVert_2 / \lVert \tens{A} \rVert_F\) for model selection.
    }

    \vspace{2mm}
    \tcp{Sample the physical branch away from the cut, in conjugate pairs}
    Initialize set of points \(\mathcal{Z} \leftarrow \{\}\) and function values \(\mathcal{M} \leftarrow \{\}\). \;
    Generate \(n/2\) points \(\{z_{\ell}\}_{\ell=1}^{n/2} \subset \mathbb{C}^{+}\) on arbitrary contours enclosing \(I\), excluding \(\epsilon\)-neighborhood from \(I\). \;
    Evaluate \(\hat{m}_{\ell} \leftarrow \hat{m}_1(z_{\ell})\). \;
    Append \((z_{\ell}, \hat{m}_{\ell})\) and \((\bar{z}_{\ell}, \overline{\hat{m}}_{\ell})\) to \((\mathcal{Z}, \mathcal{M})\). \tcp*{Total of \(n\) samples with conjugate points}

    \vspace{2mm}
    \tcp{Build the design matrix \eqref{eq:design-matrix} from samples}
    Form \(\tens{A}\in\mathbb{C}^{n\times N}\) with columns indexed by \((i,j)\in\mathcal{A}\):
    \begin{equation*}
        A_{\ell,(i,j)} \leftarrow z_\ell^i\,\hat m_\ell^{\,j},
        \qquad
        (z_{\ell}, \hat{m}_{\ell}) \in (\mathcal{Z}, \mathcal{M}).
    \end{equation*}

    Form \(\tens{A}_{\mathbb{R}} \leftarrow
    \begin{bmatrix}
        \Re(\tens{A}) \\
        \Im(\tens{A})
    \end{bmatrix}\) as in \eqref{eq:augment} \tcp*{Form the real-valued design matrix}

    \vspace{2mm}
    \tcp{Optional: enforce linear moment constraints \(\tens{B} \vect{c}=\tens{0}\) from \Cref{prop:moments-constraint}}
    \If{moments \(\{\mu_p\}_{p=0}^{r}\) are provided}{
        Build \(\tens{B}\in\mathbb{R}^{(r+1)\times N}\) from \eqref{eq:moment-constraint}\;
        Compute an orthonormal basis \(\tens{Q}\in\mathbb{R}^{N\times N_c}\) for \(\ker(\tens{B})\) (e.g. using SVD) \tcp*{\(N_{c} = N - \operatorname{rank}(\tens{B})\)}
        Set \(\tens{A}_{\mathbb{R}} \leftarrow \tens{A}_{\mathbb{R}}\tens{Q}\) and interpret unknowns as \(\vect{y}\in\mathbb{R}^{N_c}\)\;
    }

    \vspace{2mm}
    \tcp{Solve \eqref{eq:LS-c}, or \eqref{eq:LS-y} with moment constraints}
    \(d \leftarrow \operatorname{size}(\tens{A}_{\mathbb{R}}, 2)\) \tcp*{Number of columns}
    \(\tens{U} \gtens{\Sigma} \tens{V} \leftarrow \tens{A}_{\mathbb{R}}\) \tcp*{Thin SVD where \(\gtens{\Sigma} = \operatorname{diag}(\sigma_1, \dots, \sigma_d)\) and \(\sigma_1 \geq \dots \geq \sigma_{d}\)}
    \(\hat{\vect{v}} \leftarrow V_{[:, d]}\) \tcp*{Right singular vector for the smallest singular value \(\sigma_d\)}
    \leIf{moment constraints were enforced}{
        \(\hat{\vect{c}} \leftarrow \tens{Q}\hat{\vect{v}}\)
    }{
        \(\hat{\vect{c}} \leftarrow \hat{\vect{v}}\)
    }

    \(\hat{\vect{c}} \leftarrow \hat{\vect{c}} / \lVert \hat{\vect{c}} \rVert_2\) \tcp*{Normalize, optionally, fix a deterministic sign convention}

    \vspace{2mm}
    \Return{\(\hat{\vect{c}}\) and residual error \( e \coloneqq \|\tens{A}\hat{\vect{c}}\|_2/\|\tens{A}\|_F\).}\;
\end{algorithm}


\section{Computing the Stieltjes Transform and Density} \label{sec:eval-stieltjes}

This section describes the numerical procedure used to evaluate the Stieltjes transform and to recover the density from it. The first task is to evaluate the physical Stieltjes branch \(m(z)\). When empirical eigenvalues are available, this can be done directly by summing the Cauchy kernel. However, after fitting an algebraic spectral curve \(P(z, m) = 0\), evaluating \(m(z)\) becomes a \emph{root-selection} problem: for each query point \(z\), one must identify the \emph{physical} root among all algebraic roots of \(P(z, \cdot)\). This process is described in \Cref{sec:root-selection}. The second task, discussed in \Cref{sec:inv-stieltjes}, is the inverse Stieltjes problem: once \(m(z)\) has been evaluated near the real axis, we recover the density \(\rho\) from its boundary values.


\subsection{Physical Root Selection for the Stieltjes Transform} \label{sec:root-selection}

Recall from \eqref{eq:def-stieltjes} that the Stieltjes transform of a spectral measure \(\nu\) is
\begin{equation}
    m(z) = \int_{\mathbb{R}} \frac{1}{x - z} \, \nu(\mathrm{d}x), \quad z \in \mathbb{C} \setminus \operatorname{supp}(\nu).
    \label{eq:stieltjes-emp}
\end{equation}
For empirical eigenvalues \(\{\lambda_1, \dots, \lambda_n\}\), this quantity is computed directly as
\begin{equation*}
    m_n(z) = \frac{1}{n} \sum_{i=1}^n \frac{1}{\lambda_i - z}.
\end{equation*}
The more difficult problem arises when we evaluate the Stieltjes transform through a fitted algebraic relation \(P(z, m) = 0\). For each fixed query point \(z \in \mathbb{C} \setminus I\), the polynomial \(m \mapsto P(z, m)\) has \(s=\deg_m(P)\) roots. Only one of these roots is the physical Stieltjes value, while the remaining roots belong to non-physical sheets of the algebraic curve. Thus, evaluating \(m(z)\) from \(P\) is a nontrivial branch-selection problem, especially for high-degree curves and multi-bulk spectra, where roots may approach one another, exchange ordering, or pass near branch points.

\Cref{fig:diffusion-candidates} illustrates this issue for the diffusion model example; see \Cref{sec:bg-diffusion}. In the left panel, we evaluate all roots of \(m \mapsto P(z,m)\) at \(z=\lambda+i\delta\), for a small fixed \(\delta>0\) and a broad range of \(\lambda\), and plot the candidate values \(\pi^{-1}\Im m_i(\lambda+i\delta)\), \(i=1,\dots,s\). Since the roots are unordered from one value of \(\lambda\) to the next, the result is an unordered scatter of algebraic candidates rather than a connected curve. At sufficiently fine resolution, the neighboring points visually trace several candidate branches. Only one of these traces corresponds to the physical density,
\begin{equation*}
    \rho(\lambda) \approx \frac{1}{\pi}\Im m(\lambda+i\delta),
\end{equation*}
namely the trace that agrees with the empirical histogram shown in the background. The right panel shows the result after continuously selecting the physical branch. Thus, the purpose of the procedure below is to pass from the unordered scatter of algebraic roots in the left panel to the coherent physical curve in the right panel.

\begin{figure}[t]
    \centering
    \ifjournal
        \includegraphics[width=\textwidth]{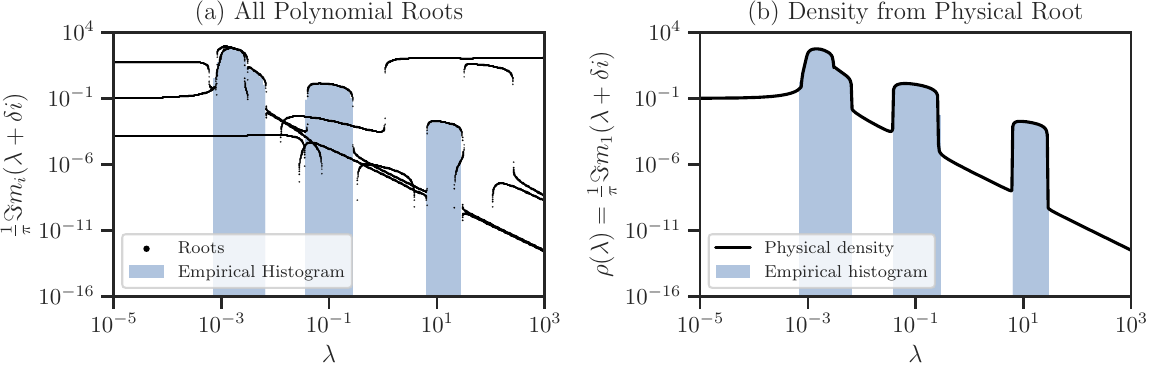}
    \else
        \includegraphics[width=0.9\textwidth]{\figdir/diffusion-candidates-init}
    \fi
    \caption{Physical branch selection for the fitted algebraic Stieltjes transform in the diffusion-model example at the initial size \(n=2^2\mathrm{K}\). Panel (a) shows all candidate densities \(\pi^{-1}\Im m_i(\lambda + i\delta)\) with \(\delta = 10^{-6}\), obtained from the seven roots \(m_i\) of \(P(\lambda + i\delta, m) = 0\), where \(\deg_m(P)=7\). These unordered candidates include both physical and non-physical sheets and illustrate the difficulty of pointwise root selection in a multi-bulk, log-scale spectrum. Panel (b) shows that the continuation procedure successfully selects the physical branch, producing a coherent density curve that matches the empirical histogram on the bulks. The remaining sloped baseline in the spectral gaps is a finite-\(\delta\) artifact and is addressed in \Cref{sec:inv-stieltjes}.}
    \label{fig:diffusion-candidates}
\end{figure}

The physical branch is characterized by the Herglotz property and by its normalization at infinity,
\begin{equation*}
    \Im m(z) > 0 \quad \text{for } z \in \mathbb{C}^+,
    \qquad
    m(z) \sim - \frac{1}{z} \quad \text{as } \vert z \vert \to \infty.
\end{equation*}
However, these conditions alone are not always sufficient for stable pointwise root selection. Instead, we treat the fitted relation as a spectral curve
\begin{equation*}
    \mathcal{C} = \left\{ (z, m) \in \mathbb{C}^2 \mid P(z, m) = 0 \right\},
\end{equation*}
and track the physical sheet geometrically on this curve from an anchor point to the desired target point.

The anchor point \(z_{\mathrm{a}}\) is chosen so that the physical root \(m_{\mathrm{a}}\) can be identified reliably. One option is to place \(z_{\mathrm{a}}\) far from the spectrum, where the normalization \(m(z) \sim -1/z\) selects the correct algebraic root. Another option is to place \(z_{\mathrm{a}}\) closer to the real axis and select the algebraic root closest to the empirical value \(m_n(z_{\mathrm{a}})\) from \eqref{eq:stieltjes-emp}. The latter approach is often more stable for complicated algebraic curves.

Starting from \((z_{\mathrm{a}}, m_{\mathrm{a}}) \in \mathcal{C}\), we choose a continuous path \(\gamma : [0,1] \to \mathbb{C}^{+}\) with \(\gamma(0)=z_{\mathrm{a}}\) and \(\gamma(1)=z\), and lift this path to the spectral curve by seeking \(m(s)\) such that
\begin{equation*}
    P(\gamma(s), m(s)) = 0,
    \qquad
    m(0)=m_{\mathrm{a}}.
\end{equation*}
Root ordering alone is unreliable along such paths: near branch points, two sheets may come close, exchange order, or locally appear to cross. To avoid jumping between sheets, we follow the local geometry of the curve using a predictor--corrector procedure. Differentiating \(P(\gamma(s),m(s))=0\) along the path gives
\begin{equation*}
    \frac{\mathrm{d}m}{\mathrm{d}s} = -\frac{\partial_z P(\gamma(s),m(s))} {\partial_m P(\gamma(s),m(s))} \, \gamma'(s).
\end{equation*}
Equivalently, away from branch points, the local tangent satisfies
\begin{equation*}
    \frac{\mathrm{d} m}{\mathrm{d}z} = - \frac{\partial_z P(z, m)}{\partial_m P(z, m)}.
\end{equation*}

This tangent relation provides a predictor for the next point \(z_{\mathrm{new}}\) along the path on the same sheet. The predicted value is then corrected by enforcing
\begin{equation*}
    P(z_{\mathrm{new}}, m_{\mathrm{new}}) = 0,
\end{equation*}
which amounts to projecting the predicted point back onto the spectral curve. This correction is performed by Newton iteration, viewing \(m_{\mathrm{new}}\in\mathbb{C}\) as a two-dimensional real variable. The selected root at \(z_{\mathrm{new}}\) is the algebraic root closest to this corrected value. If no root is sufficiently close, or if the local geometry is too stiff, the continuation step is subdivided and repeated with smaller increments until the root closeness criterion is satisfied. Repeating this predictor--corrector continuation along the path yields a continuous selection of the physical branch.

This continuation-based, or homotopic, root selection is what allows the fitted polynomial to be used as a numerical Stieltjes transform. The resulting physical-branch selection procedure is summarized in \Cref{alg:stieltjes}.

\begin{algorithm}[ht!]
    \caption{Physical branch selection for an algebraic Stieltjes transform}
    \label{alg:stieltjes}
    \SetKwInOut{Input}{Input}
    \SetKwInOut{Output}{Output}
    \DontPrintSemicolon

    \vspace{1mm}
    \Input{
        Polynomial \(P \in \mathbb{R}[z,m]\) (see \Cref{alg:poly}); \\
        Query points \(\{z_j\}_{j=1}^{N_z}\subset\mathbb{C}^{+}\); \\
        Anchor rule for selecting \(m_{\mathrm{a}}\) at \(z_{\mathrm{a}}\); \\
        Newton tolerance \(\epsilon>0\).
    }

    \vspace{2mm}
    \Output{
        Physical Stieltjes values \(m(z_j)\) satisfying \(P(z_j,m(z_j))=0\).
    }

    \vspace{2mm}
    \For{\(j = 1,\dots,N_z\)}{
        Choose an anchor point \(z_{\mathrm{a}}\) and a path
        \(\gamma:[0,1]\to\mathbb{C}^{+}\) with
        \(\gamma(0)=z_{\mathrm{a}}\), \(\gamma(1)=z_j\). \;

        Select the physical root \(m_{\mathrm{a}}\) at \(z_{\mathrm{a}}\)
        using either \(m(z)\sim -z^{-1}\) or the empirical value \(m_n(z_{\mathrm{a}})\). \;

        Initialize \(m^{(0)} \leftarrow m_{\mathrm{a}}\). \;

        \vspace{2mm}
        \tcp{Lift the path \(\gamma\) to the spectral curve \(P(z,m)=0\)}
        \For{successive points \(z_{\mathrm{old}}, z_{\mathrm{new}}\) along \(\gamma\)}{
            Compute the tangent predictor
            \[
                \dot{m}
                =
                -\frac{\partial_z P(z_{\mathrm{old}},m)}
                       {\partial_m P(z_{\mathrm{old}},m)}
                \dot{z}.
            \]

            Predict \(m_{\mathrm{pred}}\) at \(z_{\mathrm{new}}\). \;

            Correct \(m_{\mathrm{pred}}\) by Newton iteration on
            \[
                P(z_{\mathrm{new}},m)=0.
            \]

            \eIf{Newton correction is accepted}{
                Set \(m \leftarrow m_{\mathrm{new}}\). \;
            }{
                Subdivide the path segment and repeat. \tcp*{Avoid sheet jumping}
            }
        }

        Store \(m(z_j)\leftarrow m\). \;
    }

    \vspace{2mm}
    \Return{\(\{m(z_j)\}_{j=1}^{N_z}\).}\;
\end{algorithm}

The continuation step presented here is not merely a numerical convenience. Without it, direct pointwise root selection is unstable even for relatively simple algebraic curves, because root orderings can change from one query point to the next. This instability becomes much more severe for high-degree, multi-bulk, or log-scale spectra, where several sheets may come close over many orders of magnitude. The diffusion-model example in \Cref{fig:diffusion-candidates} is deliberately chosen as a stringent stress test of this type: the fitted curve has degree \(\deg_m(P)=7\), the density spans many decades in both \(\lambda\) and \(\rho(\lambda)\), the geometry is highly stiff near the spectral edges, and the support contains several strongly unbalanced bulks, with the leftmost bulk orders of magnitude taller than the others. Nevertheless, the geometric continuation procedure reliably resolves the physical branch and produces a coherent density from the fitted algebraic curve.

With the physical root \(m(z)\) selected consistently over a query grid \(z=\lambda+i\delta\), the density \(\rho(\lambda)\) can then be recovered from its boundary values using the inverse Stieltjes procedure described next.


\subsection{Density Recovery and \texorpdfstring{\(\delta\)}{delta}-Extrapolation}
\label{sec:inv-stieltjes}

Once the physical branch \(m(z)\) has been selected, the density is recovered from the boundary values of the Stieltjes transform. Recall the Plemelj formula from \eqref{eq:inv-stieltjes}:
\begin{equation*}
    \rho(\lambda) = \frac{1}{\pi} \lim_{\delta \to 0^{+}} \Im m(\lambda + i \delta).
\end{equation*}
In numerical computations, however, one cannot evaluate exactly on the real axis at \(\delta = 0\). Instead, one evaluates \(m\) at a small positive offset \(\delta\), giving the regularized density
\begin{equation*}
    \rho_{\delta}(\lambda) \coloneqq \frac{1}{\pi} \Im m(\lambda + i \delta).
\end{equation*}
The parameter \(\delta\) plays the role of a smoothing scale: larger values of \(\delta\) improve numerical stability but smear sharp features of the density, while smaller values give sharper reconstructions at the cost of increased sensitivity near branch points and spectral edges.

This smoothing effect has a precise interpretation. Indeed,
\begin{equation*}
    \rho_{\delta}(\lambda) = \frac{1}{\pi} \Im m(\lambda + i\delta) = \int_{\mathbb{R}} \frac{1}{\pi}\frac{\delta}{(\lambda-\xi)^2+\delta^2}\nu(\mathrm{d}\xi)
    = (P_{\delta} * \nu)(\lambda),
\end{equation*}
where
\begin{equation*}
    P_{\delta}(\lambda) \coloneqq \frac{1}{\pi}\frac{\delta}{\lambda^2+\delta^2}
\end{equation*}
is the Poisson kernel in the upper half-plane \citep{STEIN-2003a}. Thus, evaluating the inverse Stieltjes formula at a single finite $\delta$ amounts to convolving $\nu$ with $P_{\delta}$, which introduces a finite-$\delta$ blur and a nonzero baseline outside the support.

This effect is especially visible for spectra plotted on logarithmic scales. In spectral gaps, where the limiting density should vanish, the finite-\(\delta\) Poisson tail creates a nonzero baseline. Away from the support, this baseline behaves like \(P_{\delta}(\lambda) \sim \delta/(\pi\lambda^2)\), which appears as a line of slope \(-2\) on a log--log plot. We refer to this artifact as the Poisson \(\delta\)-floor. Decreasing \(\delta\) lowers this floor and makes the bulks stand out more sharply, but it also makes branch tracking and root selection more sensitive. The left panel of \Cref{fig:diffusion-inv-stieltjes} illustrates this issue for the diffusion-model example. In the right panel, the baseline is eliminated by the extrapolation procedure described below, and the bulk edges are sharply resolved.

\begin{figure}[t]
    \centering
    \ifjournal
        \includegraphics[width=\textwidth]{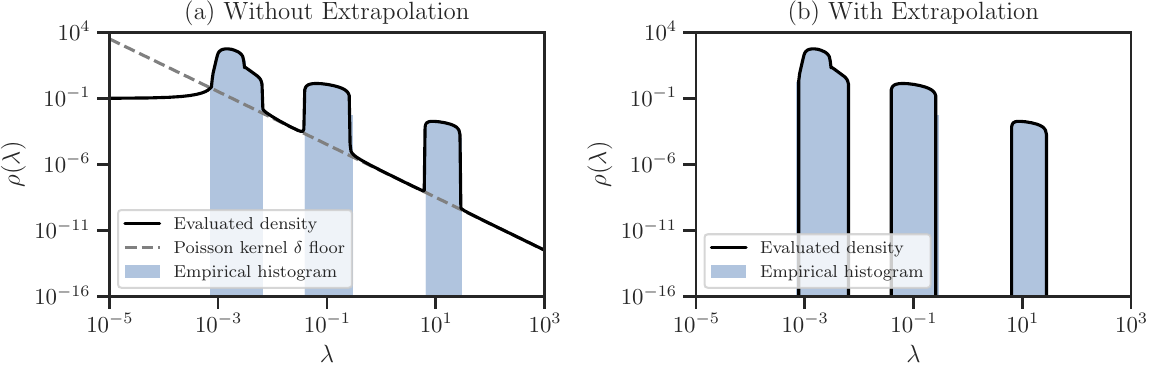}
    \else
        \includegraphics[width=0.9\textwidth]{\figdir/diffusion-inv-stieltjes}
    \fi
    \caption{Inverse Stieltjes recovery for the diffusion-model spectrum at the initial size \(n=2^2\mathrm{K}\). Panel (a) shows the direct finite-\(\delta\) density \(\rho_{\delta}(\lambda) = \pi^{-1}\Im m(\lambda + i\delta)\), which agrees with the empirical histogram on the bulks but exhibits a Poisson \(\delta\)-floor in spectral gaps. Panel (b) shows the density after extrapolating the finite-\(\delta\) evaluations to \(\delta=0\), removing the artificial gap baseline while preserving the bulk densities.}
    \label{fig:diffusion-inv-stieltjes}
\end{figure}

To remove this artifact, we evaluate the physical branch on a ladder of offsets
\begin{equation*}
    \delta_0 < \delta_1 < \cdots < \delta_L
\end{equation*}
and extrapolate the sequence \(\rho_{\delta_i}(\lambda)\) to \(\delta=0\) at each fixed \(\lambda\). A simple approach is to use a low-degree polynomial fit in \(\delta\). In practice, we use a Chebyshev basis for this extrapolation, which is more stable than a monomial basis when several \(\delta\)-levels are used  \citep{TREFETHEN-2019a}. Concretely, after mapping the offset ladder \(\{\delta_i\}_{i=0}^L\) to points \(\{\eta_i\}_{i=0}^L \subset [-1,1]\), we approximate
\begin{equation*}
    \rho_{\delta_i}(\lambda)
    \approx
    \sum_{k=0}^{q} c_k(\lambda) T_k(\eta_i),
\end{equation*}
where \(T_k\) are Chebyshev polynomials. The extrapolated density is then obtained by evaluating this fitted polynomial at the value of \(\eta\) corresponding to \(\delta=0\). This approach is particularly useful for log-scale spectra, such as the diffusion-model example shown in \Cref{fig:diffusion-inv-stieltjes}. In this example, we used four \(\delta\)-levels and a quadratic Chebyshev extrapolation. This was sufficient to reduce the artificial Poisson floor below \(10^{-20}\) across the spectral gaps, without additional regularization.

For absolutely continuous measures with densities represented on a uniform linear grid, one can alternatively exploit the convolution identity \(\rho_{\delta} = P_{\delta} \ast \rho\). In that case, the collection of finite-\(\delta\) evaluations can be viewed as a regularized deconvolution problem, which can be solved efficiently in Fourier space. Indeed, on a uniform linear grid, the convolution relation gives
\begin{equation*}
    \mathcal{F}\{\rho_{\delta}\}(\omega)
    =
    \mathcal{F}\{P_{\delta}\}(\omega) \, \mathcal{F}\{\rho\}(\omega)
    =
    e^{-\delta \vert \omega \vert} \mathcal{F}\{\rho\}(\omega),
\end{equation*}
where \(\mathcal{F}\) is the Fourier transform. Thus, if \(\rho_{\delta}\) were known exactly on a uniform grid, the formal inverse would be
\begin{equation*}
    \rho
    =
    \mathcal{F}^{-1}
    \left[
        e^{\delta \vert \omega \vert}
        \mathcal{F}\{\rho_{\delta}\}(\omega)
    \right].
\end{equation*}
The factor \(e^{\delta \vert \omega \vert}\) inverts the Poisson smoothing, but it also amplifies high-frequency numerical noise. Hence, the Fourier recovery must be regularized, for instance by damping or truncating high-frequency modes.

The \(\delta\)-extrapolation is similarly important for densities with wide log--log scale spectra, such as the diffusion-model example shown in \Cref{fig:diffusion-inv-stieltjes}. Without it, the finite-\(\delta\) Poisson floor obscures spectral gaps and makes narrow bulks appear artificially connected on a log--log plot. After extrapolation, the density retains the same bulk-scale agreement with the empirical histogram, while the gap baseline is pushed below the plotted range. This allows the multi-bulk structure and spectral edges to be resolved cleanly, even in this stringent log-scale test case.


\section{Free Decompression on Algebraic Spectral Curves}
\label{sec:fd-algebraic}

In this section, we develop the mathematical and computational foundation of free decompression from an algebraic Stieltjes relation. We begin in \Cref{sec:nica} with the Nica--Speicher theorem, which identifies the R-transform scaling that underlies free compression and motivates its inverse operation. We then define infinite decompressibility in \Cref{sec:decompressibility} and show its equivalence with free infinite divisibility.

The next issue is computational. Although free decompression leads naturally to a nonlinear PDE for the Stieltjes transform, \Cref{sec:bg-pde-hard} explains why solving this PDE directly is numerically unstable in realistic examples, especially when analytic continuation across branch cuts is required. The main purpose of the remainder of the section is to show that this difficulty can be avoided when the initial Stieltjes transform is represented by an algebraic relation \(P(z,m)=0\). In \Cref{sec:formal}, free decompression is lifted to a rational transformation of the associated spectral curve. In \Cref{sec:coeffmethod}, this transformation is converted into an evolved polynomial relation whose roots give all algebraic candidates at later times.

The remaining challenge is to select the physical Stieltjes branch from these candidates. We address this in \Cref{sec:fd-on-sc} by tracking the physical sheet geometrically along the free-decompression flow. This leads to a predictor--corrector method on the spectral curve, formulated as a \(2\times2\) Newton system in the natural curve coordinates. This viewpoint is the basis of the practical free-decompression algorithm used in the numerical experiments.


\subsection{The Nica--Speicher Theorem}
\label{sec:nica}

Our exposition begins with a complete proof of the R-transform relationship which underpins the free decompression procedure, also known as the \emph{Nica--Speicher Theorem}. Although our conditions appear to be less general than required in practice, our emphasis here is on establishing a rigorous set of assumptions upon which free decompression can be guaranteed to hold.

To do so, it is important to define several important objects pertaining to free probability theory. First, we begin with the definition of ``limits'' of matrices in the noncommutative algebraic sense.


\begin{definition}[Free Algebra Limit]
    Let $\tr_n$ denote the normalized trace, given by $\tr_n(\tens{A}) = \frac1n \mbox{tr}(\tens{A})$. A sequence of matrices $\tens{A}_n$ has a \emph{limit} $a$ if for every integer $k$, $\tr_n(\tens{A}_n^k)$ converges as $n \to \infty$ to a value, denoted $\tr_{\infty}(a^k)$. More generally, a family of sequences of matrices
    \begin{equation*}
        (\{\tens{A}_{1,n}\}_{n=1}^{\infty},\dots,\{\tens{A}_{k,n}\}_{n=1}^{\infty})
    \end{equation*}
    has limits $(a_1,\dots,a_k)$ if for any (noncommutative) polynomial combination $p(x_1, \dots, x_k)$, $\tr_n(p(\tens{A}_{1,n},\dots,\tens{A}_{k,n}))$ converges to a value, denoted $\tr_{\infty}(p(a_1,\dots,a_k))$. 
\end{definition}

Next, we consider the critical definition of asymptotic freeness, first introduced in \citet{VOICULESCU-1991}.

\begin{definition}[Asymptotically Free]
    A family of sequences of matrices
    \begin{equation*}
        (\{\tens{A}_{1,n}\}_{n=1}^{\infty}, \dots, \{\tens{A}_{k,n}\}_{n=1}^{\infty})
    \end{equation*}
    is \emph{asymptotically free} from another family of sequences of matrices
    \begin{equation*}
        (\{\tens{B}_{1,n}\}_{n=1}^{\infty}, \dots, \{\tens{B}_{l,n}\}_{n=1}^{\infty})
    \end{equation*}
    if the two families have a joint limit $(a_1,\dots,a_k,b_1,\dots,b_l)$ and $(a_1,\dots,a_k)$, $(b_1,\dots,b_l)$ are \emph{freely independent} in the following sense: for any integer $m \geq 1$ and any ``alternating'' sequence of elements $x_1, x_2, \dots, x_m$ such that $x_{2i+1}$ is a noncommutative polynomial combination of $(a_1,\dots,a_k)$ and $x_{2i}$ is a noncommutative polynomial combination of $(b_1,\dots,b_l)$, if $\tr_{\infty}(x_i) = 0$ for all $i$, then $\tr_{\infty}(x_1 x_2 \cdots x_m) = 0$.
\end{definition}

Finally, we will require another important transform, similar to the R-transform.

\begin{definition}[S-Transform]
    Let $\nu$ be a probability measure and let $\mu_n = \int x^n \mathrm{d} \nu(x)$ denote its moments for each integer $n$. The \emph{truncated moment-generating function} of $\nu$ is given by
    \[
    M(z) = \sum_{n=1}^\infty \mu_n z^n.
    \]
    If $\chi(z)$ is the functional inverse of $M$ so that $M(\chi(z)) = z$, then the S-transform of $\nu$ is given by
    \[
    S(z) = \frac{z + 1}{z} \chi(z).
    \]
\end{definition}

The significance of the S-transform is its behavior for multiplicative free convolution. If $a$ and $b$ are freely independent, then the spectral distribution of $b^{1/2} a b^{1/2}$ has S-transform $S_{ab}$ satisfying \citep[Lemma 2.19]{COUILLET-2022}
\[
    S_{ab}(z) = S_a(z) S_b(z),
\]
where $S_a$ and $S_b$ are the S-transforms of $a$ and $b$ respectively. To relate the R- and S-transforms, we have the following lemma.

\begin{lemma}
    \label{lem:RSTrans}
    Letting $R(z)$ and $S(z)$ denote the R- and S-transforms of a probability measure, if $y = z S(z)$, then $z = y R(y)$.
\end{lemma}

\begin{proof}
    From \citet[Definition 5]{COUILLET-2022}, if $m(z)$ is the Stieltjes transform, then
    \[
        m\left(\frac{z+1}{zS(z)}\right) = -zS(z).
    \]
    Substituting $y = zS(z)$ gives
    \[
        m\left(\frac{z+1}{y}\right) = -y.
    \]
    But since $m(R(y) + y^{-1}) = -y$ from \citet[Definition 5]{COUILLET-2022}, $z/y = R(y)$ and so $z = y R(y)$.
\end{proof}

Now we may prove the Nica--Speicher Theorem \citep{NICA-1996}.

\begin{theorem}[Nica--Speicher]
    Let $\{\tens{A}_n\}_{n=1}^\infty$ be a sequence of $n \times n$ deterministic real symmetric matrices whose empirical spectral distribution converges weakly as $n \to \infty$ to a compactly supported probability measure with R-transform $R(z)$ and S-transform $S(z)$. Consider a sequence of integers $\{k_n\}_{n=1}^\infty$ such that $\lim_{n\to\infty} \frac{k_n}{n} = \tau \in (0,1)$. For each $n$, let $\tens{Q}_n$ be a Haar-distributed random orthogonal matrix of size $n \times n$ and define $\tilde{\tens{A}}_n = \tens{Q}_n^{\intercal} \tens{A}_n \tens{Q}_n$. Let $\sigma_n$ be a random permutation of $\{1,\dots,n\}$, independent of $\tens{Q}_n$, and let $\tens{B}_n = (\tilde{\tens{A}}_{\sigma_n(i)\sigma_n(j)})_{i,j=1}^{k_n}$. Then as $n \to \infty$, the empirical spectral distribution of $\tens{B}_n$ converges weakly almost surely to a limiting probability measure whose R- and S-transforms $R_\tau(z)$ and $S_\tau(z)$ are given by \[
        R_\tau(z) = R(\tau z),\qquad S_\tau(z) = S(\tau z).
    \]
\end{theorem}

\begin{proof}
    First, choose $M > 0$ such that the limiting measure is supported in $(-M, M)$. Weak convergence implies that clipping the eigenvalues of $\tens{A}_n$ to $[-M, M]$ changes only $o(n)$ eigenvalues. This changes $\tens{A}_n$, and hence the corresponding compressed matrix $\tens{B}_n$, by rank $o(n)$. By the rank inequality, the empirical distribution functions of the original and clipped compressed matrices differ by at most $o(n)/k_n = o(1)$. It therefore suffices to prove the result for the clipped sequence, for which $\|\tens{A}_n\| \leq M$. We retain the same notation for the resulting matrices.

    Fix an integer $n$. For a permutation $\sigma$ of $\{1,\dots,n\}$ into itself, we may define the permutation matrix $\gtens{\Pi}_\sigma$ such that $(\gtens{\Pi}_\sigma)_{ij} = 1$ if $i = \sigma(j)$ and $0$ otherwise. In this way, $\tens{B}_n$ becomes the top-left $k_n \times k_n$ block of
    \[
        \tilde{\tilde{\tens{A}}}_n = \gtens{\Pi}_{\sigma_n}^{\intercal} \tilde{\tens{A}}_n \gtens{\Pi}_{\sigma_n} = \gtens{\Pi}_{\sigma_n}^{\intercal} \tens{Q}_n^{\intercal} \tens{A}_n \tens{Q}_n \gtens{\Pi}_{\sigma_n}.
    \]
    Since $\tilde{\tens{Q}}_n = \tens{Q}_n \gtens{\Pi}_{\sigma_n}$ is also an orthogonal matrix and the Haar measure is invariant under permutations, $\tilde{\tens{Q}}_n \overset{\mathcal{D}}{=} \tens{Q}_n$ and so $\tilde{\tilde{\tens{A}}}_n \overset{\mathcal{D}}{=} \tilde{\tens{A}}_n$. Consequently, we can assume without loss of generality that the permutations $\sigma_n$ are the identity and consider
    \[
        \tilde{\tens{B}}_n = \tens{P}_n \tilde{\tens{A}}_n \tens{P}_n,
    \]
    where $\tens{P}_n$ is a diagonal matrix with $(\tens{P}_n)_{ii} = 1$ if $i \leq k_n$ and $0$ otherwise.

    Because $\tens{Q}_n$ is orthogonal, the conjugation $\tilde{\tens{A}}_n = \tens{Q}_n^{\intercal} \tens{A}_n \tens{Q}_n$ leaves the eigenvalues invariant. Therefore, the empirical spectral distribution of $\tilde{\tens{A}}_n$ is identical to that of $\tens{A}_n$ for all $n$, and converges weakly to the same compactly supported probability measure $\nu$, which is uniquely characterised by its R-transform $R(z)$. Furthermore, $\tens{P}_n$ clearly has an empirical spectral distribution converging weakly to a Bernoulli distribution:
    \[
        \lim_{n\to\infty} \nu_{\tens{P}_n} = \nu_p = (1-\tau)\delta_0 + \tau \delta_1.
    \]
    Let $\tr_n(\tens{A}) = \frac{1}{n} \mbox{tr}(\tens{A})$ denote the normalized trace. We note that $\lim_{n\to\infty}\tr_n(\tens{P}_n) = \lim_{n\to\infty} \frac{k_n}{n} = \tau$. By Haar invariance, we may assume that $\tens{A}_n$ is diagonal with its eigenvalues in increasing order. The uniform bound and weak convergence imply convergence of all moments. Moreover, if $F_\nu^{-1}$ denotes the quantile function of $\nu$, then
    \[
        \lim_{n \to \infty} \tr_n(\tens{A}_n^r \tens{P}_n)
        = \int_0^\tau \bigl(F_\nu^{-1}(s)\bigr)^r \, \mathrm{d} s,
        \qquad r \in \{0, 1, 2, \dots\}.
    \]
    Since $\tens{A}_n$ and $\tens{P}_n$ commute and $\tens{P}_n^2 = \tens{P}_n$, their joint moments therefore converge to those of a pair $(a, p)$. The orthogonal asymptotic freeness theorem \citep[Theorem 5.2]{COLLINS-2006} then gives almost-sure convergence of $(\tens{A}_n, \tens{P}_n, \tens{Q}_n, \tens{Q}_n^\ast)$ in joint $\ast$-moments to $(a, p, u, u^\ast)$, where $u$ is a Haar unitary free from $(a, p)$. To show that $\{\tilde{\tens{A}}_n\}_{n=1}^\infty$ and $\{\tens{P}_n\}_{n=1}^\infty$ are asymptotically free, let $p_i(u^\ast a u)$ and $q_i(p)$ be centered polynomials such that $\tr_{\infty}(p_i) = \tr_{\infty}(q_i) = 0$. Note that $p_i(u^\ast a u) = u^\ast p_i(a) u$ since $u^\ast u$ is the identity. Now since $\tr_{\infty}(x y) = \tr_{\infty}(y x)$,
    \begin{align*}
        \tr_{\infty}(p_1(u^\ast a u)q_1(p)p_2(u^\ast a u)q_2(p)\cdots) &= \tr_{\infty}(u^\ast p_1(a) u q_1(p) u^\ast p_2(a) u q_2(p) \cdots) \\
        &= \tr_{\infty}(p_1(a)u q_1(p) u^\ast p_2(a) u q_2(p) \cdots u^\ast) = 0,
    \end{align*}
    since it is an alternating sequence of elements belonging to $\{a,p\}$ and $\{u,u^\ast\}$. Since $p_i$ and $q_i$ were arbitrary, $\tilde{a} = u^\ast a u$ and $p$ are freely independent.

    To identify the distribution of $\tilde{b} = p \tilde{a} p$, we note that $p^{1/2} = p$ and so the S-transform of $\tilde{b}$, denoted $S_{\tilde{b}}$ satisfies $S_{\tilde{b}} = S_{\tilde{a}} S_p = S_a S_p$. Therefore, we will need to find the S-transform of $p$.

    To do so, note that $m_n = \int x^n \nu_p(\mathrm{d} x) = \tau$ for every integer $n \geq 1$. Therefore, the truncated moment-generating function $M_p(z) = \sum_{n=1}^\infty \tau z^n = \frac{\tau z}{1-z}$. The inverse is given by $\chi_p(z) = \frac{z}{z + \tau}$ and so
    \[
        S_p(z) = \frac{z+1}{z} \chi_p(z) = \frac{z+1}{z+\tau}.
    \]
    Consequently,
    \begin{equation}
        \label{eq:NSStransform}
        S_{\tilde{b}}(z) = \frac{z + 1}{z + \tau} S_a(z).
    \end{equation}

    Now observe that $\tilde{\tens{B}}_n$ is a block matrix, with the top-left $k_n \times k_n$ block comprised of $\tens{B}_n$, and every other elements is zero. Consequently, $\mbox{tr}(\tens{B}_n^k) = \mbox{tr}(\tilde{\tens{B}}_n^k)$ for every integer $k \geq 1$, and so $\tr_n(\tilde{\tens{B}}_n^k) = \frac{1}{n} \mbox{tr}(\tens{B}_n^k) = \frac{k_n}{n} \tr_{k_n}(\tens{B}_n^k)$. The limit of $\tilde{\tens{B}}_n$ is $\tilde{b} = p \tilde{a} p$ and so $\tens{B}_n$ has a limit $b$ satisfying
    \[
        \tr_{\infty}(b^k) = \frac{1}{\tau} \tr_{\infty}(\tilde{b}^k).
    \]
    Therefore,
    \[
        M_{\tilde{b}}(z) = \sum_{k=1}^\infty \tr_{\infty}(\tilde{b}^k) z^k = \tau \sum_{k=1}^\infty \tr_{\infty}(b^k) z^k = \tau M_b(z).
    \]
    For the inverses, since $z = M_{\tilde{b}}(\chi_{\tilde{b}}(z)) = \tau M_b(\chi_{\tilde{b}}(z))$, we find that
    \[
        \chi_{\tilde{b}}(z) = \chi_b\left(\frac{z}{\tau}\right).
    \]
    Now observe that
    \begin{align*}
        S_{\tilde{b}}(z) &= \frac{z+1}{z} \chi_b\left(\frac{z}{\tau}\right) \\
        S_b\left(\frac{z}{\tau}\right) &= \frac{\frac{z}{\tau}+1}{\frac{z}{\tau}} \chi_b\left(\frac{z}{\tau}\right) = \frac{z+\tau}{z}\chi_b\left(\frac{z}{\tau}\right),
    \end{align*}
    and so by equating the two through $\chi_b$,
    \[
        S_{\tilde{b}}(z) = \frac{z+1}{z+\tau} S_b\left(\frac{z}{\tau}\right).
    \]
    Equating with \eqref{eq:NSStransform} gives $S_b(z) = S_a(\tau z)$, completing the first part of the result. Now we need to express this relationship in terms of the R-transform. To do so, recall from \Cref{lem:RSTrans} that if $y = z S_b(z)$, then $z = y R_b(y)$. Since $S_b(z) = S_a(\tau z)$, $y = z S_b(z) = z S_a(\tau z)$ implies that
    \[
        \tau y = (\tau z)S_a(\tau z).
    \]
    This immediately implies that
    \[
        \tau z = (\tau y)R_a(\tau y).
    \]
    But since $y = z S_b(z)$ also implies that $z = y R_b(y)$, there is
    \[
        \tau y R_b(y) = \tau y R_a(\tau y).
    \]
    Therefore, finally, $R_b(y) = R_a(\tau y) = R(\tau y)$. The uniform norm bound turns the almost-sure moment convergence into weak convergence, and the preceding rank argument transfers this conclusion to the original matrices, completing the proof.
\end{proof}


\subsection{Infinite Decompressibility} \label{sec:decompressibility}

The Nica--Speicher Theorem guarantees that the distribution obtained under free compression is well-defined through the R-transform by taking $\tau < 1$. However, it does not imply that the inverse operation, taking $\tau > 1$, is generally well-defined. Indeed, there are probability measures where free decompression need not be valid. In this section, we define the class of probability measures for which free decompression can be applied \emph{without bound} on $\tau$. This motivates the following definition.

\begin{definition}[Infinitely Decompressible Measure] \label{def:decompressible}
    A probability measure \(\nu\) with R-transform $R$ is \bemph{infinitely decompressible} if for every \(\tau \geq 1\), there exists a probability measure $\nu_\tau$ with R-transform satisfying $R_{\nu_\tau}(z) = R(\tau z)$.
\end{definition}

Since the R transform admits the series expansion
\begin{equation*}
    R_{\nu_\tau}(w) = \sum_{k \geq 1} r_k(\nu_\tau) w^{k-1},
\end{equation*}
where \(r_k(\nu_\tau)\) is the \(k\)-th \bemph{free cumulant} of \(\nu_\tau\), the Nica--Speicher Theorem implies
\begin{equation*}
    r_k(\nu_\tau) = \tau^{k-1} r_k(\nu), \qquad k \geq 1.
\end{equation*}

Let \(D_c \nu\) denote the \emph{dilation} of a measure \(\nu\), defined as the law of \(cX\) when \(X \sim \nu\). This is also written as \(c_{\ast} \nu\), as this operation is the \emph{push-forward} of \(\nu\) induced by the map \(x \mapsto cx\). Also, denote by \(\nu^{\boxplus s}\) the \(s\)-fold \emph{free additive convolution power}. The R-transform of dilation and convolution power operations satisfies the identities
\begin{subequations}
\begin{align}
    &R_{D_c \nu}(w) = c\, R_{\nu}(c w), \label{eq:dilation} \\
    &R_{\nu^{\boxplus s}}(w) = s \, R_{\nu}(w), \label{eq:power}
\end{align}
\end{subequations}
or, in terms of cumulants:
\begin{subequations}
\begin{align}
    &r_n(D_c \nu) = c^n r_n(\nu), \\
    &r_n\big(\nu^{\boxplus s}\big) = s r_n(\nu).
\end{align}
\end{subequations}
From the above, it is immediate that dilation and power operations commute. The fractional powers \(s\) in the above relations are due to \citet[Proposition 1.2]{NICA-1996} (see also \citep[Proposition 1.2]{SHLYAKHTENKO-2022}).

\begin{definition}
    The measure \(\nu \in \mathcal{P}(\mathbb{R})\) is \bemph{freely infinitely divisible} (FID, also called \(\boxplus\)-infinitely divisible), if for any \(n \in\mathbb{N}\), there exists \(\nu_n \in \mathcal{P}(\mathbb{R})\) such that
    \begin{equation*}
        \nu = \underbrace{\nu_n \boxplus \dots \boxplus \nu_n}_{n \text{ times}} = \nu_n^{\boxplus n}.
    \end{equation*}
\end{definition}

A necessary and sufficient condition for a measure to be FID is that its R transform should have a holomorphic extension defined on \(\mathbb{C}^{-}\) with values in \(\mathbb{C}^{-} \cup \mathbb{R}\) \citep[Theorem 5.10]{BERCOVICI-1993a}. The condition that $\nu$ is FID is closely related to free decompression through the following proposition.

\begin{proposition}[Infinitely Decompressible \(\iff\) FID] \label{prop:fid}
    A probability law \(\nu\) is infinitely decompressible if and only if it is freely infinitely divisible.
\end{proposition}

\begin{proof}
    (\emph{Sufficiency}) If \(\nu\) is FID, then for each \(\tau \geq 1\), the dilation \(D_\tau \nu\) is also FID, hence the free convolution power \(\nu_\tau \coloneqq \left(D_\tau \nu\right)^{\boxplus \tau^{-1}}\) exists. Using \eqref{eq:dilation} and \eqref{eq:power}, we obtain
    \begin{equation*}
        R_{\nu_\tau}(w) = R_{\left(D_{\tau} \nu \right)^{\boxplus \tau^{-1}}}(w) = \tau^{-1} R_{D_{\tau} \nu}(w) = R_{\nu}(\tau w),
    \end{equation*}
    so the free decompression flow exists.

    (\emph{Necessity}) Assume \(\nu\) is infinitely decompressible. Fix \(n \in \mathbb{N}\) and set \(\tau = n\), yielding a measure \(\nu_n\) such that \(R_{\nu_n}(w) = R_{\nu}(n w)\). Let \(\sigma_{n} \coloneqq D_{n^{-1}} \nu_n\). Then \(R_{\sigma_{n}}(w) = n^{-1} R_{\nu_n}(n^{-1} w) = n^{-1} R_{\nu}(w)\), hence \(R_{\sigma_{n}^{\boxplus n}}(w) = R_{\nu}(w)\) and thus \(\nu = \sigma_{n}^{\boxplus n}\). Since this holds for all \(n\), \(\nu\) is FID.
\end{proof}

We also note that any FID distribution can be approximated by compound free Poisson distributions \cite[p. 117]{SPEICHER-2015} given in \Cref{sec:compound-poisson}.


\subsection{The PDE Approach and Challenges} 
\label{sec:bg-pde-hard}

Let $\nu$ be a probability measure and let $\{\nu_\tau\}_{\tau \geq 1}$ be a family of probability measures obtained by free decompression such that for each $\tau$, denoting by $R_\tau$ the R-transform of $\nu_\tau$, $R_\tau(z) = R(\tau z)$, where $R$ is the R-transform of $\nu$. For each $\tau \geq 1$, let $m_\tau$ denote the Stieltjes transform corresponding to $\nu_\tau$. \citet{AMELI-2025b} proved that  $m(t,z) \coloneqq m_{e^t}(z)$ satisfies the PDE
\begin{equation}
    \label{eq:FreeDecPDE}
    \frac{\partial m(t,z)}{\partial t} = -m(t,z) + m(t,z)^{-1} \frac{\partial m(t,z)}{\partial z}.
\end{equation}
In the following exposition, the time change $\tau = e^t$ will be used repeatedly with reference to the flow \eqref{eq:FreeDecPDE}. 
While free decompression is always achievable in principle via \eqref{eq:FreeDecPDE}, \emph{in practice}, it is \textbf{extraordinarily difficult} to perform for several reasons. The primary difficulty lies in the ill-posedness of \eqref{eq:FreeDecPDE}, which can be seen in two ways. 

The first is through the relationship between \eqref{eq:FreeDecPDE} and the Burgers equation: if $\tau = e^t - 1$ and $u(\tau, z) = -e^{-t} / m(\tau, z)$, then $u$ satisfies the inviscid complex Burgers equation
$\partial_\tau u + u \partial_z u = 0$. 
This equation is well-studied in the setting where $m$ is real-valued, including in related problems in RMT \citep{BLAIZOT-2010, BLAIZOT-2013, BLAIZOT-2015}, and a number of techniques have been developed to obtain accurate approximations of solutions. 
Arguably the most important consideration is the development of \emph{shocks}, where after some time $t$, multiple solutions emerge. 
These shocks present fatal instabilities when solving the PDE directly. 
To deal with this issue, one often introduces regularization through the viscous complex Burgers equation $\partial_\tau u + u \partial_z u = \xi \partial_{zz} u$, taking $\xi$ to be small \citep{SENOUF-1996}. 
This equation has a unique explicit solution, called the \emph{Cole-Hopf solution}, and taking $\xi \to 0^+$ yields the \emph{entropy solution} to the inviscid Burgers equation. 
In the real-valued case, the entropy solution is given by the Hopf-Lax formula. 
However, the current setting requires that $m$ is complex valued, and methods developed for the real-valued setting \citep{SENOUF-1996} cannot be readily deployed. 
When performing free decompression, shocks inevitably emerge, and while the Cole-Hopf solution can be computed, it need not converge as $\xi \to 0^+$ to the solution where $m(\tau,z)$ remains a valid Stieltjes transform. 
To our knowledge, no examination of the complex inviscid Burgers equation has presented a method to reliably estimate this particular solution.

The second is encountered when applying the method of characteristics to solve \eqref{eq:FreeDecPDE}. This is the approach ultimately taken in \citet{AMELI-2025b}, but it is subject to several limitations. The characteristic curves for \eqref{eq:FreeDecPDE} are straightforward to obtain, but at the same time where shocks occur in the Burgers equation formulation, the characteristic curves cross a branch cut in the complex plane. To extend beyond this cut requires \emph{analytic continuation} of the Stieltjes transform $m$. Numerical analytic continuation is notoriously ill-posed, with worst-case error from na\"{i}ve approaches growing doubly exponentially fast in $t$, or exponentially fast in the original matrix size \citep{TREFETHEN-2023}. In reality, 
free decompression exhibits worst-case error that grows no more than polynomially in the matrix size \citep[Proposition 3]{AMELI-2025b}, and \citet{AMELI-2025b} present an approach using polynomial approximation with error rates often growing only \emph{logarithmically} in the size of the original matrix, with the run-time complexity being independent of $t$. 
At these rates, free decompression becomes impressively practical. 
Unfortunately, the approach is mathematically brittle, highly sensitive to the accuracy of the curve-fit used to obtain the initial data, and imposes assumptions on the form of the spectral density that rarely hold outside of the simplest synthetic cases. 
Our primary contribution is to identify an approach that maintains this high level of performance under less restrictive assumptions, rendering free decompression practical for more realistic examples.


\subsection{Formal Free Decompression along Spectral Curves}
\label{sec:formal}

Now, recall that the functional relation for the R-transform in terms of the Stieltjes transform is given by \citet[Definition 5]{COUILLET-2022} as
\[
    R(-m(z)) = z + \frac{1}{m(z)}.
\]
According to the Nica--Speicher Theorem, the operation of free decompression (inverting the operation of free compression) coincides with the map $\nu \mapsto \nu_\tau$ where the R-transforms $R$ and $R_\tau$ of $\nu$ and $\nu_\tau$, respectively, satisfies
\[ 
    R_\tau(z) = R(\tau z),\qquad \tau > 1.
\]
If $m$ is defined to satisfy the formal algebraic relation $P(z,m) = 0$ for some polynomial $P$, then, formally, we can track the behavior of the Stieltjes transform of $\nu_\tau$ under free decompression. In doing so, we can track solutions to free decompression while avoiding any potential issues with $m$ and/or $R$ being multivalued. This is captured in \Cref{thm:FormalFreeDecomp}.

\begin{theorem} \label{thm:FormalFreeDecomp}
    For a polynomial $P(z,m)$, define the following spectral curves:
    \begin{align*}
        \mathcal{C}_{m} & =\{(z,m) \in \mathbb{C}^2 \mid P(z,m)=0\}, & \text{(graph of }m\text{)} \\
        \mathcal{C}_{R} & =\{(-m,z+1/m) \mid (z,m)\in\mathcal{C}_{m},\, m \neq 0\}, & \text{(R transform of }m\text{)} \\
        \mathcal{C}_{R_{\tau}} &= \{(w,r) \in \mathbb{C}^2 \mid (\tau w,r)\in\mathcal{C}_{R}\}, & \text{(R transform of }m_\tau\text{)} \\
        \mathcal{C}_{m_{\tau}} &= \{(z,\omega) \in \mathbb{C}^2 \mid (-\omega,z+1/\omega)\in\mathcal{C}_{R_{\tau}}, \omega \neq 0\}. & \text{(graph of }m_\tau\text{)}
    \end{align*}
    Then $$\mathcal{C}_{m_\tau} \subseteq \mathcal{C}_{P_\tau} \coloneqq \left\{(z,\omega) \in \mathbb{C}^2 \,\middle|\, \omega \neq 0\text{ and }P\left(z+\frac{\tau-1}{\tau\omega},\tau\omega\right)=0\right\}.$$
\end{theorem}

\begin{proof}
    We begin by unpacking $\mathcal{C}_{m_\tau}$. Note that $(z,\omega) \in \mathcal{C}_{m_\tau}$ if and only if, with $\omega \neq 0$, $(-\omega, z+1/\omega) \in \mathcal{C}_{R_\tau}$. This implies that $(-\tau\omega, z + 1/\omega) \in \mathcal{C}_R$, which in turn implies that there exists $(z_0,m) \in \mathcal{C}_m$ such that
    \[
        -m = -\tau \omega\quad \mbox{and}\quad z_0 + \frac{1}{m} = z + \frac{1}{\omega}.
    \]
    Consequently, 
    \[
        m = \tau \omega\quad \mbox{and}\quad z_0 = z + \frac{\tau - 1}{\tau \omega}.
    \]
    Since $(z_0,m) \in \mathcal{C}_m$ implies that $P(z_0, m) = 0$, this implies that there exists $z_0$ such that
    \[
        P(z_0, \tau \omega) = 0\quad\mbox{and}\quad z_0 = z + \frac{\tau - 1}{\tau \omega}.
    \]
    Consequently, any $(z,\omega) \in \mathcal{C}_{m_\tau}$ must also satisfy $(z,\omega) \in \mathcal{C}_{P_\tau}$. 
\end{proof}

\Cref{thm:FormalFreeDecomp} immediately implies \Cref{thm:SpectralCurve}, as we identify $m_\tau$ to be the free decompression of $m$ by the factor $\tau$.

\subsection{Polynomial Evolution under Free Decompression}
\label{sec:coeffmethod}

In this section, we use \Cref{thm:FormalFreeDecomp} to derive a bivariate polynomial $Q_\tau$ such that the Stieltjes transform $m_\tau$ under free decompression by a decompression ratio $\tau$ satisfies $Q_\tau(z, m_\tau(z)) = 0$ for almost all $z$. Since this polynomial can be readily obtained from $P$ and the roots obtained using standard methods, it provides a fast method of constructing candidate solutions for $m_\tau$. The polynomial is constructed in \Cref{cor:CoefficientMethod}. 

\begin{corollary} \label{cor:CoefficientMethod}
    Suppose that $m_0: \mathbb{C}^+ \to \mathbb{C}^+$ satisfies an algebraic relation of the form $P(z,m_0(z)) = 0$ for $z \in \mathbb{C}^+$ where
    \[
        P(z,m) = \sum_{j=0}^{d_z} \sum_{k=0}^s c_{jk} z^j m^k.
    \]
    Then for $\mathcal{C}_{m_\tau}$ defined in \Cref{thm:FormalFreeDecomp}, 
    \[
        \mathcal{C}_{m_\tau} \subseteq \left\{ (z,m) \in \mathbb{C}^2 \, \middle|\, Q_\tau(z,m) = 0 \right\},
    \]
    where $Q_\tau(z,m)$ is the polynomial given by
    \[
        Q_\tau(z,m) \coloneqq \sum_{k=0}^{d_z} \sum_{l=0}^{d_z+s} c_{kl}^{(\tau)} z^k m^l,
    \]
    and
    \[
        c_{pq}^{(\tau)} \coloneqq \sum_{j=p}^{d_z} c_{j,q-d_z+j-p} \tau^{q-d_z+j-p}\binom{j}{p} (1-\tau^{-1})^{j-p},
    \]
    where $c_{ij} = 0$ if $j \notin \{0,\dots,s\}$. 
\end{corollary}

\begin{proof}
    From \Cref{thm:FormalFreeDecomp}, we have the inclusion
    \[
        \mathcal{C}_{m_\tau} \subseteq \left\{(z,m) \in \mathbb{C}^2 \, \middle| \, P\left(z + \frac{\tau - 1}{\tau m}, \tau m\right) = 0\right\}.
    \]
    In particular, for $(z,m) \in \mathcal{C}_{m_\tau}$, $m \neq 0$ and
    \[
        m^{d_z} P\left(z + \frac{\tau - 1}{\tau m},\tau m\right) = 0.
    \]
    We will show that the left hand side of this expression is the polynomial $Q_\tau(z,m)$. First note that
    \[
        P\left(z + \frac{\tau - 1}{\tau m}, \tau m\right) = \sum_{j=0}^{d_z} \sum_{k=0}^s c_{jk} \left(z + \frac{\tau - 1}{\tau m}\right)^j (\tau m)^k.
    \]
    Expanding the binomial gives
    \[
        \left(z + \frac{\tau - 1}{\tau m}\right)^j = \sum_{i=0}^j \binom{j}{i} z^{j - i} \left(\frac{\tau - 1}{\tau m}\right)^i = \sum_{i=0}^j \binom{j}{i} z^{j-i} (\tau - 1)^i \tau^{-i} m^{-i}.
    \]
    Therefore,
    \[
        m^{d_z} P\left(z + \frac{\tau - 1}{\tau m}, \tau m\right) = \sum_{j=0}^{d_z} \sum_{k=0}^s  \sum_{i=0}^j c_{jk} \binom{j}{i} (\tau - 1)^i \tau^{k-i} z^{j-i} m^{k-i+d_z}.
    \]
    Now we introduce new summation indices in place of $i,k$ by letting $p = j - i$ and $q = k - i + d_z$. Then $i = j - p$ and $k = q + i - d_z = q - d_z + j - p$ and $p \in \{0,\dots,d_z\}$, $q \in \{0,\dots,d_z+s\}$. Using these indices, since $j \geq p$ and $\binom{j}{i} = \binom{j}{j-i}$,
    \[
        m^{d_z} P\left(z + \frac{\tau - 1}{\tau m}, \tau m\right) = \sum_{p=0}^{d_z} \sum_{q=0}^{d_z + s}\sum_{j=p}^{d_z} c_{j, q-d_z+j-p} \binom{j}{p} (\tau - 1)^{j-p} \tau^{q-d_z} z^p m^q,
    \]
    where we have adopted the convention that $c_{j,k} = 0$ if $k \notin \{0,\dots,s\}$. Since $\tau^{q-d_z} (\tau - 1)^{j-p} = \tau^{q-d_z+j-p} (1-\tau^{-1})^{j-p}$. Hence
    \[
        m^{d_z} P\left(z + \frac{\tau - 1}{\tau m}, \tau m\right) = \sum_{p=0}^{d_z} \sum_{q=0}^{d_z + s} c_{pq}^{(\tau)} z^p m^q,
    \]
    as required.
\end{proof}

As stated, $Q_\tau$ need not have the same degree in $m$ as the initial polynomial $P$. However, we recall that polynomial relations of this form are not unique. We reduce $Q_\tau$ by dividing out monomial factors introduced when clearing denominators, and fix the remaining scalar ambiguity by normalizing a nonzero coefficient. We refer to the resulting polynomial as $P_\tau$. For common examples (e.g., Marchenko-Pastur or semicircle laws), $P_\tau$ often has the same degree as $P$, due to \Cref{lem:inflation} below.

\begin{lemma}\label{lem:inflation}
    If the index set $\mathcal{A}$ for $P$ is upper-triangular, then $\deg_m(P_\tau) \leq \deg_m(P)$.
\end{lemma}

\begin{proof}
    Note that it will suffice to show that $P(z + (1-\tau^{-1})m^{-1}, \tau m)$ is a bivariate polynomial in $z$ and $m$, as this will imply that $m^{d_z}$ can be factorized from $Q_\tau$. If $\mathcal{A}$ is upper-triangular, then every surviving monomial $z^j m^k$ in the expansion for $P$ satisfies $k \geq j$. But this implies that
    \begin{align*}
        P\left(z + \frac{\tau - 1}{\tau m}, \tau m\right) &= \sum_{j=0}^{d_z} \sum_{k=0}^{s-j} c_{j,j+k} \left(z + \frac{\tau - 1}{\tau m}\right)^j (\tau m)^{j+k} \\ &= \sum_{j=0}^{d_z} \sum_{k=0}^{s-j} c_{j,j+k} (\tau z m + \tau - 1)^j (\tau m)^{k},
    \end{align*}
    which has only positive powers of $z$ and $m$.
\end{proof}
For a bivariate polynomial of the form $P(z, m) = q(z) m^2 - p(z) m + 1$, the condition in \Cref{lem:inflation} suggests that \(\deg_z(q) \leq 2\) and \(\deg_z(p) \leq 1\) implies that $P_\tau$ does not have a higher degree than $P$. Many ensemble models exhibit bivariate polynomial relations of this form, including \citep{WIGNER-1955} semicircle law (free Gaussian); \citep{MARCHENKO-1967} (free Poisson); \citep{KESTEN-1959} and \citep{MCKAY-1981} (free binomial); \cite{WACHTER-1978} (free Jacobi); and the general free-Meixner family \citep{SAITOH-2001,ANSHELEVICH-2008} (free analogs of the classical \cite{MEIXNER-1934} laws). The \(Q\) and \(P\) for these ensemble models can be found in \citep[Table B.3]{AMELI-2025b}. 

However, $P$ need not be upper-triangular in general; for example, the deformed Wigner ensemble \citep[Sections 2.2 and 18.3]{PASTUR-2011} is not upper-triangular. This increase in degree exacerbates the most significant difficulty with performing free decompression using \Cref{cor:CoefficientMethod}: identifying the correct root of $P_\tau$ is extraordinarily challenging. For this reason, we will require a more refined approach by tracing along spectral curves to identify the correct root.


\subsection{Free Decompression via Geometric Continuation} \label{sec:fd-on-sc}

The previous subsection shows that, once the initial algebraic curve \(P(z, m) = 0\) is given, free decompression produces an evolved algebraic relation whose roots provide all candidate values of the transformed Stieltjes function at time \(\tau\). However, this polynomial evolution alone does not identify the physical branch. This is illustrated in \Cref{fig:diffusion-candidates-final}(a), where the roots of the evolved polynomial form a cloud of candidate densities, only one of which agrees with the empirical spectrum.

\begin{figure}[t]
    \centering
    \includegraphics[width=\textwidth]{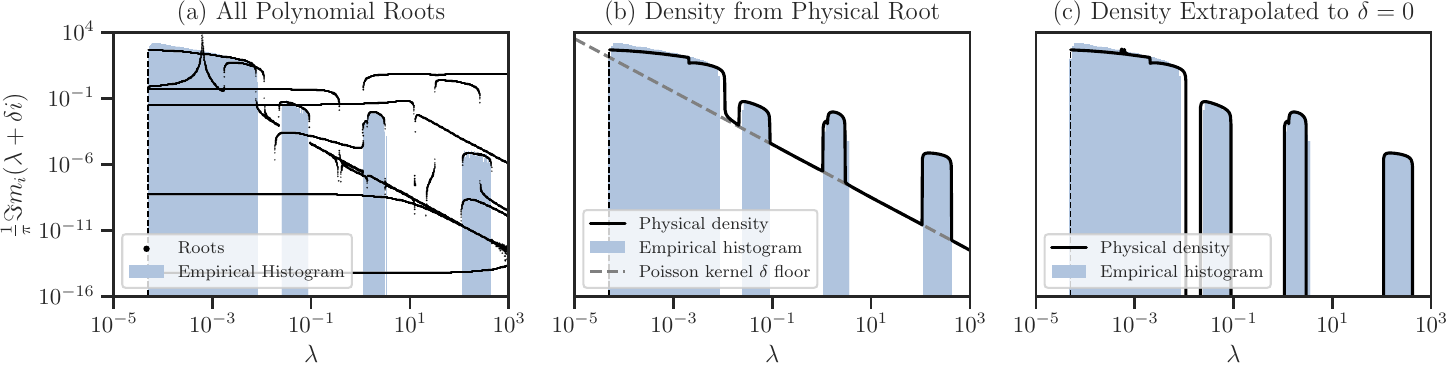}
    \caption{Free decompression on the fitted algebraic spectral curve for the diffusion-model example at the final size \(n = \num{64000}\), evolved from the initial \(n = \num{4000}\) submatrix. Panel (a) shows all candidate densities \(\pi^{-1}\Im m_i(\lambda+i\delta)\) with \(\delta=10^{-6}\), obtained from the seven roots \(m_i\) of the evolved polynomial relation with \(\deg_m(P) = 7\), at the final decompression time. These unordered candidates form a cloud of physical and non-physical sheets. Panel (b) shows the result after tracking the physical branch along the free-decompression flow, yielding a coherent density curve that matches the empirical histogram on the bulks, but still exhibits the finite-\(\delta\) Poisson floor in the spectral gaps. Panel (c) shows the final density after \(\delta\)-extrapolation to \(\delta=0\) using four \(\delta\)-levels and a quadratic Chebyshev fit, which removes the artificial gap baseline while preserving the bulk densities.}
    \label{fig:diffusion-candidates-final}
\end{figure}

This root-selection problem is fundamentally different from the one in \Cref{sec:root-selection}. There, for the initial curve, one may choose an anchor point directly in the \(z\)-plane and continue along that curve. After free decompression, by contrast, the evolved polynomial at a fixed \(\tau>1\) provides no new anchor identifying the physical sheet. Instead, the physical branch must be continued from the known physical branch at \(\tau = 1\) along the free-decompression flow itself.

We now formulate this continuation geometrically. We write \(m_{\tau}(z)\) for the Stieltjes transform after decompression ratio \(\tau\), equivalently at time \(t=\log \tau\). In particular, \(m_0\) denotes the initial Stieltjes transform at \(t=0\), or \(\tau=1\), and should not be confused with the indexed algebraic roots \(m_i\) used earlier. Let the initial fitted spectral curve be
\begin{equation*}
    \mathcal{C} = \left\{ (\zeta, y) \in \mathbb{C}^2 \mid P(\zeta, y) = 0 \right\}.
\end{equation*}
Here, \(y=m_0(\zeta)\) denotes a branch of the initial algebraic Stieltjes transform. Under free decompression by ratio \(\tau \geq 1\), the method of characteristics gives the change of variables
\begin{equation*}
    z = \zeta - \frac{\tau - 1}{y},
    \qquad
    w = \frac{y}{\tau},
\end{equation*}
where \(w = m_{\tau}(z)\) is the decompressed Stieltjes value. Equivalently,
\begin{equation*}
    \zeta = z + \frac{\tau - 1}{\tau w},
    \qquad
    y = \tau w.
\end{equation*}
Substituting this inverse map into \(P(\zeta, y) = 0\) gives the evolved algebraic relation
\begin{equation*}
    P\left(z + \frac{\tau - 1}{\tau w}, \tau w \right) = 0.
\end{equation*}
This is precisely the relation obtained in \Cref{thm:FormalFreeDecomp}. After clearing denominators, it gives the evolved polynomial whose roots provide all algebraic candidates for \(w=m_{\tau}(z)\), but not the physical one.

To track the physical root, we work instead in the original curve coordinates \((\zeta, y) \in \mathcal{C}\). For a fixed target point \(z\), define
\begin{align*}
    F_1(\zeta, y) &\coloneqq P(\zeta, y), \\
    F_2(\zeta, y; \tau, z) &\coloneqq \zeta - \frac{\tau - 1}{y} - z.
\end{align*}
Thus, a point \((\zeta, y) \in \mathcal{C}\) maps to the target \(z\) at decompression ratio \(\tau\) precisely when
\begin{equation*}
    F_1(\zeta, y) = 0,
    \qquad
    F_2(\zeta, y; \tau, z) = 0.
\end{equation*}
At \(\tau = 1\), this system reduces to \(\zeta = z\) and \(y = m_0(z)\), so the initial physical point is known from the branch-selection procedure in \Cref{sec:root-selection}. The task is then to continue this solution in \(\tau\), while keeping \(z\) fixed.

In practice, evolution along the spectral curve can pass through regions where sheets come close, roots change order, or the geometry becomes stiff. To resolve this, we evolve the above system using a predictor--corrector approach, where the predictor uses the local tangent along the constrained curve. Namely, differentiating \(F_1=0\) with respect to \(\tau\), while keeping \(z\) fixed, gives
\begin{equation*}
    \partial_{\zeta}P(\zeta,y)\dot{\zeta} + \partial_y P(\zeta,y)\dot{y} = 0.
\end{equation*}
Similarly, differentiating \(F_2=0\), namely
\begin{equation*}
    \zeta - \frac{\tau - 1}{y} - z = 0,
\end{equation*}
gives
\begin{equation*}
    \dot{\zeta} - y^{-1} + (\tau - 1)y^{-2}\dot{y} = 0,
\end{equation*}
or equivalently,
\begin{equation*}
    \dot{\zeta} + (\tau - 1)y^{-2}\dot{y} = y^{-1}.
\end{equation*}
Combining these two relations gives the \(2 \times 2\) tangent system
\begin{equation*}
    \begin{bmatrix}
        \partial_{\zeta} P(\zeta, y) & \partial_y P(\zeta, y) \\
        1 & (\tau - 1)y^{-2}
    \end{bmatrix}
    \begin{bmatrix}
        \dot{\zeta} \\
        \dot{y}
    \end{bmatrix}
    =
    \begin{bmatrix}
        0 \\
        y^{-1}
    \end{bmatrix}.
\end{equation*}
This tangent system provides the predictor direction for advancing the physical point from one value of \(\tau\) to the next. The predicted point is then corrected by Newton iteration on the nonlinear system \((F_1,F_2)=(0,0)\). After correction, the decompressed Stieltjes value is recovered as
\begin{equation*}
    m_{\tau}(z) = w = \frac{y}{\tau}.
\end{equation*}
The resulting geometric continuation procedure is summarized in \Cref{alg:decompress}. In the implementation, the Newton correction is damped and the \(\tau\)-steps are adaptively subdivided when the local geometry becomes stiff; these safeguards prevent jumps between nearby sheets without changing the underlying geometric continuation described above.

The effect of this procedure is shown in \Cref{fig:diffusion-candidates-final}. The evolved polynomial alone produces the unordered candidate cloud in panel (a), while the geometric continuation in \(\tau\) selects the physical sheet shown in panel (b). This is a stringent test case: the density is log-scale, multi-bulk, and highly unbalanced, with nearby sheets interacting over several orders of magnitude. Nevertheless, the continuation remains on the physical branch and produces a coherent density matching the empirical bulks; the remaining finite-\(\delta\) floor is then removed by the extrapolation step in panel (c).

\begin{algorithm}[ht!]
    \caption{Free decompression by geometric continuation on a spectral curve}
    \label{alg:decompress}
    \SetKwInOut{Input}{Input}
    \SetKwInOut{Output}{Output}
    \DontPrintSemicolon

    \vspace{1mm}
    \Input{
        Polynomial \(P \in \mathbb{R}[z,m]\) such that \(P(z,m_0(z))=0\) (see \Cref{alg:poly}); \\
        Query grid \(\{z_j\}_{j=1}^{N_z}\subset\mathbb{C}^{+}\); \\
        Decompression ratios \(1=\tau_0<\tau_1<\cdots<\tau_T\); \\
        Oracle for the initial physical branch \(m_0(z)\) on \(P(z, m) = 0\) (see \Cref{alg:stieltjes}); \\
        Newton tolerance \(\epsilon>0\).
    }

    \vspace{2mm}
    \Output{
        Physical Stieltjes values \(m_{\tau_{\ell}}(z_j)\) for all requested \(\tau_{\ell}\) and \(z_j\).
    }

    \vspace{2mm}
    \tcp{Initialize the physical sheet at \(\tau_0=1\)}
    \For{\(j = 1,\dots,N_z\)}{
        \(y_j^{(0)} \leftarrow m_0(z_j)\) \tcp*{Select physical root using \Cref{sec:root-selection}}
        \(\zeta_j^{(0)} \leftarrow z_j\) \tcp*{At \(\tau=1\), the map is identity in \(z\)}
        \(m_{\tau_0}(z_j) \leftarrow y_j^{(0)}\) \;
    }

    \vspace{2mm}
    \tcp{Continue the physical sheet in decompression ratio \(\tau\)}
    \For{\(\ell = 1,\dots,T\)}{
        \For{\(j = 1,\dots,N_z\)}{
            Set \((\zeta,y) \leftarrow (\zeta_j^{(\ell-1)}, y_j^{(\ell-1)})\). \;

            \vspace{2mm}
            \tcp{Predict along the tangent of the constrained system}
            Solve
            \begin{equation*}
                \begin{bmatrix}
                    \partial_{\zeta}P(\zeta,y) & \partial_y P(\zeta,y) \\
                    1 & (\tau_{\ell-1}-1)y^{-2}
                \end{bmatrix}
                \begin{bmatrix}
                    \dot{\zeta}\\
                    \dot{y}
                \end{bmatrix}
                =
                \begin{bmatrix}
                    0\\
                    y^{-1}
                \end{bmatrix}.
            \end{equation*}

            \(\Delta \tau \leftarrow \tau_{\ell}-\tau_{\ell-1}\). \;
            \(\zeta_{\mathrm{pred}} \leftarrow \zeta + \Delta\tau\,\dot{\zeta}\). \;
            \(y_{\mathrm{pred}} \leftarrow y + \Delta\tau\,\dot{y}\). \;

            \vspace{2mm}
            \tcp{Correct by projecting back to the decompression system}
            Starting from \((\zeta_{\mathrm{pred}},y_{\mathrm{pred}})\), apply damped Newton to solve
            \begin{equation*}
                P(\zeta,y)=0,
                \qquad
                \zeta - (\tau_{\ell}-1)y^{-1} - z_j = 0.
            \end{equation*}

            \vspace{1mm}
            \eIf{Newton converges to \((\zeta_{\ast},y_{\ast})\) within tolerance \(\epsilon\)}{
                \(\zeta_j^{(\ell)} \leftarrow \zeta_{\ast}\), \quad
                \(y_j^{(\ell)} \leftarrow y_{\ast}\). \;
                \(m_{\tau_{\ell}}(z_j) \leftarrow \tau_{\ell}^{-1} y_{\ast}\). \tcp*{Recover evolved Stieltjes value}
            }{
                Subdivide \([\tau_{\ell-1},\tau_{\ell}]\) and repeat the predictor--corrector step. \tcp*{Avoid sheet jumping in stiff regions}
            }
        }
    }

    \vspace{2mm}
    \Return{\(\{m_{\tau_{\ell}}(z_j)\}_{\ell,j}\).}\;
\end{algorithm}


\section{Evolution of Spectral Quantities}
\label{sec:evo-features}

Free decompression primarily evolves the full density \(\rho(t, x)\), or equivalently its Stieltjes transform \(m(t, z)\). In many applications, however, one only needs derived spectral quantities such as edge locations, cusps (merging/splitting points), gap widths, atoms' weights and locations, or spectral moments. This section shows that these quantities can often be evolved directly from the spectral curve, without reconstructing the density pointwise.


\subsection{Evolving Spectral Edges} 
\label{sec:edge}

The \emph{spectral edges} of a density are of particular interest; for instance, \citep{BONNAIRE-2025} shows that the locations of bulk edges play an important role in the dynamics of diffusion models during training. Informally, if the support of the density \(\rho(t, x)\) is a union of \(k\) disjoint intervals
\begin{equation*}
    I(t) \coloneqq \bigcup_{j=1}^{k} I_j(t), \qquad I_j(t) \coloneqq[a_j(t), b_j(t)],
\end{equation*}
then the edges are the endpoints \(\{a_j(t), b_j(t)\}_{j=1}^{k}\), where \(\rho(t, x)\) vanishes outside \(I(t)\) and is positive inside the interior of each \(I_{j}(t)\).

When the Stieltjes transform \(m(t, z)\) is algebraic, the notion of edge admits a precise characterization in terms of the \emph{branch points} of the spectral curve. Concretely, suppose \(m\) is a branch \(m = m_1(z)\) of the algebraic relation \(P(z, m) = 0\), and define the spectral curve
\begin{equation*}
    \mathcal{C} \coloneqq \left\{ (\zeta, y) \in \mathbb{C}^2 \mid P(\zeta, y) = 0 \right\}.
\end{equation*}
The curve \(\mathcal{C}\) defines a branched covering of the \(\zeta\)-plane, with local sheets given by the roots \(y = y_i(\zeta)\) of \(P(\zeta, y) = 0\). A \emph{branch point} is a point \((\zeta_{\ast}, y_{\ast}) \in \mathcal{C}\) where two (or more) local sheets coalesce; equivalently,
\begin{equation*}
    P(\zeta_{\ast}, y_{\ast}) = 0,
    \qquad \text{and} \qquad
    \partial_y P(\zeta_{\ast}, y_{\ast}) = 0.
\end{equation*}
Branch points determine how sheets glue across branch cuts.

Among branch points, \emph{spectral edges} are distinguished by the \emph{physical sheet}. A real branch point \(x_{\ast} \in \mathbb{R}\) is a spectral edge if the associated branch cut connects the physical sheet to a non-physical sheet, so that crossing the cut switches between \(m_1\) and a different branch. In contrast, other branch points whose cuts glue two non-physical sheets do not bound the support of the density (even if they are real branch points). In multi-bulk situations, this distinction is essential: not every real branch point corresponds to an endpoint of the support.

With this notion in mind, we first characterize the evolution of branch points under free decompression, and then interpret which of the resulting real branch points correspond to spectral edges via their sheet connectivity to the physical branch.

\begin{proposition}[Evolving Branch Point under Free Decompression] \label{prop:fd-branch-point}
    Let \((z_{\ast}, m_{\ast}) \in \mathbb{C}^2\) be a branch point of \(m(t, z)\) with \(m_{\ast} \coloneqq m(t, z_{\ast})\). Define the pullback \((\zeta_{\ast}, y_{\ast}) = \phi_t^{-1}(z_{\ast}, m_{\ast})\), and assume that it is finite with \(y_{\ast} \neq 0\). Then, \((\zeta_{\ast}, y_{\ast}) \in \mathcal{C}\) satisfies
    \begin{subequations} \label{eq:fd-branch-point}
    \begin{align}
        &P(\zeta_{\ast}, y_{\ast}) = 0, \\
        &y_{\ast}^2\ \partial_y P(\zeta_{\ast}, y_{\ast}) - (\tau - 1)\ \partial_{\zeta} P(\zeta_{\ast}, y_{\ast}) = 0,
    \end{align}
    \end{subequations}
    followed by the pushforward \(z_{\ast} = \zeta_{\ast} - (\tau - 1) y_{\ast}^{-1}\). In particular, real solutions \(z_{\ast} \in \mathbb{R}\) of \eqref{eq:fd-branch-point} include the evolving spectral edges of the real density.
\end{proposition}

\begin{proof}
    The spectral edges of the real density at time \(t\) correspond to real branch points of the map \((\zeta, y) \mapsto z = \zeta - (\tau - 1) y^{-1}\) restricted to the spectral curve \(\mathcal{C}\).
    
    Along the spectral curve \(\mathcal{C}\), implicit differentiation of \(P(\zeta, y) = 0\) gives
    \begin{equation*}
        \mathrm{d} P(\zeta, y) = \frac{\partial P(\zeta, y)}{\partial \zeta} \mathrm{d} \zeta + \frac{\partial P(\zeta, y)}{\partial y} \mathrm{d} y.
    \end{equation*}
    On the spectral curve, \(P(\zeta, y) = 0\), hence \(\mathrm{d} P(\zeta, y) = 0\). Whenever \(\partial_y P(\zeta, y) \neq 0\), it holds
    \begin{equation*}
        \frac{\mathrm{d} y}{\mathrm{d} \zeta} = -\frac{\partial_{\zeta} P(\zeta, y)}{\partial_y P(\zeta, y)}.
    \end{equation*}
    On the other hand, differentiating the projection \(\pi \circ \phi_t(\zeta, y) = z(\zeta, y) = \zeta - (\tau - 1) y^{-1}\) along \(\mathcal{C}\) yields
    \begin{equation*}
        \frac{\mathrm{d} z}{\mathrm{d} \zeta} = 1 + (\tau - 1) y^{-2} \frac{\mathrm{d} y}{\mathrm{d} \zeta} = 1 - (\tau - 1) y^{-2} \frac{\partial_{\zeta} P(\zeta, y)}{\partial_y P(\zeta, y)}.
    \end{equation*}
    The critical points of the projection \(z = \pi \circ \phi_t\) are characterized by \(\mathrm{d} z / \mathrm{d} \zeta = 0\) together with \(P(\zeta, y) = 0\), which is equivalent to \eqref{eq:fd-branch-point}.
\end{proof}

Since spectral edges are a subset of branch points, \Cref{prop:fd-branch-point} applies in particular to the evolution of the edges \(\{a_j(t), b_j(t)\}_{j=1}^k\). For each fixed \(t \geq 0\), the algebraic system \eqref{eq:fd-branch-point} may admit a set of solutions
\begin{equation*}
    \mathcal{B}(t) \coloneqq \left\{ z_{\ast}^{(j)}(t) \right\}_{j=1}^p,
\end{equation*}
where typically \(p \geq k\). The real subset \(\mathcal{B}(t) \cap \mathbb{R}\) provides candidates for spectral edges; one then selects those candidates whose cuts connect the physical sheet to a non-physical sheet. Operationally, this can be checked by whether the physical branch exhibits a nontrivial jump across the real axis at that point (equivalently, by the Stieltjes inversion formula): a candidate \(x_{\ast} \in \mathbb{R}\) corresponds to an edge of the density precisely when the boundary values \(m_1(t, x_{\ast} + i 0^{+})\) and \(m_1(t, x_{\ast} + i 0^{-})\) switch sheets across the associated cut. \Cref{alg:edge} summarizes the resulting procedure.

\begin{algorithm}[ht!]
    \caption{Evolution of spectral edges under algebraic free decompression}
    \label{alg:edge}
    \SetKwInOut{Input}{Input}
    \SetKwInOut{Output}{Output}
    \DontPrintSemicolon

    \vspace{1mm}
    \Input{
        Polynomial \(P \in \mathbb{R}[z, m]\) with \(\deg_m(P) = s\) such that \(P(z, m_0(z)) = 0\) (see \Cref{alg:poly}); \\
        Time grid \(0 = t_0 < t_1 < \cdots < t_T\); \\
        Initial support \(I(0) = \bigcup_{j=1}^{k_0} [a_j(0), b_j(0)]\) (optional, for initialization); \\
        Oracle for the physical branch \(m_0(z)\) on \(\mathbb{C} \setminus I\) (see \Cref{alg:stieltjes}); \hfill {\textit{\textcolor{gray}{\footnotesize{// Only needed if \(I(0)\) is provided.}}}} \\
        Tolerance \(\epsilon > 0\).
    }

    \vspace{2mm}
    \Output{
        Edge trajectories \(\{a_j(t_{\ell}), b_j(t_{\ell})\}\) over the requested time grid; \\
        Active bulk count \(k(t_{\ell})\) (number of connected components).
    }

    \vspace{2mm}
    \tcp{Define the edge system from \Cref{prop:fd-branch-point}}
    Set \(\tau \leftarrow e^t\) and define the functions
    \begin{align*}
        G_1(\zeta, y; t) &\coloneqq P(\zeta, y) \\
        G_2(\zeta, y; t) &\coloneqq y^2 \partial_y P(\zeta, y) - (\tau - 1) \partial_{\zeta} P(\zeta, y).
    \end{align*}

    \tcp{Initialization at \(t_0 = 0\)}
    \If{initial edges \(I(0)\) are provided}{
        Initialize guesses \(\{(\zeta_0^{(j)}, y_0^{(j)})\}_{j=1}^{2 k_0}\) as follows: \tcp*{Any reasonable edge-local seeding works}
        \(\quad\triangleright\) Sample \(2 k_0\) points \(\zeta\), one near each endpoint \;
        \(\quad\triangleright\) Choose Herglotz-consistent root \(y_0^{(j)} \leftarrow m_0(\zeta_0^{(j)})\). \;
    }
    \Else{
        Obtain initial guesses \(\{(\zeta_0^{(j)}, y_0^{(j)})\}_{j=1}^{2 k_0}\) by solving \((G_1, G_2) = (0, 0)\) at \(t_0\) using multiple seeds.
    }

    \vspace{2mm}
    \tcp{Track edge preimages \((\zeta, y)\) by warm-started Newton over time}
    \(k(t_0) \leftarrow k_0\) \;
    \For{\(\ell = 1, \dots, T\)}{
        \(\tau \leftarrow e^{t_{\ell}}\) \;
        \(k(t_{\ell}) \leftarrow k(t_{\ell-1})\) \;
        \ForEach{\(j = 1, \cdots, 2 k(t_{\ell-1})\)}{
            Use previous solution \((\zeta_{\ell-1}^{(j)}, y_{\ell-1}^{(j)})\) as initial guess. \;
            Apply damped Newton to solve \((G_1(\zeta, y; t_{\ell}), G_2(\zeta, y; t_{\ell})) = (0, 0)\) for \((\zeta_{\ell}^{(j)}, y_{\ell}^{(j)})\). \;
            \(z_{\ell}^{(j)} \leftarrow \zeta_{\ell}^{(j)} - (\tau - 1) / y_{\ell}^{(j)}\) \tcp*{Pushforward to the \(z\)-plane}
            \leIf{\(\vert \Im z_{\ell}^{(j)} \vert < \epsilon\)}{
                Store \(x_{\ell}^{(j)} \leftarrow \Re z_{\ell}^{(j)}\)
            }{
                mark failure.
            }
        }

        \vspace{2mm}
        \tcp{Detect newly opened gaps within the active bulks}
        \ForEach{\(I_j(t_{\ell}) = [a_j(t_{\ell}), b_j(t_{\ell})]\)}{
            Search additional real solutions \((\zeta_-, y_-)\), \((\zeta_+, y_+)\) of \((G_1, G_2) = (0, 0)\) whose pushforwards
            \(x_\pm = \zeta_\pm - (\tau - 1) / y_\pm\) satisfy
            \(x_\pm \in I_j(t_{\ell})\) and \(x_- < x_+\) \;

            \vspace{1mm}
            \If{such \((x_-, x_+)\) exists and is a physical edge pair}{
                \tcp{Replace the parent bulk by the two split child bulks}
                \(\{I_r\}_{r=1}^{k(t_{\ell}) + 1}
                \leftarrow
                \{I_1, \dots, I_{j-1}, [a_j, x_-], [x_+, b_j], I_{j+1}, \dots, I_{k(t_{\ell})}\}\) \;
                \(k(t_{\ell}) \leftarrow k(t_{\ell}) + 1\) \tcp*{Update number of bulks}
                Store new preimages \((\zeta_-, y_-)\) and \((\zeta_+, y_+)\) \;
            }
        }
    }
    
    \vspace{2mm}
    \Return{Spectral edges \(\left\{ \{a_{j}(t_{\ell}), b_{j}(t_{\ell})\}_{j=1}^{k(t_{\ell})} \right\}_{\ell=0}^{T}\) and active bulk count \(\{k(t_{\ell})\}_{\ell=0}^{T}\)};
\end{algorithm}


\subsection{Spectral Cusps and The Mechanism of Splitting Bulks}\label{sec:merge}

Under free decompression, a bulk may split into two when a gap opens in its interior. Viewed in reverse, this is the merger of two neighboring bulks under free compression. We describe this transition through the geometry of the spectral curve, focusing on the branch points and their \emph{sheet connectivity} (monodromy).

Let \(I_j(t) \coloneqq [a_j(t), b_j(t)]\) denote the \(j\)-th bulk before the split. Suppose that at time \(t = t_{\ast} > 0\), a gap begins to open at a point \(x_{\ast}\) in its interior. We consider a regular transition, whose geometry is described below. Let \(x_{-}(t)\) and \(x_{+}(t)\) denote the left and right endpoints of the new gap. At the splitting time,
\begin{equation*}
    x_{-}(t_{\ast}) = x_{+}(t_{\ast}) = x_{\ast}.
\end{equation*}
For \(t > t_{\ast}\) sufficiently close to \(t_{\ast}\), we have \(x_{-}(t) < x_{+}(t)\), and the density vanishes in the gap \((x_{-}(t), x_{+}(t))\). We then relabel the child bulks, increasing the indices of the subsequent bulks by one and simultaneously setting
\begin{equation*}
    (a_j, b_j, a_{j+1}, b_{j+1}) \leftarrow (a_j, x_{-}, x_{+}, b_j).
\end{equation*}
The two resulting bulks are \(I_j(t) \coloneqq [a_j(t), b_j(t)]\) and \(I_{j+1}(t) \coloneqq [a_{j+1}(t), b_{j+1}(t)]\), with \(I_j(t) \cap I_{j+1}(t) = \emptyset\) just after the split. The new inner \emph{spectral edges}, now labeled \(b_j(t)\) and \(a_{j+1}(t)\), continue to evolve according to the branch-point equations in \Cref{prop:fd-branch-point}.

This splitting is a phase transition in sheet gluing relative to the physical branch. The corresponding branch points of the deformed spectral projection, described by \Cref{prop:fd-branch-point}, form a nonreal conjugate pair just before \(t_{\ast}\). They locally glue non-physical branches in their respective half-planes, where the physical Stieltjes transform remains holomorphic. At \(t_{\ast}\), they meet on the real axis. Just after \(t_{\ast}\), they connect the physical sheet to non-physical sheets and become the two inner spectral edges. In this sense, the branch points are present on both sides of the transition, but they bound a gap in the physical density only after the split.

At the splitting time \(t_{\ast}\), the density develops a \emph{cusp} (a cubic-root singularity at a sharp local minimum), associated with the Pearcey universality class in complex Wigner-type models \citep{ERDOS-2020}. Geometrically, the colliding branch points correspond to a higher-order critical point of the projection of the spectral curve. The following proposition gives algebraic conditions for this cusp point on \(\mathcal{C}\).

\begin{proposition}[Cusp point at Splitting Bulks] \label{prop:cusp}
    Suppose a bulk splits at a regular spectral cusp at time \(t_{\ast} > 0\), with \(x_{\ast} = b_j(t_{\ast}) = a_{j+1}(t_{\ast}) \in \mathbb{R}\). Let \(m_{\ast} \coloneqq m(t_{\ast}, x_{\ast})\) and define the pullback \((\zeta_{\ast}, y_{\ast}) \coloneqq \phi_{t_{\ast}}^{-1}(x_{\ast}, m_{\ast})\). Assume that \(y_{\ast} \neq 0\) and \(\partial_y P(\zeta_{\ast}, y_{\ast}) \neq 0\). Then, \((\zeta_{\ast}, y_{\ast}) \in \mathcal{C}\) is a higher-order critical point of the projection \(z(\zeta, y) = \zeta - (\tau_{\ast} - 1) y^{-1}\) restricted to \(\mathcal{C}\), and hence satisfies
    \begin{subequations} \label{eq:fd-cusp}
        \begin{align}
            &P(\zeta_{\ast}, y_{\ast}) = 0 \\
            &y_{\ast}^2\ \partial_{y} P - (\tau_{\ast} - 1)\ \partial_{\zeta} P = 0, \\
            &y_{\ast} \left((\partial_{\zeta \zeta} P) (\partial_y P)^2  - 2 (\partial_{\zeta y} P)  (\partial_{\zeta} P) (\partial_y P) + (\partial_{yy} P) (\partial_{\zeta} P)^2\right) + 2 (\partial_{\zeta} P)^2 (\partial_y P) = 0,
        \end{align}
    \end{subequations}
    where all derivatives are evaluated at \((\zeta_{\ast}, y_{\ast})\), and \(\tau_{\ast} \coloneqq e^{t_{\ast}}\). Moreover, the cusp location is the pushforward \(x_{\ast} = \zeta_{\ast} - (\tau_{\ast} - 1) y_{\ast}^{-1}\).
\end{proposition}

\begin{proof}
    Consider the restriction of the map
    \begin{equation*}
        z(\zeta, y) = \pi \circ \phi_t(\zeta, y) = \zeta - (\tau - 1) y^{-1}
    \end{equation*}
    to the spectral curve \(\mathcal{C} = \{ (\zeta, y) \in \mathbb{C}^2 \mid  P(\zeta, y) = 0\}\). Near \((\zeta_{\ast}, y_{\ast}) \in \mathcal{C}\), let \(y = y(\zeta)\) be a local parameterization of \(\mathcal{C}\). A regular \emph{cusp} at time \(t_{\ast}\) corresponds to a critical point of multiplicity two of the projection \(\zeta \mapsto z(\zeta, y(\zeta)) = z(\zeta)\) along \(\mathcal{C}\), \ie
    \begin{equation*}
        \frac{\mathrm{d} z(\zeta_{\ast})}{\mathrm{d} \zeta} = 0,
        \qquad
        \frac{\mathrm{d}^2 z(\zeta_{\ast})}{\mathrm{d} \zeta^2} = 0,
        \qquad
        \frac{\mathrm{d}^3 z(\zeta_{\ast})}{\mathrm{d} \zeta^3} \neq 0,
    \end{equation*}
    together with \(P(\zeta_{\ast}, y_{\ast}) = 0\).

    The identity \(P(\zeta, y(\zeta)) = 0\) along \(\mathcal{C}\) implies \(\mathrm{d} P / \mathrm{d} \zeta = 0\) and \(\mathrm{d}^2 P / \mathrm{d} \zeta^2 = 0\) at \((\zeta_{\ast}, y_{\ast})\). Implicit differentiation of \(P(\zeta, y(\zeta)) = 0\) along \(\mathcal{C}\) gives
    \begin{align*}
        \frac{\mathrm{d} P}{\mathrm{d} \zeta} &= \frac{\partial P}{\partial \zeta} + y' \frac{\partial P}{\partial y}, \\
        \frac{\mathrm{d}^2 P}{\mathrm{d} \zeta^2} &= \frac{\partial^2 P}{\partial \zeta^2} + 2 y' \frac{\partial^2 P}{\partial \zeta \partial y} + (y')^2 \frac{\partial^2 P}{\partial y^2} + y'' \frac{\partial P}{\partial y},
    \end{align*}
    where \(y' \coloneqq \mathrm{d} y / \mathrm{d} \zeta\) and \(y'' \coloneqq \mathrm{d}^2 y / \mathrm{d} \zeta^2\). Solving \(\mathrm{d} P / \mathrm{d} \zeta = 0\) and \(\mathrm{d}^2 P / \mathrm{d} \zeta^2 = 0\) for \(y'\) and \(y''\) yields
    \begin{subequations} \label{eq:yp-ypp}
    \begin{align}
        y' &= -\frac{\partial_{\zeta} P}{\partial_y P}, \\
        y'' &= - \frac{\partial_{\zeta \zeta} P + 2 y' \partial_{\zeta y} P + (y')^2 \partial_{y y} P}{\partial_{y} P},
    \end{align}
    \end{subequations}
    whenever \(\partial_y P(\zeta, y) \neq 0\).

    Next, differentiating \(z(\zeta) = \zeta - (\tau - 1) y(\zeta)^{-1}\) along \(\mathcal{C}\) yields
    \begin{subequations} \label{eq:zp-zpp}
    \begin{align}
        \frac{\mathrm{d} z}{\mathrm{d} \zeta} &= 1 + (\tau - 1) y^{-2} y', \\
        \frac{\mathrm{d}^2 z}{\mathrm{d} \zeta^2} &= (\tau - 1) \left(-2 y^{-3} (y')^2 + y^{-2} y'' \right).
    \end{align}
    \end{subequations}
    Imposing \(\mathrm{d} z / \mathrm{d} \zeta = 0\) and using \(y'\) from \eqref{eq:yp-ypp} gives the second equation in \eqref{eq:fd-cusp},
    \begin{equation*}
        y^2 \partial_{y} P - (\tau - 1) \partial_{\zeta} P = 0.
    \end{equation*}
    Next, imposing \(\mathrm{d}^2 z / \mathrm{d} \zeta^2 = 0\), substitute \(y'\) and \(y''\) from \eqref{eq:yp-ypp}, and clear denominators by multiplying by \(y^3 (\partial_y P)^3 / (\tau - 1)\). After simplification, this yields the third equation in \eqref{eq:fd-cusp},
    \begin{equation*}
        y \left( \partial_{\zeta \zeta} P (\partial_{y} P)^2 - 2 (\partial_{\zeta y} P) (\partial_{\zeta} P) (\partial_{y} P) + (\partial_{yy} P) (\partial_{\zeta}P)^2 \right) + 2 (\partial_{\zeta} P)^2 (\partial_y P) = 0,
    \end{equation*}
    where all partial derivatives are evaluated at \((\zeta_{\ast}, y_{\ast})\). Finally, the cusp location in the physical coordinate is the pushforward \(x_{\ast} = z(\zeta_{\ast}, y_{\ast}) = \zeta_{\ast} - (\tau_{\ast} - 1) y_{\ast}^{-1}\) at time \(t_{\ast} = \log(\tau_{\ast})\).
\end{proof}

Practically, \eqref{eq:fd-cusp} can be used to solve for candidate triples \((\zeta_{\ast}, y_{\ast}, \tau_{\ast})\) and hence recover both the splitting time \(t_{\ast} = \log(\tau_{\ast})\) and the split location \(x_{\ast}\), after checking the stated nondegeneracy conditions and that the cusp lies on the physical branch. This provides a principled and more accurate alternative compared to monitoring gap opening on a discrete time grid. Once \(t_{\ast}\) is identified for an isolated split, one can treat the immediate post-split configuration as having one more bulk, and interpret subsequent real branch points through the updated sheet gluing relative to the physical branch.


\subsection{Dynamics of Gap and Support Intervals} \label{sec:gap}

Beyond tracking the density itself, it is often useful to characterize whether a bulk interval widens or narrows, and whether a gap between two bulks shrinks or expands. Define the bulk widths \(s_j\) and gap widths \(g_j\) by
\begin{equation*}
    s_j(t) \coloneqq \vert I_j(t) \vert = b_j(t) - a_j(t),
    \qquad
    g_j(t) \coloneqq a_{j+1}(t) - b_{j}(t).
\end{equation*}

The dynamics of these intervals can be obtained from the relative velocity of edges. We derive an expression for the velocity of an edge, \(x_{\ast}'(t)\), in terms of the geometry of the spectral curve, and then specialize at \(t = 0\), where it admits a particularly simple interpretation in terms of the Hilbert transform.

\begin{proposition}[Edge Velocity under Free Decompression] \label{prop:vel-edge}
    Let \(x_{\ast}(t) \in \mathbb{R}\) be a regular spectral edge of the density, and let \((\zeta_{\ast}(t), y_{\ast}(t))\) denote a corresponding preimage on the spectral curve \(\mathcal{C} = \{(\zeta, y) \in \mathbb{C}^2 \mid P(\zeta, y) = 0\}\) under the free decompression map \(z = \zeta - (\tau - 1) y^{-1}\) with \(\tau = e^t\). Suppose \((\zeta_{\ast}, y_{\ast})\) is a simple critical point of the projection (so it is not a cusp point), hence it satisfies
    \begin{subequations} \label{eq:def-rel}
    \begin{align}
        &P(\zeta_{\ast}, y_{\ast}) = 0, \\
        &F(\zeta_{\ast}, y_{\ast}, \tau) = y_{\ast}^2 \partial_y P(\zeta_{\ast}, y_{\ast}) - (\tau - 1) \partial_{\zeta} P(\zeta_{\ast}, y_{\ast}) = 0,
    \end{align}
    \end{subequations}
    and the edge location is \(x_{\ast}(t) = \zeta_{\ast}(t) - (\tau - 1) y_{\ast}(t)^{-1}\). Then, \(x_{\ast}(t)\) is differentiable and its velocity is
    \begin{equation}
        x'_{\ast}(t) = \zeta'_{\ast}(t) - \tau\, y_{\ast}(t)^{-1} + (\tau - 1)\, y'_{\ast}(t)\, y_{\ast}(t)^{-2},
        \label{eq:vel-x}
    \end{equation}
    where \((\zeta'_{\ast}(t), y'_{\ast}(t))\) is obtained by solving the linear system
    \begin{equation}
        \begin{bmatrix}
            \partial_{\zeta} P & \partial_y P \\
            \partial_{\zeta} F & \partial_y F
        \end{bmatrix}_{(\zeta_{\ast}, y_{\ast}, \tau)}
        \begin{bmatrix}
            \zeta'_{\ast} \\
            y'_{\ast}
        \end{bmatrix}
        = - \tau
        \begin{bmatrix}
            0 \\
            \partial_{\tau} F
        \end{bmatrix}_{(\zeta_{\ast}, y_{\ast}, \tau)},
        \label{eq:lin-sys-vel}
    \end{equation}
    with
    \begin{align*}
        \partial_{\tau} F &= - \partial_{\zeta} P(\zeta, y), \\
        \partial_{\zeta} F &= y^2 \partial_{\zeta y} P - (\tau - 1) \partial_{\zeta \zeta} P, \\
        \partial_{y} F &= 2y \partial_y P + y^2 \partial_{y y} P - (\tau - 1) \partial_{\zeta y} P.
    \end{align*}
\end{proposition}

\begin{proof}
    We differentiate the defining relations \eqref{eq:def-rel} along the solution curve \(t \mapsto (\zeta_{\ast}(t), y_{\ast}(t))\). From \(P(\zeta_{\ast}, y_{\ast}) = 0\), we obtain
    \begin{equation*}
        \partial_{\zeta} P(\zeta_{\ast}, y_{\ast})\, \zeta'_{\ast} + \partial_y P(\zeta_{\ast}, y_{\ast})\, y'_{\ast} = 0.
    \end{equation*}
    From \(F(\zeta_{\ast}, y_{\ast}, \tau) = 0\) and \(\tau'(t) = \tau(t)\), we obtain
    \begin{equation*}
        \partial_{\zeta} F(\zeta_{\ast}, y_{\ast}, \tau)\, \zeta'_{\ast}
        + \partial_y F(\zeta_{\ast}, y_{\ast}, \tau)\, y'_{\ast}
        + \partial_{\tau} F(\zeta_{\ast}, y_{\ast}, \tau)\, \tau'(t) = 0,
    \end{equation*}
    which yields the linear system \eqref{eq:lin-sys-vel}. Finally, differentiating \(x_{\ast}(t) = \zeta_{\ast}(t) - (\tau(t) - 1) y_{\ast}(t)^{-1}\) gives \eqref{eq:vel-x}.
\end{proof}

\begin{corollary}[Initial edge velocity] \label{cor:init-vel-edge}
    Let \(x_{\ast}(0)\) be a regular spectral edge at \(t = 0\), and let
    \begin{equation*}
        H(t, x) \coloneqq \mathcal{H}\{ \nu(t) \}(x) = \mathrm{p.v.} \, \frac{1}{\pi} \int_{\mathbb{R}} \frac{\nu(t, \mathrm{d} y)}{x - y},
    \end{equation*}
    be the Hilbert transform of the full spectral measure \(\nu(t)\) (including any atoms). Then,
    \begin{equation}
        x'_{\ast}(0) = \frac{1}{\pi H(0, x_{\ast})}.
    \end{equation}
\end{corollary}

\begin{proof}
    At \(t = 0\), we have \(\tau = 1\) and \(x_{\ast}(0) = \zeta_{\ast}(0)\). Also \(y_{\ast}(0) = m(0, x_{\ast})\). Using the Plemelj formula, \(m(0, x_{\ast}) = -\pi H(0, x_{\ast}) + i \pi \rho(0, x_{\ast})\), but on the spectral edge \(\rho(0, x_{\ast}) = 0\), so \(y_{\ast}(0) = -\pi H(0, x_{\ast})\). Also, from \Cref{prop:vel-edge}, we have
    \begin{equation*}
        x'_{\ast}(0) = \zeta'_{\ast}(0) - \tau(0)\, y_{\ast}(0)^{-1}.
    \end{equation*}
    Moreover, at \(t = 0\) the criticality condition in \eqref{eq:def-rel} reduces to \(\partial_y P(\zeta_{\ast}(0), y_{\ast}(0)) = 0\). If \(x_{\ast}(0)\) is a regular edge (not a cusp), then \(\partial_{\zeta} P(\zeta_{\ast}(0), y_{\ast}(0)) \neq 0\), and differentiating \(P(\zeta_{\ast}(t), y_{\ast}(t)) = 0\) at \(t = 0\) implies \(\zeta'_{\ast}(0) = 0\). Hence \(x'_{\ast}(0) = -y_{\ast}(0)^{-1} = 1/ (\pi H(0, x_{\ast}))\).
\end{proof}

The dynamics of bulk widths \(s_j(t)\) and gap widths \(g_j(t)\) follows immediately. At \(t = 0\), \Cref{cor:init-vel-edge} yields the explicit identities
\begin{subequations}
\begin{align}
    s'_j(0) &= \frac{1}{\pi H(0, b_j(0))} - \frac{1}{\pi H(0, a_j(0))}, \\
    g'_j(0) &= \frac{1}{\pi H(0, a_{j+1}(0))} - \frac{1}{\pi H(0, b_j(0))}.
\end{align}
\end{subequations}

The sign of the Hilbert transform at inner edges depends on the global arrangement of mass across bulks and atoms, and the identities above determine the initial change in each gap width. In contrast, if the full measure \(\nu(0)\), including any atoms, is supported on \([a_1(0), b_k(0)]\), the signs at the two outermost bulk edges are fixed: the integral defining \(H\) gives \(H(0, a_1(0)) < 0\) and \(H(0, b_k(0)) > 0\), implying \(a_1'(0) < 0\) and \(b_k'(0) > 0\). Thus, under this assumption, the span \(b_k(t) - a_1(t)\) is increasing at \(t = 0\).


\subsection{Dynamics of Atoms} \label{sec:atoms}

An atom of \(\nu\) corresponds to a simple pole of the Stieltjes transform \(m(z)\) at \(z = x_{\circ}\), with weight
\begin{equation*}
    w = - \Res_{z = x_{\circ}} m(z) = \lim_{z \to x_{\circ}} (x_{\circ} - z) m(z).
\end{equation*}
When \(m\) is algebraic and satisfies \(P(z, m) = 0\), where \(P(z, m) \coloneqq \sum_{j=0}^{s} a_j(z) m^j\), a necessary condition for a pole at a finite real point \(x_{\circ}\) is that \(x_{\circ}\) is a root of the leading coefficient, \ie \(a_s(x_{\circ}) = 0\). Conversely, a real root of \(a_s\) corresponds to an atom on the physical branch provided that \(m\) has a simple pole at \(x_{\circ}\) (\ie the singularity is not canceled and is not of higher order). In that case, if \(x_{\circ}\) is a simple root of \(a_s\), coefficient matching yields
\begin{equation*}
    w = \frac{a_{s-1}(x_{\circ})}{a_s'(x_{\circ})}.
\end{equation*}

\begin{proposition} \label{prop:atom}
    Let \(x_{\circ} \in \mathbb{R}\) be an atom of \(m(z)\) with initial mass \(w_{\circ} \in (0, 1]\). Under the free decompression flow, the atom location remains fixed \(x_{\circ}(t) = x_{\circ}\), while its weight evolves as
    \begin{equation}
        w(t) = 1 - (1 - w_{\circ}) e^{-t}.
        \label{eq:atom-mass}
    \end{equation}
\end{proposition}

\begin{proof}
    Let \(x_{\circ}(t)\) denote the (a priori time-dependent) location of the atom with weight \(w(t)\). The Laurent expansion of \(m(t, z)\) near \(z = x_{\circ}(t)\) is
    \begin{equation*}
        m(t, z) = - \frac{w(t)}{z - x_{\circ}(t)} + O(1).
    \end{equation*}
    Differentiating gives
    \begin{equation*}
        \partial_t m(t, z) = - \frac{w'(t)}{z - x_{\circ}(t)} - \frac{w(t)\, x_{\circ}'(t)}{(z - x_{\circ}(t))^2} + O(1),
    \end{equation*}
    and
    \begin{equation*}
        \partial_z m(t, z) = \frac{w(t)}{(z - x_{\circ}(t))^2} + O(1).
    \end{equation*}
    Moreover, the leading term of \(m^{-1}\) near \(z = x_{\circ}(t)\) is
    \begin{equation*}
        m(t, z)^{-1} = - \frac{z - x_{\circ}(t)}{w(t)} + O\!\left((z - x_{\circ}(t))^2\right).
    \end{equation*}
    Substituting these expansions into the PDE \(\partial_t m = -m + m^{-1} \partial_z m\) yields
    \begin{equation}
        - \frac{w'(t)}{z - x_{\circ}(t)} - \frac{w(t)\, x_{\circ}'(t)}{(z - x_{\circ}(t))^2} + O(1)
        = \frac{w(t)}{z - x_{\circ}(t)} - \frac{1}{z - x_{\circ}(t)} + O(1).
        \label{eq:atom-perturb}
    \end{equation}
    The right-hand side has no second-order pole \((z - x_{\circ}(t))^{-2}\), hence the coefficient of \((z - x_{\circ}(t))^{-2}\) on the left-hand side must vanish, implying \(w(t)\,x_{\circ}'(t) = 0\). If \(w(t) > 0\), then \(x_{\circ}'(t) = 0\), and therefore \(x_{\circ}(t) = x_{\circ}(0)\) is constant.

    Finally, equating the coefficients of the first-order poles in \eqref{eq:atom-perturb} gives the ODE \(w'(t) = 1 - w(t)\), which is solved by \eqref{eq:atom-mass} with the initial condition \(w(0) = w_{\circ}\).
\end{proof}

\begin{remark}
    The conclusion of \Cref{prop:atom} is local: it does not depend on whether \(x_{\circ}\) lies inside, outside, or at the boundary of the absolutely-continuous support of \(\nu(0)\). Moreover, if \(\nu(0)\) has finitely many atoms \(\{(x_{\circ}^{(i)}, w_{\circ}^{(i)})\}_{i=1}^{n_a}\) at distinct locations \(x_{\circ}^{(i)}\), then each location \(x_{\circ}^{(i)}\) remains fixed and each weight evolves independently according to \eqref{eq:atom-mass}.
\end{remark}


    

\begin{remark}
    In backward flow with \(t < 0\) (free \emph{compression}) and \(0 < w_{\circ} < 1\), there is a finite time \(t_{\circ} \coloneqq \log(1 - w_{\circ}) < 0\) at which the atom weight vanishes, \(w(t_{\circ}) = 0\). For \(t < t_{\circ}\), the pole (atom) ansatz is no longer valid, reflecting that the atom has disappeared.

    In particular, the atom dynamics are forward deterministic while the atom is present (\(w_{\circ} > 0\)). However, disappearance of the atom does not imply loss of information from the full compressed measure. For a known compression ratio \(c \in (0, 1)\), the original and compressed R-transforms satisfy \(R_c(w) = R_0(c w)\), hence \(R_0(w) = R_c(w / c)\). Consequently, decompression from the exact compressed measure remains unique whenever it exists.

    A concrete example is provided by rank-deficient matrices. Suppose \(\tens{A}\) is an \(n \times n\) symmetric matrix with nullity \(0 < q < n\), so its empirical spectral measure has an atom of weight \(w_{\circ} = q / n\) at the origin. Then, after a Haar-distributed orthogonal change of basis, a principal submatrix \(\tens{A}_{\circ}\) of size \(n_{\circ}\) is almost surely full rank when \(n_{\circ} \le n - q\), and hence its empirical spectral measure has no atom at zero.
\end{remark}


\subsection{Evolution of Moments} \label{sec:evo-moments}

We now turn to computing the moments of the evolving measures \(\nu_\tau\) under free decompression starting from $\nu \equiv \nu_1$. For \(n \geq 0\), define the time-dependent raw moment
\begin{equation*}
	\mu_n^{(\tau)} \coloneqq \int x^n  \mathrm{d} \nu_\tau (x).
\end{equation*}
These moments can be evolved without reconstructing \(\nu_\tau\): the free decompression PDE from \citet[Proposition 1]{AMELI-2025b} induces a closed triangular system of ODEs for \(\mu_n(t) = \mu_n^{\exp(t)}\), allowing each moment to be computed directly from lower-order moments. Throughout this section, we let $m(t, z)$ denote the Stieltjes transform of $\nu_{e^t}$ and assume that \(m(t, z) \neq 0\) along the free decompression flow.

\begin{lemma}[ODEs for Moments] \label{lem:moment}
    The evolution of the moment \(\mu_n(t)\) is governed by the system of the triangularly coupled ODEs,
    \begin{alignat}{3} \label{eq:ODE-moment}
        &\left(\frac{1}{2} \frac{\mathrm{d}}{\mathrm{d} t} + 1 \right) \sum_{i=0}^{n} \mu_i(t) \mu_{n-i}(t) = (n+1) \mu_{n}(t), \qquad &&n \in \mathbb{N},
    \end{alignat}
and, in particular, $\mu_0(t) = 1$ and $\mu_1(t) = \mu_1(0)$ for all $t > 0$. 
\end{lemma}

\begin{proof}
    Since by the hypothesis \(m(t, z)  \neq 0\), we can rearrange the PDE from \citet[Proposition 1]{AMELI-2025b} to
    \begin{equation}
        \left( \frac{1}{2} \frac{\partial}{\partial t} + 1 \right) m(t, z)^2 = \frac{\partial}{\partial z} m(t, z). \label{eq:PDE-prod-m}
    \end{equation}
    At  \(\vert z \vert > \max(\lambda)\) and under mild regularity conditions, we can expand a convergent Laurent series of \(m(t, z)\) by
    \begin{equation*}
        m(t, z) = \int \frac{\mathrm{d} \sigma(\lambda)}{\lambda - z} = -\int \left( \sum_{n=0}^{\infty} \frac{\lambda^n}{z^{n+1}} \right) \mathrm{d} \sigma(\lambda) = -\sum_{n=0}^{\infty} \frac{\mu_n(t)}{z^{n+1}}.
    \end{equation*}
    The Laurent series of \(m(t,z)^2\) becomes
    \begin{equation*}
        m(t, z)^2 = \sum_{n=0}^{\infty} \left( \sum_{i=0}^{n} \mu_i(t) \mu_{n-i}(t) \right) \frac{1}{z^{n+2}}.
    \end{equation*}
    Also,
    \begin{equation}
    \frac{\partial}{\partial z} m(t, z) = \sum_{n=0}^{\infty} (n+1)\frac{\mu_{n}(t)}{z^{n+2}}.
    \end{equation}
    Substituting the Laurent series into \eqref{eq:PDE-prod-m} and rearranging for the powers of \(z^{-(n+1)}\), we obtain \eqref{eq:ODE-moment}. For the first few terms, we find that $\mu_0' = 1- \mu_0$, and since $\mu_0(0) = 1$, $\mu_0(t) = 1$ for all $t > 0$. Similarly, we find $\mu_1'(t) = 0$ and so $\mu_1(t)$ is also constant.
\end{proof}

The triangular structure of \cref{lem:moment} suggests solving the moment equations recursively. In particular, the lower-order moments determine the forcing term in the ODE for \(\mu_n(t)\), and this forcing is a finite linear combination of powers of \(e^t\). We therefore seek \(\mu_n(t)\) in the form
\begin{equation*}
	\mu_n(t) = \sum_{k = 0}^{n - 1} \kappa_{n, k} e^{k t},\qquad \mbox{or equivalently}\qquad \mu_n^{(\tau)} = \sum_{k=0}^{n-1} \kappa_{n, k} \tau^k,
\end{equation*}
where the coefficients \(\kappa_{n, k}\) are determined recursively from the initial moments. The following proposition gives the resulting recurrence.

\begin{proposition}[Moment Recurrence] \label{prop:moment-decompression}
    For any $t \geq 0$ and $n \geq 1$, the $n$-th moment $\mu_n(t)$ satisfies $\mu_n(t) = \sum_{k=0}^{n-1} \kappa_{n, k} e^{kt}$ where $\kappa_{1,0} = \mu_1(0)$,
    \begin{equation}
        \kappa_{n,k} = \begin{cases}\displaystyle\frac{\frac12 k + 1}{n - k - 1} \sum_{i=1}^{n-1} \sum_{j=0}^k \kappa_{i,j} \kappa_{n-i, k-j},&\qquad 0 \leq k \leq n-2,\\
        \mu_n(0) - \displaystyle\sum_{j=0}^{n-2} \kappa_{n,j} & \qquad k = n - 1,
        \end{cases}
        \label{eq:MomentsKappa}
    \end{equation}
    and where $\kappa_{n,k} = 0$ if $k \geq n$.
\end{proposition}

\begin{proof}
    Starting from \eqref{eq:ODE-moment},
    \[
        \left(\frac{1}{2}\frac{\mathrm{d}}{\mathrm{d}t}+1\right)\left[\sum_{i=1}^{n-1}\mu_{i}(t)\mu_{n-i}(t)+2\mu_{n}(t)\right]=(n+1)\mu_{n}(t).
    \]
    Substituting $\mu_n(t) = \sum_{k=0}^{n-1} \kappa_{n,k} e^{kt}$ gives
    \begin{align*}
        (n+1)\mu_n(t) &= \left(\frac{1}{2}\frac{\mathrm{d}}{\mathrm{d}t}+1\right)\left[\sum_{i=1}^{n-1}\mu_{i}(t)\mu_{n-i}(t)+2\sum_{k=0}^{n-1}\kappa_{nk}e^{kt}\right] \\
        &= \left(\frac{1}{2}\frac{\mathrm{d}}{\mathrm{d}t}+1\right)\sum_{i=1}^{n-1}\mu_{i}(t)\mu_{n-i}(t)+\sum_{k=0}^{n-1}\kappa_{nk}\left(k+2\right)e^{kt},
    \end{align*}
    and hence
    \begin{equation}
        \left(\frac{1}{2}\frac{\mathrm{d}}{\mathrm{d}t}+1\right)\sum_{i=1}^{n-1}\mu_{i}(t)\mu_{n-i}(t)=\sum_{k=0}^{n-2}\left(n-k-1\right)\kappa_{nk}e^{kt}\label{eq:MomentExprReduced}
    \end{equation}
    Rearranging terms gives
    \begin{align}
        \sum_{i=1}^{n-1}\mu_{i}(t)\mu_{n-i}(t) &= \sum_{i=1}^{n-1}\left(\sum_{j=0}^{i-1}\kappa_{ij}e^{jt}\right)\left(\sum_{j=0}^{n-i-1}\kappa_{n-i,j}e^{jt}\right) \nonumber \\
        &=\sum_{i=1}^{n-1}\sum_{k=0}^{n-2}\sum_{j=0}^{k}\kappa_{ij}\kappa_{n-i,k-j}e^{kt}.\label{eq:KappaConv}
    \end{align}
    Substituting \eqref{eq:KappaConv} into \eqref{eq:MomentExprReduced} gives
    \[
        \sum_{k=0}^{n-2}\left(\frac{1}{2}k+1\right)\sum_{i=1}^{n-1}\sum_{j=0}^{k}\kappa_{ij}\kappa_{n-i,k-j}e^{kt}=\sum_{k=0}^{n-2}\left(n-k-1\right)\kappa_{nk}e^{kt},
    \]
    and \eqref{eq:MomentsKappa} follows by equating terms for $0 \leq k \leq n-2$. Only the expression for $\kappa_{n,n-1}$ remains, which follows from
    $\sum_{j=0}^{n-1} \kappa_{n,j} = \mu_n(0)$.
\end{proof}


\section{Finite-Size Correction to the Spectral Curve}
\label{sec:finite-correction}

It is critical to note that the discussion so far has focused on cases where the spectrum of a finite size matrix is approximated by that of a matrix with ``equivalent'' spectral distribution, but is of infinite size, and therefore has a continuous spectrum. This is typically fine for most applications, but it can fail in cases where a positive-definite matrix is considered with a nonzero minimal eigenvalue, yet the limiting spectrum possesses a left edge at zero. In these cases, especially to estimate the condition number of matrix, accounting for finite size effects becomes essential. This is intractable in the general cases considered here, so we shall instead develop an effective heuristic based on a special class of random matrix distributions. We note that the discussion below is \emph{strictly informal}, as a robust proof is well beyond the scope of this paper. 

Let $\tens{A}$ be an $n \times n$ real symmetric matrix and assume that the probability measure over the elements of $\tens{A}$ satisfies
\[
    p(\tens{A}) \propto \exp\left(-\frac{n}{2}\mbox{tr} V(\tens{A})\right) \prod_{i,j=1}^n \mathrm{d} \tens{A}_{ij},
\]
where $V$ is an arbitrary polynomial. Our objective is to estimate $\bar{m}_n(z) = \mathbb{E}m_n(z)$ where $m_n(z) = \frac{1}{n}\mbox{tr}G(z)$ and $\tens{G}(z) = (\tens{A} - z\tens{I})^{-1}$ is the resolvent for $z \in \mathbb{C} \setminus \mathbb{R}$. Consider Stein's Lemma: for any matrix-valued smooth function $f = (f_{ij})$ that grow subexponentially,
\[
    \int \sum_{i \leq j} \frac{\partial}{\partial \tens{A}_{ij}}(f_{ij}(\tens{A}) e^{-\frac{n}{2}\mbox{tr}V(\tens{A})}) \mathrm{d} \tens{A} = 0,
\]
and so
\[
    \mathbb{E}\left[\sum_{i\leq j} \frac{\partial f_{ij}(\tens{A})}{\partial \tens{A}_{ij}}\right] = \frac{n}{2}\mathbb{E}\left[\sum_{i \leq j} f_{ij}(\tens{A})\frac{\partial}{\partial \tens{A}_{ij}} \mbox{tr}V(\tens{A})\right] = \frac{n}{2}\mathbb{E}[\mbox{tr}(f(\tens{A})V'(\tens{A}))].
\]
Here, we want to take $f(\tens{A}) = \tens{G}(z)$. Note that since $V$ is a polynomial, $Q(x,z) = (V'(z)-V'(x))/(z - x)$ is a bivariate polynomial. Then
\[
    \tens{G}(z)V'(\tens{A}) = (\tens{A} - z\tens{I})^{-1} V'(\tens{A}) = V'(z)(\tens{A} - z\tens{I})^{-1} + Q(\tens{A},z) = V'(z) \tens{G}(z) + Q(\tens{A},z).
\]
Substituting this into the right-hand side of Stein's Lemma yields
\[
    \frac{n}{2}\mathbb{E}\mbox{tr}(\tens{G}(z)V'(\tens{A})) = \frac{n}{2}(V'(z)\mathbb{E}\mbox{tr}\tens{G}(z)+\mathbb{E}\mbox{tr} Q(\tens{A},z)).
\]
By letting $m_n(z) = \frac{1}{n}\mbox{tr}\tens{G}(z)$ and defining the polynomial expectation $Q_n(z) = \mathbb{E}[\frac{1}{n}\mbox{tr} Q(\tens{A},z)]$, the right-hand side simplifies:
\[
    \frac{n}{2}\mathbb{E}\mbox{tr}(\tens{G}(z)V'(\tens{A})) = \frac{n^2}{2}(V'(z)\bar{m}_n(z)+Q_n(z)).
\]
To evaluate the left-hand side, we apply the matrix chain rule $\mathrm{d}\tens{G} = -\tens{G}(\mathrm{d} \tens{A}) \tens{G}$ to find the derivative
\[
    \frac{\partial \tens{G}_{ij}}{\partial \tens{A}_{ij}} = \begin{cases} 
        -(\tens{G}_{ii}\tens{G}_{jj} + \tens{G}_{ij}^2) & \text{ for }i \neq j,\\
        -\tens{G}_{ii}^2 & \text{ for }i = j.
    \end{cases}
\]
Consequently,
\[
    \sum_{i \leq j} \frac{\partial \tens{G}_{ij}}{\partial \tens{A}_{ij}} = -\frac12 ((\mbox{tr}\tens{G}(z))^2 + \mbox{tr}(\tens{G}(z))^2).
\]
Since $\mbox{tr}(\tens{G}(z)^2) = \frac{\mathrm{d}}{\mathrm{d}z}\mbox{tr}\tens{G}(z) = n m_n'(z)$,
\[
    \sum_{i \leq j} \frac{\partial \tens{G}_{ij}}{\partial \tens{A}_{ij}} = -\frac{n^2}{2} \mathbb{E}[m_n(z)^2] - \frac{n}{2}\mathbb{E}[m_n'(z)].
\]
Therefore,
\[
    \mathbb{E}[m_n(z)^2] + V'(z)\bar{m}_n(z) + Q_n(z) + \frac{1}{n}\mathbb{E}[m_n'(z)] = 0.
\]
Here, we assert that $m_n(z) \approx m(z) + \mathcal{O}(1/n)$, and so $\mbox{Var} (m_n(z)) = \mathcal{O}(n^{-2})$. To prove this rigorously is quite challenging, and we refer the reader to \cite[Section 8]{GUIONNET-2009}. Consequently, keeping only first-order terms,
\begin{equation}
\label{eq:QuadraticModelRM}
    \bar{m}_n(z)^2 + V'(z) \bar{m}_n(z) + Q_n(z) + \frac{1}{n}\bar{m}_n'(z) \approx 0.
\end{equation}

\begin{remark}
\label{rem:EnsembleQuadratic}
    Taking $n \to \infty$ in \eqref{eq:QuadraticModelRM} reveals that this large class of random matrix ensembles always exhibits a polynomial relation in the limiting Stieltjes transform $m$ and its argument $z$ that is quadratic in $m$. 
\end{remark}

Motivated by \eqref{eq:QuadraticModelRM}, we now make the additional heuristic approximation of replacing the expected resolvent \(\bar{m}_n\) by the finite-size resolvent \(m_n\), while retaining the leading derivative correction. Observing that the first three terms comprise a bivariate polynomial in \(z\) and \(m\), we approximate this polynomial near a point of interest \(m_\ast\) by its second-order expansion:
\[
    P_2(z,m) = P(z, m_\ast) + \partial_m P(z, m_\ast) (m-m_\ast) + \tfrac12 \partial_m^2 P(z, m_\ast) (m-m_\ast)^2.
\]
This leads to our \textbf{correction ansatz}:
\begin{equation}
	\label{eq:FirstOrderAnsatz}
	\begingroup
	\setlength{\fboxsep}{5pt}
	\setlength{\fboxrule}{0.4pt}
	\fcolorbox{black!70}{gray!5}{%
		\(\displaystyle P_2(z, m_n(z)) + \frac{1}{n} m_n'(z) = 0\)%
	}
	\endgroup
\end{equation}

Now we consider how this affects the location of branch points.

\begin{proposition}
    \label{prop:FiniteEffectBranch}
    Let $(z_\ast,m_\ast) \in \mathbb{C}^2$ be a branch point of $m(t,z)$ with $m_\ast \coloneqq m(t,z_\ast)$. Define the pullback $(\zeta_\ast,y_\ast) = \phi_t^{-1}(z_\ast,m_\ast)$. Then, evaluated along the integral curve $y(\zeta)$ satisfying the initial differential equation \eqref{eq:FirstOrderAnsatz}, the point $(\zeta_\ast, y_\ast)$ satisfies
    \[
        y_\ast^2 - n(\tau - 1)P_2(\zeta_\ast, y_\ast) = 0,
    \]
    followed by the pushforward $z_\ast = \zeta_\ast - (\tau-1)y_\ast^{-1}$. 
\end{proposition}

\begin{proof}
    From the differential equation $P_2(\zeta,y(\zeta)) + \frac{1}{n}y'(\zeta) = 0$, we immediately have that
    \[
        \frac{\mathrm{d}y}{\mathrm{d}\zeta} = -n P_2(\zeta,y).
    \]
    Substituting this into the criticality condition $\frac{\mathrm{d}z}{\mathrm{d}\zeta} = 0$ gives
    \[
        \frac{\mathrm{d}z}{\mathrm{d}\zeta} = 1 + (\tau - 1) y^{-2}(-n P_2(\zeta,y)) = 0,
    \]
    from whence the result follows.
\end{proof}

A few remarks about \Cref{prop:FiniteEffectBranch}: assuming that $y_\ast$ remains finite as $n \to \infty$ and $P \approx P_2$, the condition $y_\ast^2 - n(\tau-1) P_2(\zeta_\ast,y_\ast) = 0$ implies that
\[
    \lim_{n\to\infty} P_2(\zeta_\ast, y_\ast) = 0.
\]
Furthermore, since $P_2(\zeta,y) = -\frac{1}{n} y'$, we find that $y_\ast^2 + (\tau-1)y_\ast' = 0$, and so $y_\ast' = -\frac{1}{\tau - 1} y_\ast^2$. Applying this to the derivative of \eqref{eq:FirstOrderAnsatz} gives
\[
    0 = \partial_\zeta P_2 + (\partial_y P_2) y' + \frac{1}{n}y'' = \partial_\zeta P_2 + (\partial_y P_2)\left(-\frac{y_\ast^2}{\tau-1}\right) + \frac{1}{n} y_\ast'' = 0.
\]
Therefore, assuming additionally that $n^{-1} y_\ast'' \to 0$ as $n \to \infty$, we find that
\[
    y_\ast^2 \partial_y P_2 - (\tau-1) \partial_\zeta P_2 = 0.
\]
Consequently, under these assumptions, we recover both of the original branch point conditions in the limit as $n \to \infty$. Most significantly, however, is that we can show that the finite-size correction suggests a Tracy--Widom-type scaling of $\mathcal{O}(n^{-2/3})$ in the left-edge. Performing a complete analysis in this respect is well beyond the scope of this paper. Instead, we can shortcut the proof by assuming a particular form of $m_n(z_n)$ where each $z_n$ is the branch point of the finite-size correction $m_n$. In particular, if $(z_\ast, m_\ast)$ is the limiting branch point as $n \to \infty$, we may assert that 
\[
m_n(z_n) \approx m_\ast + b_n u(a_n^{-1}(z_n - z_\ast)),
\]
for sequences $a_n, b_n$ and a function $u$.

Given that \(m_n\) has \(n\) poles located at the eigenvalues, and we should expect that \(a_n\) is the appropriate rescaling of the spectral gap as \(n \to \infty\), \(u\) should contain at least two poles: one for the eigenvalue near the branch point (e.g., the left edge), and another for the next closest eigenvalue. This assumption is sufficient to uniquely determine the limiting behavior, and we will in fact find that $u$ has infinitely many poles to capture \emph{all} of the eigenvalues. 


\begin{proposition}
    \label{prop:TW}
    Assume that for each $n$, the finite-order resolvent $m_n$ satisfies \eqref{eq:FirstOrderAnsatz} and that $(z_\ast,m_\ast)$ is a branch point of $P$ in the sense that $P(z_\ast,m_\ast) = 0$ and $\partial_m P(z_\ast, m_\ast) = 0$. Assume that $\partial_z P(z_\ast,m_\ast) \neq 0$ and $\partial_m^2 P(z_\ast, m_\ast) \neq 0$. Consider a sequence $(z_n,m_n(z_n))$ defined in terms of coefficients $a_n, b_n > 0$, a parameter $\xi$, and a function $u_n(\xi)$ by
    \[
    z_n = z_\ast + a_n \xi,\qquad m_n(z_n) = m_\ast + b_n u_n(\xi),
    \]
    where $a_n, b_n \to 0$ as $n \to \infty$. Assume that \(u_n(\xi)\) converges locally uniformly away from its poles to a meromorphic limit \(u(\xi)\) with at least two poles. Then
    \[
    a_n \asymp n^{-2/3},\qquad b_n \asymp n^{-1/3}.
    \]
\end{proposition}

\begin{proof}
    Let $x_n = a_n \xi$ and $y_n = b_n u_n(\xi)$ denote the local scaling variables such that $z_n = z_\ast + x_n$ and $m_n(z_n) = m_\ast + y_n$. Expanding $P$ around $z_\ast$ gives
    \begin{subequations}
    \label{eq:FiniteOrderTaylor}
    \begin{align}
    P(z_n, m_\ast) &= P(z_\ast, m_\ast) + \partial_z P(z_\ast, m_\ast) x_n + \mathcal{O}(x_n^2), \label{eq:Taylor_P}\\
    \partial_m P(z_n, m_\ast) &= \partial_m P(z_\ast, m_\ast) + \mathcal{O}(x_n), \label{eq:Taylor_dP}\\
    \partial_m^2 P(z_n, m_\ast) &= \partial_m^2 P(z_\ast, m_\ast) + \mathcal{O}(x_n), \label{eq:Taylor_ddP}
    \end{align}
    \end{subequations}
    Since $(z_\ast, m_\ast)$ is a branch point, $P(z_\ast,m_\ast) = \partial_m P(z_\ast,m_\ast) = 0$. Now define the non-zero constants $A = \frac12 \partial_m^2 P(z_\ast, m_\ast)$ and $B = \partial_z P(z_\ast, m_\ast)$. Then substituting \eqref{eq:FiniteOrderTaylor} into the definition of $P_2$ gives
    \[
    P_2(z_n, m_n(z_n)) = A y_n^2 + B x_n + E_n,
    \]
    where $E_n = o(x_n) + \mathcal{O}(x_n y_n) + o(y_n^2)$. But now
    \[
    m_n'(z_n) = \frac{\mathrm{d}(m_\ast + b_n u_n(\xi))}{\mathrm{d}(z_\ast + a_n \xi)} = \frac{b_n}{a_n} u_n'(\xi),
    \]
    and so \eqref{eq:FirstOrderAnsatz} implies that
    \begin{equation}
    \label{eq:IntermediateFirstOrder}
    A b_n^2 u_n(\xi)^2 + B a_n \xi + \frac{b_n}{na_n} u_n'(\xi) + E_n = 0.
    \end{equation}
    Now let $\theta_n = \max\{|A|b_n^2, |B|a_n, |b_n/(n a_n)|\}$. Dividing \eqref{eq:IntermediateFirstOrder} by $\theta_n$ gives
    \[
    \alpha_n u_n(\xi)^2 + \beta_n \xi + \gamma_n u_n'(\xi) + \frac{E_n}{\theta_n} = 0,
    \]
    where $\alpha_n = A b_n^2 / \theta_n$, $\beta_n = B a_n / \theta_n$, and $\gamma_n = b_n / (n a_n \theta_n)$. By construction, $\alpha_n,\beta_n,\gamma_n \in [-1,1]$, and for any $n$, at least one of these coefficients has an absolute value of $1$. By the Bolzano--Weierstrass Theorem, passing to a subsequence if necessary, these coefficients converge to finite limits $\alpha,\beta,\gamma$, respectively, with at least one limit being strictly non-zero. Since $E_n$ is strictly bounded by terms that vanish faster than $\theta_n$, the remainder vanishes in the limit. Therefore, taking the limit as $n \to \infty$, 
    \begin{equation}
    \label{eq:FiniteOrderLimODE}
    \alpha u(\xi)^2 + \beta \xi + \gamma u'(\xi) = 0.
    \end{equation}
    If $\gamma = 0$, then $\alpha u^2 + \beta \xi = 0$. This implies that $u(\xi) = \pm \sqrt{-\frac{\beta}{\alpha}\xi}$, which has a branch cut and is without poles, contradicting the hypotheses\footnote{This represents the regime where the finite-order correction is removed entirely, and so only describes the asymptotic macroscopic square-root behavior near the edge of the spectral density.}. If $\alpha = 0$ then $\beta \xi + \gamma u'(\xi) = 0$. Integrating this reveals that $u(\xi) = -\frac{\beta}{2\gamma} \xi^2 + C$ for some constant $C$ which has no poles. If $\beta = 0$, the equation becomes $\alpha u^2 + \gamma u' = 0$. Solving this differential equation yields $u(\xi) = \frac{\gamma}{\alpha(\xi - C)}$ for some constant $C>0$. This is meromorphic, but possesses only one pole\footnote{This corresponds to the regime where the spectral gap has been stretched to the point where only one pole at the location of the left edge remains.}. Therefore, $\alpha,\beta,\gamma$ must all be non-zero. For this to occur, it must be the case that
    \[
    b_n^2 \asymp a_n \asymp \frac{b_n}{n a_n}.
    \]
    Now, since $b_n^2 \asymp a_n$, this implies that $b_n^2 \asymp \frac{b_n}{nb_n^2} = \frac{1}{nb_n}$ and so $b_n^3 \asymp n^{-1}$. Immediately this implies that $b_n \asymp n^{-1/3}$ and $a_n \asymp n^{-2/3}$. It remains only to identify that the limiting $u$ satisfies the condition of being meromorphic with at least two poles. Note that the limiting equation \eqref{eq:FiniteOrderLimODE} is a Riccati equation, and so by applying the transformation
    \[
    u(\xi) = \frac{\gamma}{\alpha} \cdot \frac{v'(\xi)}{v(\xi)},
    \]
    and then the linear rescaling $w = -(\alpha \beta \gamma^{-2})^{1/3} \xi$, the differential equation of $v(\psi)$ becomes the Airy equation $v''(w) - w v(w) = 0$. The solutions for \(v\) are linear combinations of the Airy functions \(\operatorname{Ai}\) and \(\operatorname{Bi}\) in \(w\) \citep[Section 9.2(i)]{NIST-DLMF}. For the relevant real solution, \(v\) is a nontrivial real linear combination of these functions and has infinitely many real zeros, which are simple \citep[Section 9.9]{NIST-DLMF}. Therefore, \(u(\xi)\) is meromorphic with infinitely many poles on the real axis, and in particular satisfies the required condition.
\end{proof}

The finite-size correction becomes particularly relevant in our diffusion model example, as free decompression predicts that the left-edge should go to zero, even though the matrix in question is strictly positive-definite. Fortunately, the decay rate in the intermediate regime appears to be well-captured by the Tracy--Widom-type scaling behavior predicted in \Cref{prop:TW}; see \Cref{fig:diff-left-edge}. Unfortunately, as $n$ grows larger, the left-edge begins to obey an exponential regime instead. This regime is not captured by the finite-size correction. Instead, over the larger matrix sizes considered here, the observed behavior is well described by an exponential fit. For related probabilistic bounds on the smallest singular values of random matrices with independent subgaussian entries, see \citet{RUDELSON-2009}. We do not expect that simple corrections to the free decompression operation will adequately capture this behavior at these size scales.

\begin{figure}[t!]
    \centering
    \includegraphics[width=0.47\textwidth]{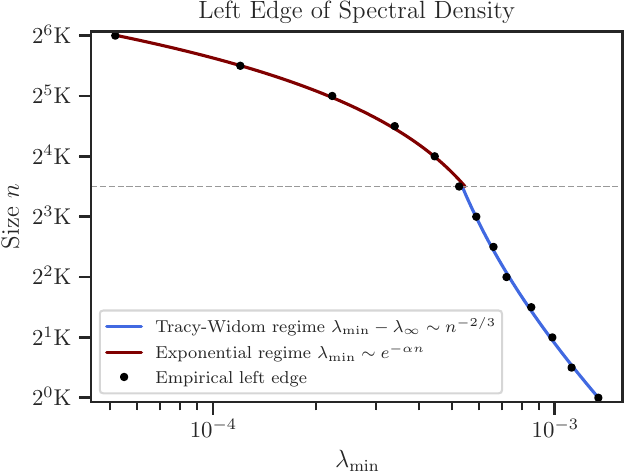}
    \caption{Simulated (black) smallest eigenvalue $\lambda_{\min}$ of the diffusion model matrix example, with Tracy--Widom (blue) predicted decay at rate $\asymp n^{-2/3}$ and exponential (red) predicted decay at rate $\asymp e^{-\alpha n}$ for fitted $\alpha$. }
    \label{fig:diff-left-edge}
\end{figure}


\section{Random Matrix Models} \label{sec:ensemble-models}

This section describes the random matrix models and empirical datasets used in \Cref{sec:results}. We begin with the analytically tractable compound free Poisson and free L\'{e}vy laws in \Cref{sec:compound-poisson,sec:levy}, whose Stieltjes transforms satisfy explicit algebraic relations and therefore provide controlled benchmarks for free decompression. We then discuss the Pennington--Bahri Hessian model in \Cref{sec:pennington-bahri}, which is a neural-network-inspired special case of the algebraic free L\'{e}vy family. Finally, we describe the empirical Hessian and diffusion-model matrices in \Cref{sec:hessian-example,sec:bg-diffusion}, which serve as large-scale data-driven examples.


\subsection{Compound Free Poisson Law} \label{sec:compound-poisson}

The compound free Poisson law \citep[p. 218]{NICA-2006} is a convenient multi-bulk testbed whose Stieltjes transform satisfies an explicit algebraic relation. It may be viewed as a natural generalization of the Marchenko--Pastur law, with the usual Marchenko--Pastur case recovered as the simplest one-parameter member of the family. 

This distribution admits a natural random-matrix realization \citep[Section 22.5.4]{SPEICHER-2015}. Let \(\tens{Z} \in \mathbb{R}^{p \times n}\), with \(p/n \to \lambda \in (0, \infty)\), have iid centered entries with variance \(1\). Suppose \(\gtens{\Sigma}\) is deterministic positive semidefinite with empirical spectral distribution converging weakly to a compactly supported probability measure \(\nu\). Then the companion sample covariance matrix
\begin{equation}
    \underline{\tens{S}} \coloneqq \frac{1}{n} \tens{X}^{\intercal} \tens{X},
    \qquad
    \tens{X} \coloneqq \gtens{\Sigma}^{\frac{1}{2}} \tens{Z},
    \label{eq:comp-S}
\end{equation}
has a limiting eigenvalue distribution given by the compound free Poisson law with rate \(\lambda\) and jump distribution \(\nu\). When \(\lambda < 1\), the companion matrix \(\underline{\tens{S}}\) has at least \(n - p\) zero eigenvalues. The limiting law has mass \(\max\{1 - \lambda (1 - \nu(\{0\})), 0\}\) at the origin, which reduces to \(\max\{1 - \lambda, 0\}\) when \(\nu(\{0\}) = 0\).

The limiting eigenvalue distribution above can be characterized through its R-transform described below. Let \(\nu\) denote the jump distribution of the compound free Poisson law. For the closed-form algebraic examples used here, we take \(\nu\) to be a finite-atom probability measure,
\begin{equation}
    \nu(\mathrm{d} x) \coloneqq \sum_{i=1}^{r} w_i \delta_{t_i}(\mathrm{d} x),
    \label{eq:finite-jump-measure}
\end{equation}
where \(\delta_{t_i}\) denotes a unit point mass at \(x=t_i\), the atoms are distinct, \(w_i>0\), and \(\sum_{i=1}^{r} w_i=1\). The R-transform is then
\begin{equation}
    R(w) = \lambda \int_{\mathbb{R}} \frac{x}{1-xw}\, \nu(\mathrm{d}x)
    = \lambda \sum_{i=1}^{r} w_i \frac{t_i}{1-t_i w},
    \label{eq:cfp-R}
\end{equation}
where \(\lambda>0\) denotes the intensity, or rate. Let \(m(z)\) be the Stieltjes transform of the law. The standard relation between the Stieltjes transform and the R-transform reads; see, \eg \citep[Theorem 12.7]{NICA-2006},
\begin{equation}
    z = R(w) + \frac{1}{w},
    \qquad
    w = -m(z),
    \label{eq:R-m}
\end{equation}
where the physical branch of \(m\) is selected by the Herglotz condition \(\Im m(z) > 0\) for \(z \in \mathbb{C}^+\).

When \(\nu\) is finitely supported as in \eqref{eq:finite-jump-measure}, multiplying \eqref{eq:R-m} by the common denominator \(m \prod_{i=1}^{r} (1+t_i m)\) yields an explicit algebraic constraint
\begin{equation}
    P(z,m)=0,
\end{equation}
where \(P\) has degree \(r+1\) in \(m\). In the special case \(r=1\), this relation reduces to the familiar quadratic equation for the Marchenko--Pastur Stieltjes transform.

\begin{example}
    For the two-atom case \(r=2\) used in our experiments, \(P\) is cubic:
    \begin{equation}
        P(z,m) \coloneqq a_3(z) m^3 + a_2(z) m^2 + a_1(z) m + a_0(z)=0,
        \label{eq:cp-poly}
    \end{equation}
    with coefficients
    \begin{equation}
        \begin{aligned}
            a_3(z) &= z\, t_1 t_2, \\
            a_2(z) &= z(t_1+t_2) + (1-\lambda) t_1 t_2 , \\
            a_1(z) &= z + (t_1+t_2) - \lambda (w_1 t_1 + w_2 t_2), \\
            a_0(z) &= 1,
        \end{aligned}
        \label{eq:cp-poly-coeff}
    \end{equation}
    where \(w_2 \coloneqq 1-w_1\). Hence, \(m(z)\) is a three-sheet algebraic function, providing a controlled benchmark for continuation and evolution in multi-bulk settings.
\end{example}

\begin{figure}[t]
    \centering
    \includegraphics[width=\textwidth]{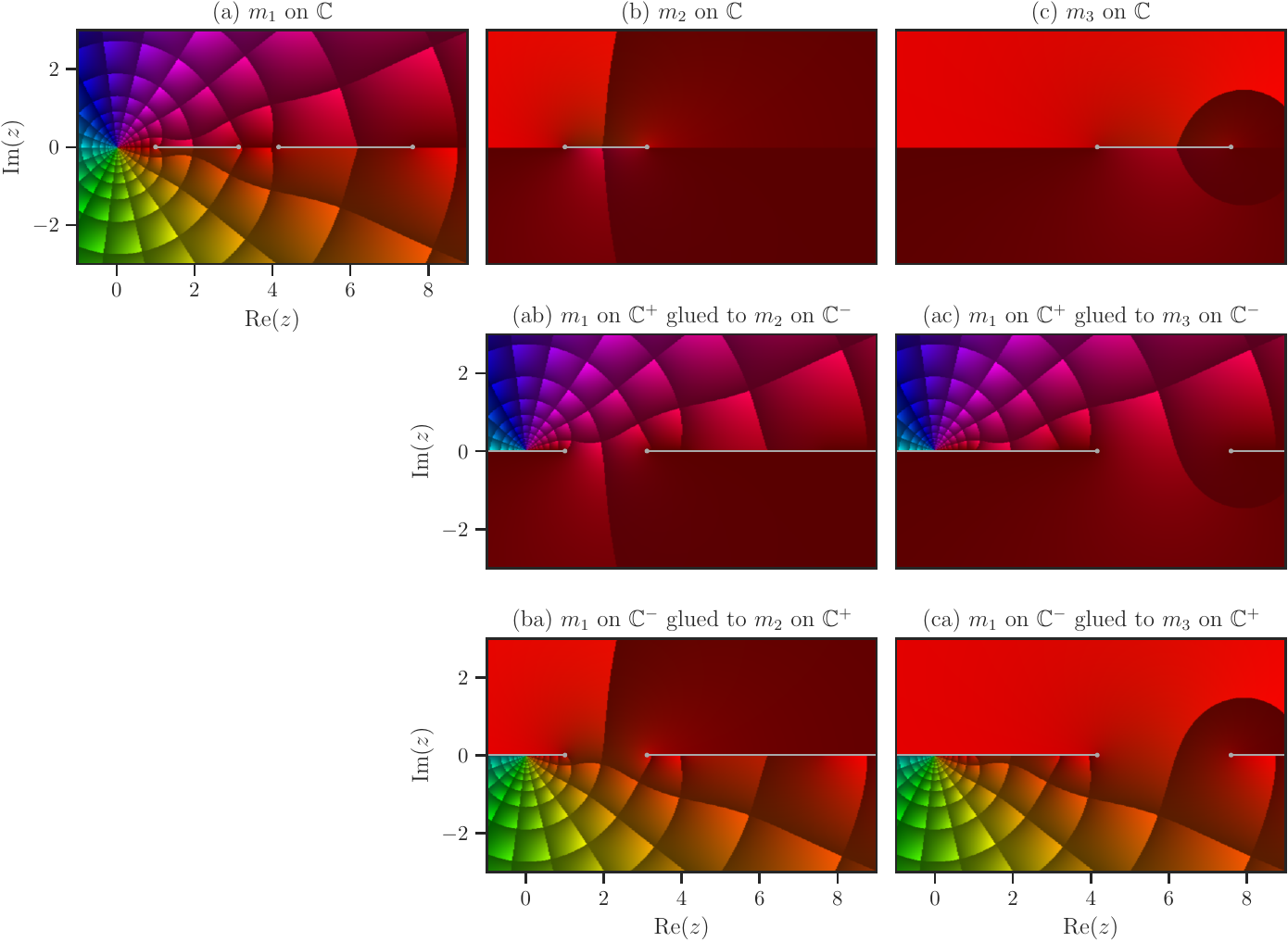}
    \caption{Branches and analytic continuations of the cubic compound free Poisson Stieltjes transform in \eqref{eq:cp-poly}. The first row shows the three algebraic branches \(m_1,m_2,m_3\) on the \(z\)-plane, with branch cuts marked on the real axis. The physical branch is \(m_1\). The lower panels show the four possible half-plane gluings involving the physical sheet: for example, \((ab)\) denotes \(m_1\) on \(\mathbb{C}^{+}\) glued to \(m_2\) on \(\mathbb{C}^{-}\), while \((ba)\) denotes the opposite gluing. Across a shared cut, the individual branches are discontinuous, but the glued half-planes form a continuous analytic continuation. Colors encode the complex value of the branch using the domain-coloring convention of \Cref{sec:notation}.}
    \label{fig:cfp-branches}
\end{figure}

\begin{figure}[t]
    \centering
    \ifjournal
        \includegraphics[width=0.67\textwidth]{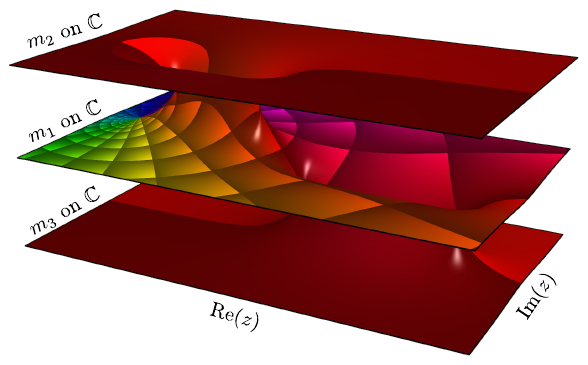}
    \else
        \includegraphics[width=0.63\textwidth]{\figdir/cfp-glue-sheets}
    \fi
    \caption{Three-sheet visualization of the cubic compound free Poisson spectral curve from \Cref{fig:cfp-branches}. The middle sheet represents the physical branch \(m_1\), while the upper and lower sheets represent the non-physical branches \(m_2\) and \(m_3\). The sheets are vertically separated for visualization and are brought together along shared branch cuts to illustrate analytic continuation between branches. This geometric view explains why representing the Stieltjes transform as one sheet of a larger algebraic curve is useful: crossing a branch cut does not terminate the function, but moves the continuation to another sheet.}
    \label{fig:cfp-glue-sheets}
\end{figure}

For the numerical example using the compound free Poisson law, we use \(\lambda = 0.1\), \((t_1, t_2) = (2, 5\frac{1}{2})\), and \((w_1, w_2) = (0.75, 0.25)\). This produces a cubic algebraic Stieltjes transform, so in the numerical fit we use \((\deg_m, \deg_z) = (3, 1)\) and impose only the zeroth-moment normalization constraint.

The cubic compound-free-Poisson example also provides a useful visualization of the sheet structure underlying algebraic Stieltjes transforms. Since \eqref{eq:cp-poly} is cubic, each point \(z\) away from branch points has three algebraic roots, denoted \(m_1, m_2, m_3\), shown in the first row of \Cref{fig:cfp-branches}. The physical Stieltjes transform is the branch selected by the Herglotz property and the normalization at infinity; here, this branch is \(m_1\). The remaining roots are nonphysical on \(\mathbb{C}^{+}\), but they are not spurious: they encode the analytic continuation of the physical branch after crossing real branch cuts. For the physical branch, these cuts coincide with the support intervals of the density on the real axis. The lower panels of \Cref{fig:cfp-branches} show the half-plane gluings that continue the physical branch through these cuts, and \Cref{fig:cfp-glue-sheets} shows the same continuation geometry as a three-sheeted algebraic surface. In this surface view, the cuts are not endpoints of the Stieltjes transform, but seams where the physical sheet is attached to nonphysical sheets. Thus, crossing a cut does not terminate the analytic continuation; it moves the continuation from one sheet of the spectral curve to another. This is the main geometric advantage of working with the full algebraic relation \(P(z,m)=0\): the physical Stieltjes transform is only one sheet of a larger connected algebraic object, while the other sheets record its analytic continuations.


\subsection{Free L\'{e}vy Law} \label{sec:levy}

The compound free Poisson law can be generalized by adding a drift and a semicircular, or free Gaussian, component. This yields the free L\'{e}vy law
\begin{equation*}
    \nu_{\mathrm{FL}} = \nu_{\delta_a} \boxplus \nu_{\mathrm{SC}_{\sigma^2}} \boxplus \nu_{\mathrm{CFP}(\lambda,\nu)},
\end{equation*}
where \(\delta_a\) denotes the deterministic shift \(X \mapsto X+a\), \(\mathrm{SC}_{\sigma^2}\) is Wigner's semicircle law with variance \(\sigma^2\), and \(\mathrm{CFP}(\lambda,\nu)\) is the compound free Poisson law with rate \(\lambda\) and jump distribution \(\nu\), as in \Cref{sec:compound-poisson}. In full generality, the free L\'{e}vy--Khintchine representation characterizes freely infinitely divisible distributions: when \(\nu\) is a L\'{e}vy measure, the corresponding law is freely infinitely divisible, and conversely every freely infinitely divisible law admits such a representation \citep[Theorem 13.8]{NICA-2006}; see also \citet{GEORGES-2005} for random-matrix constructions associated with freely infinitely divisible laws. In this work, we restrict to the finite-atom case \eqref{eq:finite-jump-measure}, which preserves an algebraic Stieltjes transform while still producing a broad family of multi-bulk distributions.

The free L\'{e}vy law also admits a natural random-matrix realization by combining the drift, Gaussian, and jump components; see \citet{BURDA-2002, GEORGES-2005} for related random-matrix realizations of free L\'{e}vy and freely infinitely divisible laws. Let \(\tens{W}\) be an \(n \times n\) Wigner matrix with independent centered entries of variance \(1/n\), up to symmetry, and let \(\underline{\tens{S}}\) be the companion sample covariance matrix in \eqref{eq:comp-S}, whose limiting law is \(\mathrm{CFP}(\lambda,\nu)\). Assuming \(\tens{W}\) and \(\underline{\tens{S}}\) are independent, the additive model
\begin{equation*}
    \tens{A} \coloneqq a \tens{I} + \sigma \tens{W} + \underline{\tens{S}}
\end{equation*}
has an empirical spectral distribution converging to \(\nu_{\mathrm{FL}}\) under standard assumptions.

The R-transform of \(\nu_{\mathrm{FL}}\) is the sum of the R-transforms of its free additive components:
\begin{equation*}
    R_{\mathrm{FL}}(w) = R_{\delta_a}(w) + R_{\mathrm{SC}_{\sigma^2}}(w) + R_{\mathrm{CFP}(\lambda,\nu)}(w).
\end{equation*}
Here, \(R_{\delta_a}(w)=a\), \(R_{\mathrm{SC}_{\sigma^2}}(w)=\sigma^2 w\), and, for the finite-atom jump distribution \eqref{eq:finite-jump-measure}, \(R_{\mathrm{CFP}(\lambda,\nu)}\) is given by \eqref{eq:cfp-R}. Hence,
\begin{equation}
    R_{\mathrm{FL}}(w) = a + \sigma^2 w + \lambda \sum_{i=1}^{r} w_i \frac{t_i}{1-t_i w},
    \label{eq:fl-R}
\end{equation}
which is a rational function.

Using the relation between the R-transform and the Stieltjes transform in \eqref{eq:R-m}, with \(w=-m(z)\), we obtain the self-consistent equation
\begin{equation*}
    -\frac{1}{m(z)} = z - a + \sigma^2 m(z) - \lambda \sum_{i=1}^{r} w_i \frac{t_i}{1+t_i m(z)}.
\end{equation*}
Multiplying by the common denominator \(m \prod_{i=1}^{r} (1+t_i m)\) yields the implicit algebraic constraint \(P(z,m)=0\), where
\begin{equation}
    P(z,m) \coloneqq \left(\sigma^2 m^2 + (z-a)m + 1 \right) \prod_{i=1}^{r} (1+t_i m) - \lambda m \sum_{i=1}^{r} w_i t_i \prod_{j \neq i} (1+t_j m).
    \label{eq:fl-poly}
\end{equation}
For \(\sigma=0\), the polynomial \eqref{eq:fl-poly} has degree \(r+1\) in \(m\) and reduces to the compound free Poisson relation, up to the drift shift \(a\). When \(\sigma>0\), the additional semicircular term increases the degree to \(r+2\).

\begin{example} \label{ex:free-levy}
    For the two-atom case \(r=2\) and \(\sigma>0\), \eqref{eq:fl-poly} becomes a quartic polynomial in \(m\),
    \begin{equation}
        P(z,m) = a_4(z)m^4 + a_3(z)m^3 + a_2(z)m^2 + a_1(z)m + a_0(z),
        \label{eq:fl-P}
    \end{equation}
    with coefficients
    \begin{equation}
        \begin{aligned}
            a_4(z) &= \sigma^2 t_1 t_2, \\
            a_3(z) &= (z-a)t_1 t_2 + \sigma^2(t_1+t_2), \\
            a_2(z) &= (z-a)(t_1+t_2) + (1-\lambda)t_1t_2 + \sigma^2, \\
            a_1(z) &= (z-a) + (t_1+t_2) - \lambda(w_1t_1+w_2t_2), \\
            a_0(z) &= 1.
        \end{aligned}
        \label{eq:fl-poly-coeff}
    \end{equation}
    In particular, for \(\sigma=0\), the leading coefficient \(a_4\) vanishes and \eqref{eq:fl-P} reduces to the cubic compound free Poisson law with drift.
\end{example}

For the numerical example below, we use \(\lambda = 0.1\), \((t_1, t_2) = (2, 5\frac{1}{2})\), \((w_1, w_2) = (0.75, 0.25)\), \(a = 0\), and \(\sigma = 0.4\). This produces a quartic algebraic Stieltjes transform, so in the numerical fit we use \((\deg_m, \deg_z) = (4, 1)\) and impose only the zeroth-moment normalization constraint. \Cref{fig:fl-branches} visualizes the four sheets and their analytic continuations. As in the compound free Poisson example, the individual branches are discontinuous across their cuts, but the appropriate half-plane gluings form continuous analytic continuations on the full spectral curve.

\begin{figure}[t]
    \centering
    \includegraphics[width=\textwidth]{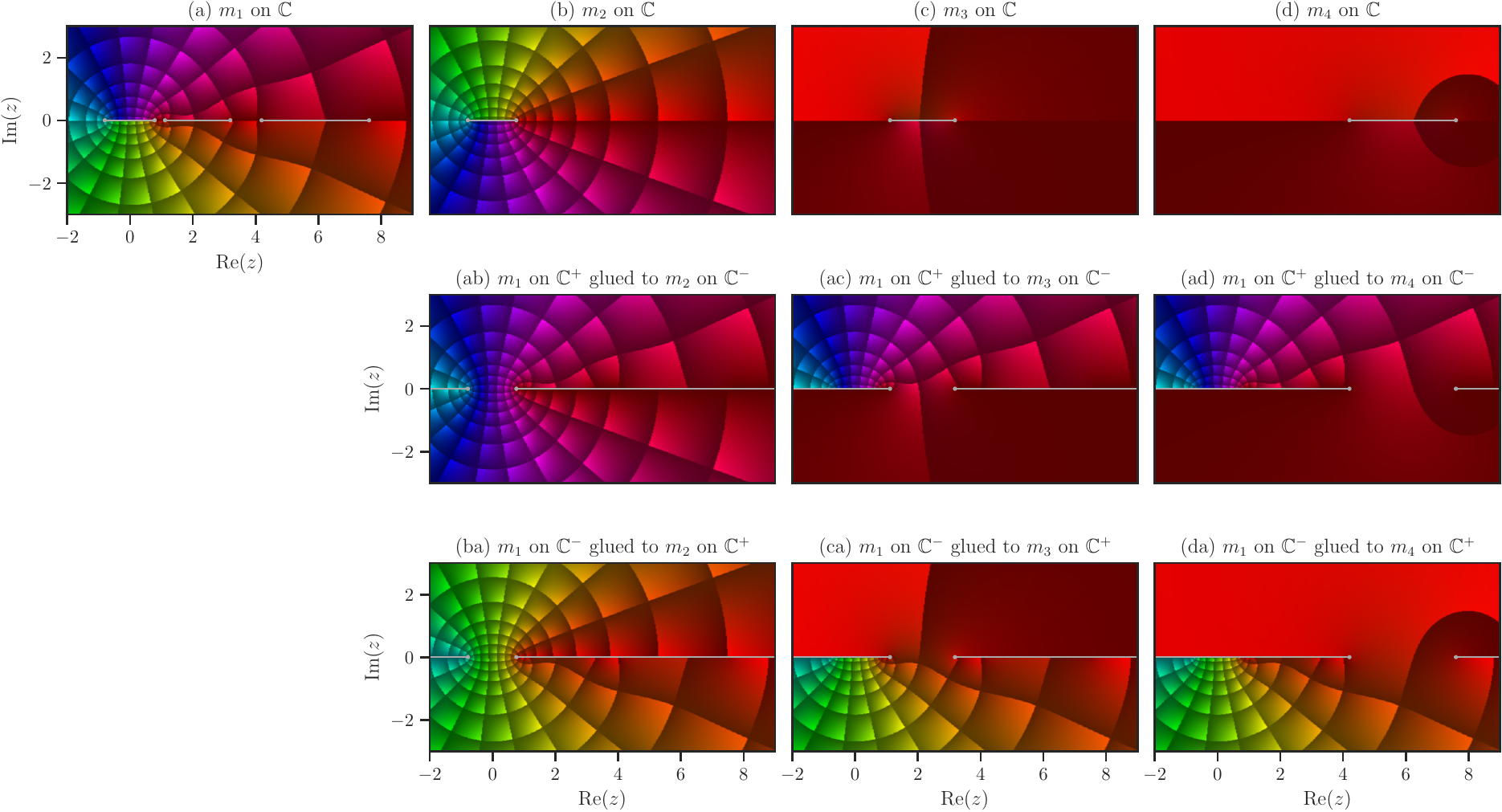}
    \caption{Branches and analytic continuations of the quartic free L\'evy Stieltjes transform. The first row shows the four algebraic branches of \(P(z,m)=0\) on the \(z\)-plane, with branch cuts marked on the real axis. The lower panels show half-plane gluings between the physical sheet and the non-physical sheets; for example, \((ad)\) denotes the upper half-plane of the physical sheet glued to the lower half-plane of the branch shown in panel (d), while \((da)\) denotes the opposite gluing. Across a shared cut, the separate branches are discontinuous, but the glued half-planes form a continuous analytic continuation. Colors encode the complex value of the branch using the domain-coloring convention of \Cref{sec:notation}.}
    \label{fig:fl-branches}
\end{figure}

We next use this law as a controlled free-decompression benchmark. A matrix realization of size \(n=2^4\mathrm{K}\) is generated from the free L\'evy law above, and a \(2^2\mathrm{K}\)-sized principal submatrix is used as the initial compressed matrix. The algebraic curve is fitted from the initial empirical spectrum and then evolved by free decompression to the full size.

\Cref{fig:fl-flow} shows the resulting density evolution. Panel (a) compares the initial empirical histogram with the density recovered from the fitted algebraic curve. Panel (b) shows the free-decompression flow through intermediate sizes, and panel (c) compares the final FD prediction at \(n=2^4\mathrm{K}\) with the empirical spectrum of the full matrix. The initial density has one bulk, while the decompressed density splits into three bulks. The two right bulks arise from the compound free Poisson component, and the smallest rightmost bulk is resolved despite having much smaller mass and height than the dominant component.

\begin{figure}[t]
    \centering
    \includegraphics[width=\textwidth]{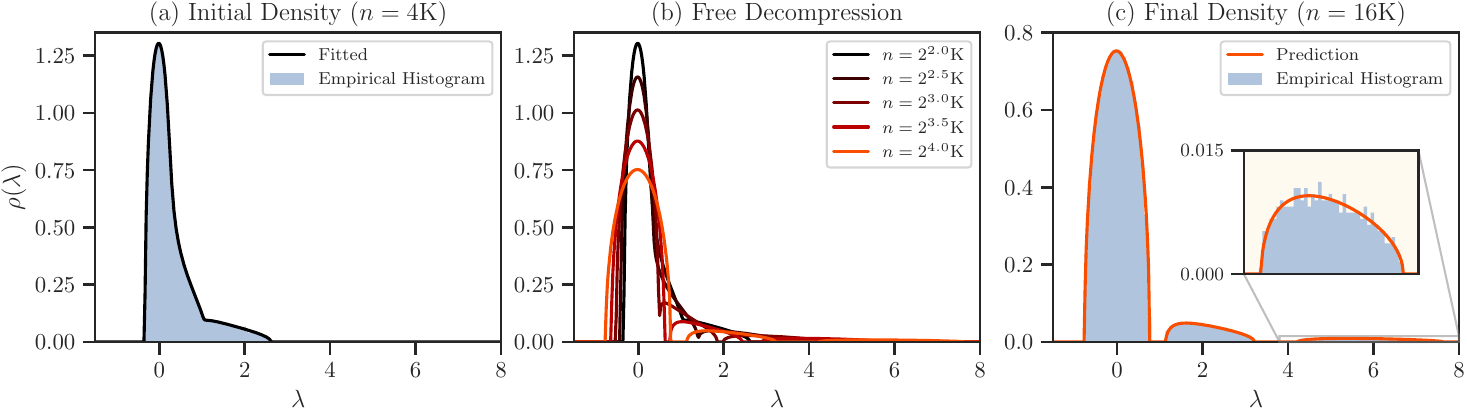}
    \caption{Free decompression of the free L\'evy benchmark. A matrix realization of size \(n=2^4\mathrm{K}\) is generated from the quartic free L\'evy law, and a \(2^2\mathrm{K}\)-sized principal submatrix is used as the initial compressed matrix. Panel (a) shows the initial empirical histogram together with the density recovered from the fitted algebraic curve. Panel (b) shows the FD evolution through intermediate sizes. Panel (c) compares the final FD prediction with the empirical histogram of the full matrix. The initially single-bulk density splits into three bulks, including a small rightmost component that is accurately recovered by the evolved spectral curve.}
    \label{fig:fl-flow}
\end{figure}

The same example also illustrates the direct evolution of spectral edges. Rather than locating edges from the reconstructed density at each time, we evolve the branch-point conditions directly as described in \Cref{sec:edge}. \Cref{fig:fl-edge2} shows the edge trajectories from \(n=2^2\mathrm{K}\) to \(n=2^4\mathrm{K}\). The density field computed by free decompression is shown in the background, while the black curves show the independently evolved spectral edges. Starting from one initial bulk, two cusp events create four additional edges, producing three bulk intervals at the final size. The agreement between the black edge curves and the boundary of the nonzero-density region confirms that the edge evolution and density evolution are consistent.

\begin{figure}[t]
    \centering
    \ifjournal
        \includegraphics[width=0.75\textwidth]{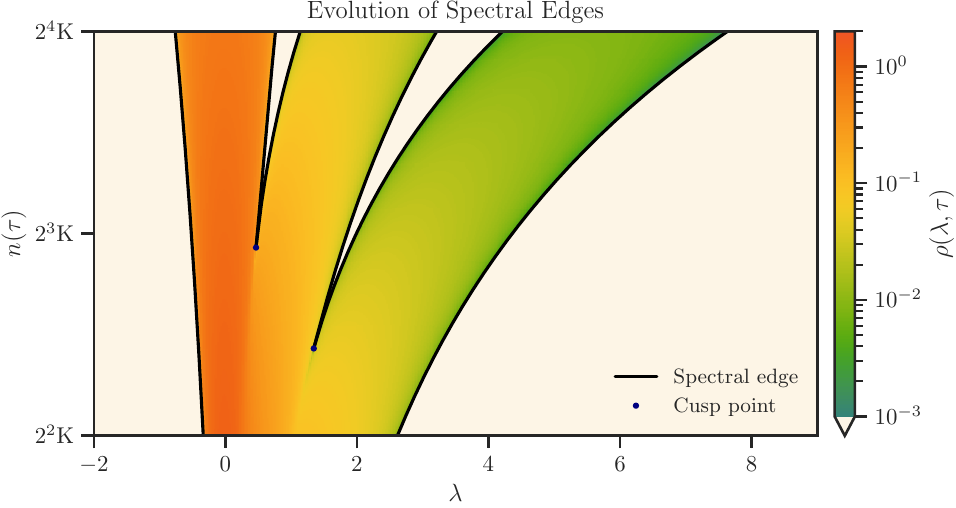}
    \else
        \includegraphics[width=0.65\textwidth]{\figdir/fl-edge}
    \fi
    \caption{Evolution of spectral edges for the free L\'evy benchmark. The vertical axis is the matrix size \(n(\tau)\), and the horizontal axis is the spectral coordinate \(\lambda\). Black curves show spectral edges evolved directly from the branch-point equations in \Cref{sec:edge}, while the background color shows the density \(\rho(\lambda,\tau)\) computed by free decompression on a fine time grid. The initial spectrum has one bulk; two cusp events create new pairs of edges, leading to three bulk intervals at the final size. The evolved edges align with the boundary of the nonzero-density region, showing consistency between direct edge evolution and full density evolution.}
    \label{fig:fl-edge2}
\end{figure}

\paragraph{Sensitivity to Polynomial Degree.}

The free L\'evy example provides a controlled setting for studying degree selection, since its Stieltjes transform has a known minimal bidegree \((d_z, s) = (1, 4)\). Keeping the initial submatrix and numerical settings fixed, we repeat the experiment with different polynomial degrees, as reported in \Cref{tab:fl-degree}. We find that the dominant failure mode is structural under-specification. The pairs \((1, 2)\) and \((1, 3)\) fail to reproduce the correct three-bulk topology. This agrees with the necessary degree condition \((2s-1) d_z \geq 2k\) for \(k\) bulk intervals, derived in \Cref{sec:constraint-branch}. Below this bound, the polynomial lacks sufficient degree to represent the required spectral cuts.

In contrast, all tested models from the minimal \((1,4)\) through the larger \((4,4)\) recover the three-bulk structure accurately. The small, non-monotone variations in error are consistent with numerical fluctuations from fitting and continuation, rather than a systematic penalty from modest over-specification. This is expected since larger polynomial classes still contain the minimal curve. Thus, the exact minimal degree need not be known in advance. We can begin with the smallest structurally admissible pair and verify stability across nearby larger orders. Excessively high degrees may nevertheless introduce conditioning and root-tracking difficulties.

\begin{table}[t]
\centering
\caption{Sensitivity to polynomial degrees for the free L\'evy benchmark. All fits use the same \(2^2\mathrm{K}\) initial submatrix and are decompressed to \(2^4\mathrm{K}\), varying only \((d_z,s)\). Mass errors are measured before density normalization, while KS and \(W_1/L\) compare normalized recovered densities with the exact target density.}
\label{tab:fl-degree}
\begin{tabular}{llrrrr}
    \toprule
     &  & \multicolumn{2}{c}{Mass Error} & \multicolumn{2}{c}{Distributional Distance}  \\
    \cmidrule(lr){3-4} \cmidrule(lr){5-6}
    \((d_z,s)\) & Specification & Initial & Target & KS & \(W_1/L\) \\
    \midrule
    \((1,2)\) & Under         & \(70.798\%\) & \(85.301\%\) & \(90.048\%\) & \(28.493\%\) \\
    \((1,3)\) & Under         & \( 0.008\%\) & \( 0.336\%\) & \( 8.777\%\) & \( 1.634\%\) \\
    \((1,4)\) & Minimal exact & \( 0.004\%\) & \( 0.023\%\) & \( 0.061\%\) & \( 0.009\%\) \\
    \((1,5)\) & Over          & \( 0.001\%\) & \( 0.745\%\) & \( 0.712\%\) & \( 0.196\%\) \\
    \((1,6)\) & Over          & \( 0.002\%\) & \( 0.127\%\) & \( 0.087\%\) & \( 0.023\%\) \\
    \((2,4)\) & Over          & \( 0.015\%\) & \( 0.029\%\) & \( 0.052\%\) & \( 0.015\%\) \\
    \((4,4)\) & Over          & \( 0.003\%\) & \( 0.162\%\) & \( 0.176\%\) & \( 0.029\%\) \\
    \bottomrule
\end{tabular}
\end{table}


\subsection{Pennington--Bahri Model} \label{sec:pennington-bahri}

The Pennington--Bahri model \citep{PENNINGTON-2017c} is a random-matrix model for Hessian matrices of single-layer neural networks with ReLU activations under simplifying assumptions. In this model, the network is assumed to be trained to mean-squared error training loss \(n_2\varepsilon\) over \(m\) data points, where \(n_2\) is the number of data labels. Writing \(\tens{J}\) for the network Jacobian, the Hessian decomposes into two additive components
\begin{equation*}
    \tens{H} = \tens{H}_0 + \tens{H}_1,
\end{equation*}
where \(\tens{H}_0 = \tens{J}\tens{J}^{\intercal}/m\) is the Jacobian component and \(\tens{H}_1\) is the network Hessian term. Under the simplifying assumption that the entries of \(\tens{J}\) and \(\tens{H}_1\) are independent Gaussian variables, \(\tens{H}_0\) and \(\tens{H}_1\) have Wishart and Wigner limiting distributions, respectively.

Consequently, the limiting Hessian distribution can be written as an additive free convolution
\begin{equation*}
    \nu_H = \nu_{H_0} \boxplus \nu_{H_1}.
\end{equation*}
Equivalently, its R-transform is the sum of the R-transforms of the two components,
\begin{equation*}
    R_H(w) = R_{H_0}(w) + R_{H_1}(w) = \frac{1}{1-\lambda w} + 2\varepsilon w,
\end{equation*}
where \(\lambda\) is the aspect ratio of the Jacobian. This places the Pennington--Bahri law within the finite-atom free L\'{e}vy family of \Cref{sec:levy}: it corresponds to zero drift, one compound-free-Poisson support point, and a nonzero semicircular component. In particular, its Stieltjes transform satisfies a cubic algebraic relation.

To perform decompression of this model, we set \(\varepsilon = 0.1\) and \(\lambda = \frac{3}{4}\), and sampled a \(\num{6000}\times\num{6000}\) matrix from the Pennington--Bahri law. From this matrix, a \(\num{2000} \times \num{2000}\) submatrix was uniformly sampled, and free decompression was applied to recover an estimate of the spectral distribution of the original matrix. \Cref{fig:pb-flow} shows the result and should be compared with Figure~1 of \citet{PENNINGTON-2017c}. This example demonstrates the potential of free decompression as a tool for estimating spectral properties of large Hessian models from smaller submatrices that can be explicitly formed or diagonalized.

\begin{figure}[t]
    \centering
    \includegraphics[width=\textwidth]{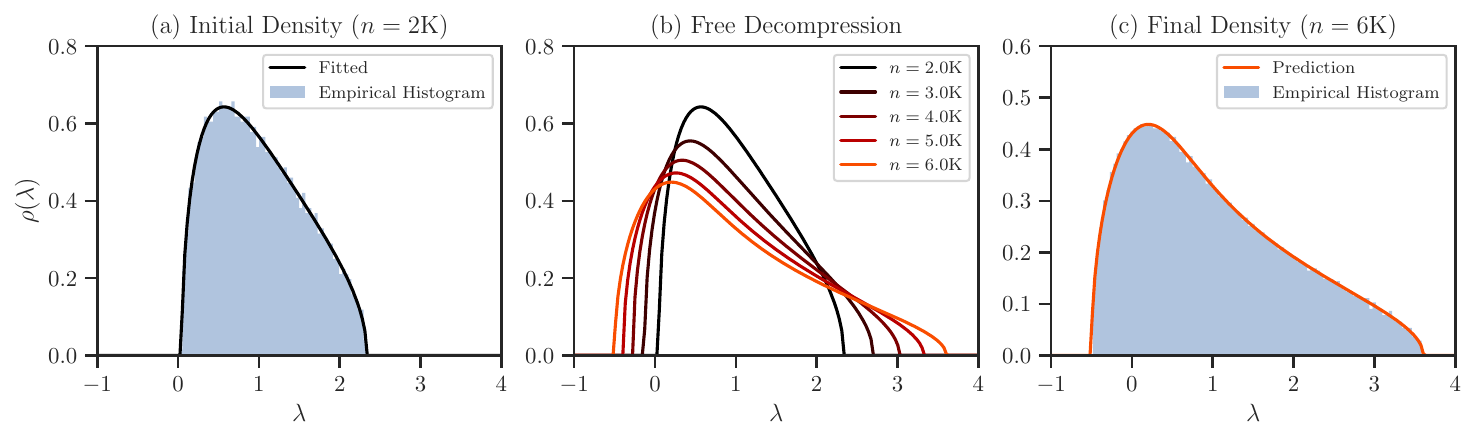}
    \caption{Free decompression for the Pennington--Bahri Hessian model. A matrix \(\tens{A}\in\mathbb{R}^{n\times n}\) with \(n=\num{6000}\) is generated from the Pennington--Bahri law with \(\lambda=\nicefrac{3}{4}\) and \(\varepsilon=\nicefrac{1}{10}\), and a submatrix of size \(n_0=\num{2000}\) is used as the initial compressed matrix. Panel (a) shows the empirical spectrum of the submatrix together with the fitted density. Panel (b) shows the free-decompression flow from \(n_0\) to \(n\). Panel (c) compares the final FD prediction with the empirical spectrum of the full matrix.}
    \label{fig:pb-flow}
\end{figure}


\subsection{Hessian of Two-Layer Network}\label{sec:hessian-example}

Suppose that $\tens{X} \in \mathbb{R}^{n\times p}$ is a matrix of $n$ inputs in $\mathbb{R}^p$. Let $f(\vect{w}, \tens{X}) \in \mathbb{R}^{n \times m}$ denote the network output for inputs $\tens{X}$ with weights $\vect{w} = [w_i]_{i=1}^N$. To investigate the spectral behavior of Hessian matrices for $f$ at initialization, we let $\vect{w}^\ast$ denote a target collection of weights comprised of i.i.d elements generated from a normal distribution $\mathcal{N}(0,\sigma^2)$. The choice of i.i.d elements here is to mimic standard initialization schemes, including the popular He initialization \cite{HE-2016}. Let $\tens{E} \in \mathbb{R}^{n \times m}$ denote a collection of residual errors, where $E_{ij} \sim \mathcal{N}(0,2 \epsilon)$ are also independently generated for some $\epsilon > 0$. The Hessian matrix to be considered is given by $\tens{H} = [H_{ij}]_{i,j=1}^N$ where
\[
    H_{ij} = \frac{1}{2n} \frac{\partial^2}{\partial w_i \partial w_j} \|f(\vect{w}, \tens{X}) - \tens{Y}\|_F^2,
    \qquad \tens{Y} = f(\vect{w}^{\ast}, \tens{X}) + \tens{E}.
\]
Recall that $\tens{H}$ is symmetric and therefore has real eigenvalues. However, $\tens{H}$ is not necessarily positive-definite, and indeed, often possesses negative eigenvalues corresponding to directions of descent. The fraction of negative eigenvalues is called the \emph{normalized index} and is a prominent object of study in random matrix analyses of deep learning models \cite{PENNINGTON-2017c}:
\[
    \mathrm{index} = \frac{\#\{\lambda_i(\tens{H}) < 0\}}{N},\qquad \tens{H} \in \mathbb{R}^{N \times N}.
\]
At scale, when the Hessian becomes so large that it cannot be stored in memory, the index becomes a prime target to be computed using free decompression. 

As a concrete example, we consider a two-layer feedforward neural network with a ReLU activation function where
\[
f(\tens{W}_1, \tens{W}_2, \tens{X}) =  \phi(\tens{X}\tens{W}_1) \tens{W}_2,\qquad \tens{W}_1, \tens{W}_2 \in \mathbb{R}^{p \times p},
\]
where $\phi(x) = \max\{0,x\}$ is the ReLU activation function applied elementwise. Our input data $\tens{X}$ is constructed from the CIFAR-10 dataset \cite{CIFAR10}. Let $\tens{X}_{\mathrm{raw}} \in \mathbb{R}^{60000 \times 1024}$ represent the total batch of 60,000 training images from CIFAR-10 of size $32 \times 32$ rendered in grayscale and flattened. To control the dimensionality, we truncate this matrix to its $n \times p$ upper-left corner, called $\tilde{\tens{X}}$. In our tests, we use $n = 512$ and $p=128$.
This matrix is centered by:
\[
\tens{X}_c = \tilde{\tens{X}} - \frac{1}{p}\tilde{\tens{X}} \vect{1} \vect{1}^{\intercal},
\]
where $\vect{1}$ is a $p$-dimensional vector of ones. The empirical covariance matrix of this centered data is given by $\tens{\Sigma} = \frac{1}{p-1}\tens{X}_c \tens{X}_c^{\intercal}$. To ensure the data is isotropic as per \citet{PENNINGTON-2018a}, we apply Zero Component Analysis (ZCA) whitening. Using the eigendecomposition $\tens{\Sigma} = \tens{V} \tens{\Lambda} \tens{V}^{\intercal}$, we construct the ZCA whitening matrix
\[
\tens{Y}_{\mathrm{ZCA}} = \tens{V} \tens{\Lambda}^{-1/2} \tens{V}^{\intercal},
\]
where the inverse square root of the eigenvalue matrix $\tens{\Lambda}$ is bounded below for numerical stability. The final network input is the whitened data $\tens{X} = \tens{Y}_{\mathrm{ZCA}} \tens{X}_c$. Finally, we choose $\sigma^2 = 1/p$, and $\epsilon = 0.1$. 

\Cref{fig:hessian-index} shows the normalized Hessian index estimated by free decompression across matrix sizes, together with the empirical index computed from direct eigendecomposition at the available sizes.

\Cref{tab:hessian-complexity} complements this comparison by reporting runtime and scalar accuracy metrics. Since the FD output is a continuous spectral measure, whereas the empirical spectrum is discrete, we use distributional metrics that compare probability measures directly. In particular, the normalized Wasserstein distance \(W_1/L\) provides a scale-free discrepancy that accounts for the geometry of the spectral axis and does not require an additional KDE or binning convention \cite{GIBBS-2002,PANARETOS-2019}. The Hessian's index error remains small across all sizes, while the direct eigendecomposition time increases rapidly with \(O(n^3)\). In contrast, free decompression only requires the eigendecomposition of the initial submatrix (at size \(n_0 = 2^{12}\)), followed by a comparatively inexpensive evolution step.

\begin{figure}[t]
    \centering
    \includegraphics[width=0.5\textwidth]{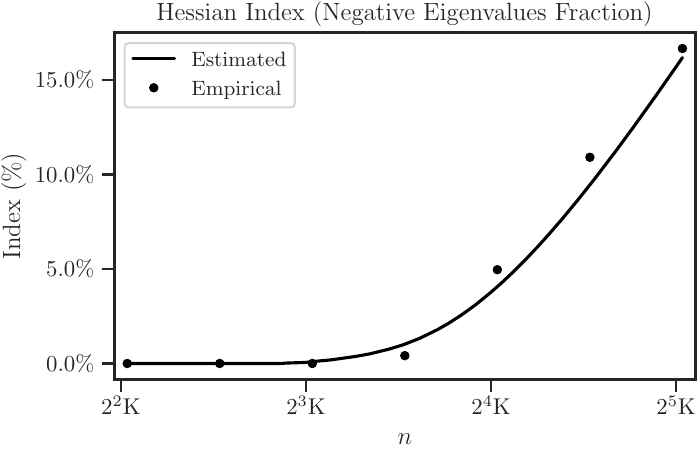}
    \caption{Normalized Hessian index across matrix sizes. The curve shows the index estimated by free decompression, while markers show the empirical index computed directly from empirical eigenvalues at the available sizes.}
    \label{fig:hessian-index}
\end{figure}

For this empirical Hessian example, we fit a polynomial with \((\deg_m, \deg_z) = (8, 4)\) and impose moment constraints up to degree two.

\begin{table}[t!]
\caption{Comparison of direct empirical eigenvalues versus free decompression (FD) for the Hessian example. The Direct column reports the process time for computing the eigenvalues of the Hessian at each size. In the FD column, the first term is the time to compute eigenvalues of the initial \(2^{12}\)-sized submatrix, and the second term is the free-decompression evolution time. The index error is the absolute error in the normalized Hessian index, namely the fraction of negative eigenvalues. KS denotes the Kolmogorov--Smirnov distance between the empirical spectrum and samples drawn from the FD reconstruction. The normalized Wasserstein distance \(W_1/L\) is the Wasserstein-1 distance divided by the empirical spectral width \(L=\lambda_{\max}-\lambda_{\min}\) at each size. The last two columns report relative errors of the first moment \(\mu_1\) and standard deviation \(\sigma\). All accuracy metrics are multiplied by \(100\) and reported as percentages.}
\label{tab:hessian-complexity}
\vspace{1mm}
\centering
\begin{tabular}{lrrrrrrr}
    \toprule
    Size & \multicolumn{2}{c}{Process Time (sec)} & \multicolumn{3}{c}{Accuracy Metrics} & \multicolumn{2}{c}{Moments Rel. Error} \\
    \cmidrule(lr){2-3} \cmidrule(lr){4-6} \cmidrule(lr){7-8}
    \(n\) & Direct & FD (ours) & Index Error & KS & \(W_1/L\) & \(\Delta \mu_1 /\mu_{1}\) & \(\Delta \sigma / \sigma\) \\
    \midrule
    \(2^{12.0}\) & \(     0.9\) & \(0.9 + \phantom{0}0.0\) & \(0.0\%\) & \(0.4\%\) & \(0.0\%\) & \(0.3\%\) & \(0.3\%\) \\
    \(2^{12.5}\) & \(    91.9\) & \(0.9 + \phantom{0}4.7\) & \(0.0\%\) & \(1.2\%\) & \(0.3\%\) & \(3.5\%\) & \(6.1\%\) \\
    \(2^{13.0}\) & \(   271.9\) & \(0.9 + \phantom{0}5.4\) & \(0.1\%\) & \(3.0\%\) & \(0.7\%\) & \(8.5\%\) & \(0.9\%\) \\
    \(2^{13.5}\) & \(   759.4\) & \(0.9 + \phantom{0}7.8\) & \(0.6\%\) & \(2.9\%\) & \(0.7\%\) & \(9.0\%\) & \(0.1\%\) \\
    \(2^{14.0}\) & \(  2154.4\) & \(0.9 +           11.1\) & \(0.8\%\) & \(2.4\%\) & \(0.5\%\) & \(6.8\%\) & \(2.0\%\) \\
    \(2^{14.5}\) & \(  4755.3\) & \(0.9 +           15.7\) & \(1.3\%\) & \(2.1\%\) & \(0.3\%\) & \(2.1\%\) & \(4.9\%\) \\
    \(2^{15.0}\) & \( 14023.5\) & \(0.9 +           22.2\) & \(0.4\%\) & \(2.8\%\) & \(0.3\%\) & \(3.5\%\) & \(13.0\%\) \\
    \bottomrule
\end{tabular}
\end{table}


\subsection{Details of the Diffusion Model Example}
\label{sec:bg-diffusion}

The diffusion model that we consider is based on a random features neural network model, which is by now well studied \citep{LIAO-2021b} and, in the diffusion-model setting, was introduced in \citet{GEORGE-2025}. A key advantage of this model is that, despite its heavy simplification, it still exhibits memorization--generalization phase transitions that are characteristic of full diffusion models with modern architectures \citep{BIROLI-2024, VENTURA-2025}. In this simplified setting, the spectral distribution of the expected activation matrix is tied to the transition between early training, generalization, and memorization behavior \citep{BONNAIRE-2025}. Our goal is to demonstrate that this spectral distribution, even when highly multi-scale and multi-bulk, can be accurately estimated from a sub-sampled model.

To this end, a $\num{64000}\times\num{64000}$ matrix was generated by sampling $m = 800$ i.i.d data-points $\vect{x}^j$, with each $\vect{x}^j \sim \mathcal{N}(\tens{0}, \tens{I}_d)$ for $d = 100$. The random features were generated by sampling  $\num{64000}$ i.i.d random feature weights $\vect{w} \sim \mathcal{N}(\tens{0}, \tens{I}_{d_f})$, and computing the expected diffusion activation matrix
\begin{equation*}
    \tens{U} = \frac{1}{m} \sum_{j = 1}^{m} \mathbb{E} \left[ \sigma \left( \frac{ \tens{W} \vect{x}_t^j}{\sqrt d} \right) \sigma \left( \frac{\tens{W} \vect{x}_t^j}{\sqrt d} \right)^{\intercal} \right].
\end{equation*}
Here, the nonlinearity \(\sigma\) is taken to be the hyperbolic tangent function, and
\begin{equation*}
    \vect{x}_t^j(\xi) = e^{-t}\vect{x}^j + \sqrt{\Delta_t}\,\xi,
    \qquad \xi \sim \mathcal{N}(\tens{0},\tens{I}_d),
\end{equation*}
denotes the Ornstein--Uhlenbeck evolved sample at diffusion time \(t=0.01\). The expectation in \(\tens{U}\) is approximated by Monte Carlo using \(\num{1024}\) samples.

We fit an algebraic spectral curve \(P(z, m) = 0\) to the empirical Stieltjes transform of the \(n_0 = 2^2\mathrm{K}\) submatrix, using degrees \((\deg_m, \deg_z) = (7, 5)\) and moment constraints up to degree four. The fitted curve is then evolved by free decompression from \(n_0=2^2\mathrm{K}\) to \(n=2^6\mathrm{K}\). For the edge diagnostics, we do not infer edges from the reconstructed density alone; instead, we evolve branch-point candidates directly from the spectral curve using \Cref{prop:fd-branch-point}, and identify the corresponding physical spectral edges as described in \Cref{sec:edge}. The evolution detects one cusp, corresponding to the splitting of the leftmost component into a spike-like bulk and a neighboring small-eigenvalue bulk, consistent with the cusp mechanism in \Cref{prop:cusp}.

The spectrum of \(\tens{U}\) has a pronounced multi-bulk structure over several orders of magnitude. In particular, the leftmost component is an extremely narrow, spike-like bulk carrying most of the spectral mass: it contains about \(80\%\) of the mass at \(n=\num{4000}\), increasing to about \(98\%\) at \(n=\num{64000}\), while occupying only a very small interval near the origin. Importantly, in our computation this component is not modeled as an atom; it is treated as an absolutely-continuous bulk and evolved by the same free-decompression procedure as the rest of the spectrum. Nevertheless, its mass follows the atom-law prediction of \Cref{prop:atom} with high accuracy. \Cref{tab:diffusion-atom} reports this comparison, together with the center of mass of the leftmost bulk.

\begin{table}[t!]
\caption{Atom-like behavior of the leftmost bulk in the diffusion-model spectrum. The empirical mass is the mass of the leftmost bulk detected from the histogram, while the atom-law column reports the prediction \(w(n)=1-(1-w_0) \, n_0/n\) from \Cref{prop:atom}, with \(w_0=0.8\) and \(n_0=2^2\mathrm{K}\). The center-of-mass column reports the empirical center of the same bulk, and its error is measured relative to the reference center after the leftmost component has separated from the neighboring bulk.}
\label{tab:diffusion-atom}
\vspace{1mm}
\centering
\begin{tabular}{lrrrrr}
    \toprule
    Size & \multicolumn{3}{c}{Left Bulk Mass} & \multicolumn{2}{c}{Left Bulk Center of Mass} \\
    \cmidrule(lr){2-4} \cmidrule(lr){5-6}
    \(n\) & Empirical & Atom Law (ours) & Error & Empirical & Error \\
    \midrule
    \(2^{2}\)K & \(80.0\%\) & \(80.0\%\) & \(0.0\%\) & \(1.87 \times 10^{-3} \) & \(3.7\%\) \\
    \(2^{3}\)K & \(88.8\%\) & \(90.0\%\) & \(1.3\%\) & \(1.80 \times 10^{-3} \) & \(0.0\%\) \\
    \(2^{4}\)K & \(94.4\%\) & \(95.0\%\) & \(0.6\%\) & \(1.80 \times 10^{-3} \) & \(0.1\%\) \\
    \(2^{5}\)K & \(97.2\%\) & \(97.5\%\) & \(0.3\%\) & \(1.81 \times 10^{-3} \) & \(0.4\%\) \\
    \(2^{6}\)K & \(98.6\%\) & \(98.8\%\) & \(0.2\%\) & \(1.81 \times 10^{-3} \) & \(0.4\%\) \\
    \bottomrule
\end{tabular}
\end{table}

We also compare the locations of the bulk edges detected from empirical histograms with the edges evolved by free decompression; see \Cref{fig:diff-edge}. Since the spectrum spans several orders of magnitude, edge errors are measured on the logarithmic \(\lambda\)-axis. At the initial size \(n=2^2\mathrm{K}\), the two leftmost components have not yet separated, while at larger sizes they split into a spike-like leftmost bulk and a neighboring small-eigenvalue bulk. The only edge excluded from the comparison is the left edge of the spike-like bulk. This edge is governed by finite-size extreme-eigenvalue effects rather than by the macroscopic free-decompression flow: initially it follows the Tracy--Widom-type behavior analyzed in \Cref{sec:finite-correction} and illustrated in \Cref{fig:diff-left-edge}. The remaining edges, which correspond to the macroscopic bulk boundaries, are reported in \Cref{tab:diffusion-edge}.

\begin{table}[t!]
\caption{Errors of the evolved bulk edges for the diffusion-model spectrum. The entries compare empirically detected histogram edges with the corresponding edges evolved by free decompression, measured on the logarithmic \(\lambda\)-axis. At \(n=2^2\mathrm{K}\), the two leftmost components have not yet separated, so the entries under Bulk 2 are omitted; after the split, Bulk 1 denotes the spike-like leftmost component and Bulk 2 denotes the neighboring small-eigenvalue bulk. The left edge \(a_1\) of the spike-like bulk is excluded because it is governed by finite-size extreme-eigenvalue effects rather than by the macroscopic free-decompression flow; see \Cref{sec:finite-correction} and \Cref{fig:diff-left-edge}.}
\label{tab:diffusion-edge}
\vspace{1mm}
\centering
\begin{tabular}{lrrrrrrrr}
    \toprule
    Size & \multicolumn{2}{c}{Bulk 1} & \multicolumn{2}{c}{Bulk 2} & \multicolumn{2}{c}{Bulk 3} & \multicolumn{2}{c}{Bulk 4} \\
    \cmidrule(lr){2-3} \cmidrule(lr){4-5} \cmidrule(lr){6-7} \cmidrule(lr){8-9}
    \(n\) & \(a_1\) & \(b_1\) & \(a_2\) & \(b_2\) & \(a_3\) & \(b_3\) & \(a_4\) & \(b_4\) \\
    \midrule
    \(2^{2}\)K &  \( 0.0\%\) & \(0.0\%\) & \textemdash & \textemdash & \(0.0\%\) & \(0.2\%\) & \(0.0\%\) & \(0.1\%\) \\
    \(2^{3}\)K & \textemdash & \(0.3\%\) & \(0.1\%\)   & \(0.3\%\)   & \(0.2\%\) & \(0.5\%\) & \(0.2\%\) & \(0.1\%\) \\
    \(2^{4}\)K & \textemdash & \(0.1\%\) & \(0.2\%\)   & \(0.2\%\)   & \(0.1\%\) & \(0.2\%\) & \(0.2\%\) & \(0.1\%\) \\
    \(2^{5}\)K & \textemdash & \(0.8\%\) & \(0.7\%\)   & \(0.3\%\)   & \(0.1\%\) & \(0.6\%\) & \(0.1\%\) & \(0.1\%\) \\
    \(2^{6}\)K & \textemdash & \(1.2\%\) & \(0.4\%\)   & \(0.2\%\)   & \(0.0\%\) & \(0.8\%\) & \(0.1\%\) & \(0.0\%\) \\
    \bottomrule
\end{tabular}
\end{table}


\section{Implementation and Reproducibility Guide} \label{app:software}

The algorithms in this paper are implemented through the \texttt{\freealg.AlgebraicForm} interface of the open-source Python package \texttt{\freealg}.\footnote{\texttt{\freealg} is available for installation from PyPI (\FreeAlgPypiUrl), the documentation can be found at \FreeAlgDocUrl, and the source code is available at \FreeAlgGithubUrl.} The package also includes several closed-form free-probability distributions, such as compound free Poisson and free L\'{e}vy laws, which are useful for synthetic experiments and validation. \Cref{list:freealg1} shows a minimal setup for a compound free Poisson example: it constructs the distribution, visualizes its spectral-curve branches, generates a large matrix realization, and then subsamples a smaller principal submatrix. We found that error induced by sampling different submatrices was negligible, being orders of magnitude smaller than the deterministic errors reported in our tables.

The synthetic distribution in \Cref{list:freealg1} is used only to provide a controlled demonstration. In applications, the user may instead start directly from a different input matrix, as in the empirical Hessian and diffusion-model examples in \Cref{sec:hessian-example,sec:bg-diffusion}. Given only the smaller matrix \(A_{\mathrm{small}}\), \Cref{list:freealg2} fits an algebraic relation \(P(z,m)=0\), estimates the corresponding density, and applies free decompression to infer the density and atoms at the larger matrix size \(n=\num{8000}\). The same interface can also evolve densities over an array of matrix sizes and track lower-dimensional spectral features such as edges, cusps, and atom masses; the atom and edge evolution for this example is shown in \Cref{fig:cfp-atom}.

Notebook files that reproduce the figures appearing in this work can be found in the codebase. Further details on function arguments and class parameters are available in the package documentation. All experiments were performed on a consumer grade machine, with an AMD Ryzen 7 5800X processor, 128GB DDR4 RAM, running Fedora Linux.

\ifjournal
    \clearpage
\else
\fi

\begin{lstlisting}[
    style=mystyle,
    language=Python,
    caption={Synthetic compound free Poisson setup using \texttt{\freealg}. The distribution object is used only to generate a controlled matrix example and visualize the associated spectral-curve branches.},
    label={list:freealg1},
    float=ht!]
# Install ;#\freealg#; with ;#"\texttt{pip install \freealg}"#;
import freealg as fa
import numpy as np

# Create an object for the Compound Free Poisson distribution (see ;# \Cref{sec:compound-poisson}#;)
dist = fa.distributions.CompoundFreePoisson(t=[2, 5.5], w=[0.75, 0.25], lam=0.1)

# Visualize all branches and their gluing as in ;# \Cref{fig:cfp-branches} #;
x = np.linspace(-1, 9, 200)
y = np.linspace(-3, 3, 200)
dist.plot_branches(x, y)

# Generate a matrix realization of size ;#$ n=\num{8000} $#; for this distribution
A_large = dist.matrix(size=8000)

# Sample the above matrix to a smaller submatrix of size ;#$ n_0 = \num{1000} $#;
A_small = fa.submatrix(A_large, 1000)
\end{lstlisting}


\begin{lstlisting}[
    style=mystyle,
    language=Python,
    caption={Free-decompression workflow using \texttt{\freealg.AlgebraicForm}. Starting from the smaller matrix \(A_{\mathrm{small}}\), the code fits an algebraic spectral curve, estimates the density, decompresses to a larger matrix size, and tracks edges and cusps.},
    label={list:freealg2},
    float=ht!]
# Create an algebraic-form object for the small matrix
af = fa.AlgebraicForm(A_small)

# Fit the polynomial ;#$ P $#; with ;#$ \deg_m(P) = 3 $#; and ;#$ \deg_z(P) = 1 $#; (see ;# \Cref{alg:poly}#;)
af.fit(deg_m=3, deg_z=1)

# Estimate the density of the small matrix from ;#$ P $#; using ;# \Cref{alg:stieltjes} #;
x = np.linspace(0, 10, 200)
rho_small = af.density(x, plot=True)

# Free decompression of density and atoms corresponding to a larger matrix of
# size ;#$ n = \num{8000} $#; using ;# \Cref{alg:decompress} #; (similar to ;# \Cref{fig:cfp-flow} #;). 
rho_large, atoms_large = af.decompress(
    size=8000, x=x, kind='free', return_atoms=True, plot=True)

# Evolve spectral edges from the small to the large matrix size using ;# \Cref{alg:edge} #;
# Also find cusp points where the bulks split (similar to ;# \Cref{fig:cfp-atom}#;)
sizes = np.linspace(1000, 8000, 100)
t = np.log(sizes / 1000)
edges, cusps = af.edge(t, kind='free')
\end{lstlisting}


\end{appendices}

\end{document}